\documentclass{article}

\usepackage{microtype}
\usepackage{graphicx}
\usepackage{subfigure}
\usepackage{booktabs} % for professional tables
\usepackage{hyperref}

\usepackage[accepted]{icml2019}

\icmltitlerunning{Tighter Problem-Dependent Regret Bounds}

\usepackage{algorithm}
\usepackage{algorithmic}
\usepackage{enumerate}
\usepackage{mathtools}
\usepackage{enumitem}
\usepackage{amsthm}
\usepackage{amsmath}
\usepackage{amssymb}
\usepackage{bbm}
\usepackage[symbols,nogroupskip]{glossaries-extra}
\usepackage{todonotes}
\usepackage{titletoc}
\usepackage{macros}
\usetikzlibrary{arrows,automata,calc,arrows.meta,positioning}
\allowdisplaybreaks
\def\MaxReturn{\ensuremath{\mathcal G}}
\def\Vstar{\ensuremath{V^{\pi^*}}}
\def\ll{\tilde L}
\newcommand{\cc}[1]{c_{#1}}

\newcommand{\MainExplorationBonusSimplified}[1]{\ensuremath{\underbrace{\overbrace{\sqrt{\frac{\Var_{s\sim\pohata{s}{a}} \vo{t+1} }{n_k(s,a)}}}^{\substack{\textsc{Dominant Term} \\ \textsc{of Exploration Bonus}}} + \frac{H}{n_k(s,a)}}_{\textsc{Empirical Bernstein}} + \underbrace{\(\frac{\|\vo{t+1}-\vp{t+1} \|_{\pohata{s}{a}}}{\sqrt{\nsaa{s}{t}{k}}} +\frac{H}{\nsaa{s}{t}{k}}\)}_{\textsc{Correction Bonus}}}}
\newcommand{\sumall}[0]{\sum_{\substack{k \in [K] \\ t \in [H] \\ (s,a) \in \mathcal S \times \mathcal A}}}
\newcommand{\Esa}[0]{\E_{(s,a) \sim \piok}}
\newcommand{\ExtraBonus}[0]{\ensuremath{
 \frac{4J +\BpPlus}{\nsaa{s}{t}{k}} + \frac{\Bv \devihat{t+1}}{\sqrt{\nsaa{s}{t}{k}}}}}
\newcommand{\LowerOrderTerm}[0]{\ensuremath{\sqrt{S}SAH^2(\sqrt{S}+\sqrt{H})}}
\newcommand{\LOTbound}[0]{\ensuremath{\sqrt{S}SAH(F+D+H^{\frac{3}{2}}) + S^2AH}}
\newcommand{\CDObound}[0]{\ensuremath{SAH^2(F+D)^2 + SAH^5}}

\newcommand{\TDEbound}[0]{\ensuremath{\sqrt{\Complexity SAT} + JSA}}
\newcommand{\ROEbound}[0]{\ensuremath{\sqrt{\ComplexityReward{}SAT} + SA}}
\newcommand{\TDEboundPi}[0]{\ensuremath{\sqrt{\ComplexityPi SAT} + JSA}}
\newcommand{\TDObound}[0]{\ensuremath{\sqrt{\Complexity SAT} + (J+\BpPlus)SA + \DeltaGbound}}
\newcommand{\TDOboundPi}[0]{\ensuremath{\sqrt{\ComplexityPi SAT} + (J+\BpPlus)SA + \DeltaGbound}}

\newcommand{\gtruea}[2]{\ensuremath{g(\potruea{#1}{#2},\vtrue{#2+1})}}

\newcommand{\Bernstein}[0]{\ensuremath{\sqrt{\frac{2 \Var_p \vtrue{t+1} \ln\frac{2SAT}{\delta'}}{n_k(s,a)}} + \frac{H\ln\frac{2SAT}{\delta'}}{3n_k(s,a)} }}
\newcommand{\EmpiricalBernsteinRewards}[0]{\ensuremath{\sqrt{\frac{2 \widehat\Var R(s,a) \ln\frac{4SAT}{\delta'}}{n_k(s,a)}} + \frac{14\ln\frac{4SAT}{\delta'}}{3n_k(s,a)} }}
\newcommand{\EmpiricalBernsteinGenericHAlgorithm}[2]{\ensuremath{\sqrt{\frac{2 \widehat\Var_{#1} (#2) \ln\frac{4SAT}{\delta'}}{n_k(s,a)}} + \frac{H\ln\frac{4SAT}{\delta'}}{3(n_k(s,a)-1)} }}
\newcommand{\EmpiricalBernsteinRewardsAlgorithm}[0]{\ensuremath{\sqrt{\frac{2 \widehat\Var R(s,a) \ln\frac{4SAT}{\delta'}}{n_k(s,a)}} + \frac{7\ln\frac{4SAT}{\delta'}}{3(n_k(s,a)-1)} }}

\newcommand{\Ltk}[0]{\ensuremath{L_{k}}}

\newcommand{\devihat}[1]{\ensuremath{\| \vo{#1}-\vp{#1} \|_{2,\hat p}}}
\newcommand{\devi}[1]{\ensuremath{\| \vo{#1}-\vp{#1} \|_{2,p}}}
\newcommand{\devihatstar}[1]{\ensuremath{\| \vo{#1}-\vtrue{#1} \|_{2,\hat p}}}

\newcommand{\devistarv}[1]{\ensuremath{\| \vo{#1}-V \|_{2,p}}}

\newcommand{\devistarpessimisticv}[1]{\ensuremath{\| V-\vp{#1} \|_{2,p}}}
\newcommand{\bbonus}[0]{\ensuremath{b_k^{pv}}}

\newcommand{\Qrv}[0]{\ensuremath{\mathbb{Q}}}
\newcommand{\ComplexityQ}[0]{\ensuremath{\mathbb{Q^*}}}
\newcommand{\Complexity}[0]{\ensuremath{\mathbb{C^*}}}
\newcommand{\ComplexityReward}[0]{\ensuremath{\mathbb{C}^*_r}}

\newcommand{\ComplexityPi}[0]{\ensuremath{\mathbb{C^{\pi}}}}

\newcommand{\lnnsa}[0]{\ensuremath{H\ln\frac{SAH}{\delta'}}}

\newcommand{\visita}[3]{\ensuremath{w_{#2#3}(#1,a)}}

\newcommand{\Alg}[1]{\ensuremath{\textsc{Euler}_{#1}}}
\newcommand{\piok}[0]{\ensuremath{\tilde \pi_{k} }}

\newcommand{\nsaa}[3]{\ensuremath{n_{#3}(#1,a)}}

\newcommand{\phit}[2]{\ensuremath{\tilde \phi_{k#2}(#1,\tilde \pi_{k#2}(#1))}}

\newcommand{\dphit}[2]{\ensuremath{{}^{\dagger}\tilde \phi_{k#2}(#1,\tilde \pi_{k#2}(#1))}}

\newcommand{\rohat}[2]{\ensuremath{\hat r_{k}(#1,\tilde \pi_{k}(#1,#2))}}
\newcommand{\rotrue}[2]{\ensuremath{r(#1, \tilde \pi_{k}(#1,#2))}}
\newcommand{\rtrue}[2]{\ensuremath{{r}(#1, \pi^*(#1,#2))}}
\newcommand{\roa}[2]{\ensuremath{\tilde r_{k}(#1,a)}}
\newcommand{\rohata}[2]{\ensuremath{\hat r_{k}(#1,a)}}
\newcommand{\rotruea}[2]{\ensuremath{r(#2,a)}}

\newcommand{\poa}[2]{\ensuremath{\tilde p_{k}(\cdot \mid #1,a)}}
\newcommand{\potrue}[2]{\ensuremath{p(\cdot \mid #1,\tilde \pi_{k}(#1,#2))}}
\newcommand{\potruea}[2]{\ensuremath{p(\cdot \mid #1,a)}}

\newcommand{\potruespa}[3]{\ensuremath{p(#3 \mid #1,a)}}
\newcommand{\pohat}[2]{\ensuremath{\hat{p}_k(\cdot \mid #1,\tilde \pi_{k}(#1,#2))}}
\newcommand{\rhat}[2]{\ensuremath{\hat{r}_k(#1,#2)}}
\newcommand{\br}[2]{\ensuremath{b^r_k(#1,#2)}}
\newcommand{\phat}[2]{\ensuremath{\hat{p}_k(\cdot \mid #1,\pi^*(#1,#2))}}

\newcommand{\pohata}[2]{\ensuremath{\hat{p}_k(\cdot \mid #1,a)}}

\newcommand{\pohatasp}[3]{\ensuremath{\hat{p}_k(#3 \mid #1,a)}}
\newcommand{\ptrue}[2]{\ensuremath{p(\cdot \mid #1,\pi^*(#1,#2))}}
\newcommand{\vo}[1]{\ensuremath{\overline V^{\piok}_{#1k}}}
\newcommand{\vp}[1]{\ensuremath{\underline V^{\piok}_{#1k}}}
\newcommand{\vtrue}[1]{\ensuremath{V^{\pi^*}_{#1}}}
\newcommand{\votrue}[1]{\ensuremath{V^{\piok}_{#1}}}

\newcommand{\vos}[2]{\ensuremath{\tilde V^{\piok}_{k#1}(#2)}}

\newcommand{\dvos}[3]{\ensuremath{{}^{#3,\dagger}\tilde V^{\piok}_{#1}(#2)}}

\newcommand{\vostrue}[2]{\ensuremath{V^{\piok}_{#1}(#2)}}

\newcommand{\sumt}[1]{\ensuremath{\sum_{t=#1}^H}}
\newcommand{\sumtau}[1]{\ensuremath{\sum_{\tau=#1}^H}}
\newcommand{\sumk}[0]{\ensuremath{\sum_{k=1}^{K}}}
\newcommand{\sumsLk}[0]{\ensuremath{\sum_{(s,a)\in \Ltk}}}
\newcommand{\sumsnLk}[0]{\ensuremath{\sum_{(s,a) \not \in \Ltk}}}
\newcommand{\sumsa}[0]{\ensuremath{\sum_{s,a}}}

\newcommand{\ostk}[1]{\ensuremath{\overline s_{#1}}}
\newcommand{\ustk}[1]{\ensuremath{\underline s_{#1}}}
\newcommand{\tstk}[1]{\ensuremath{\tilde s_{#1}}}
\newcommand{\dostk}[1]{\ensuremath{\overline s^{\dagger}_{#1}}}
\newcommand{\dustk}[1]{\ensuremath{\underline s^{\dagger}_{#1}}}
\newcommand{\dtstk}[1]{\ensuremath{\tilde s^{\dagger}_{#1}}}
\newcommand{\sdagger}[0]{\ensuremath{{}^{\dagger}\mathcal S^{\piok}}}
\newcommand{\ssdagger}[1]{\ensuremath{{}^{\dagger,#1}\mathcal S^{\piok}}}
\newcommand{\sosdagger}[1]{\ensuremath{{}^{\dagger,#1}\mathcal{\tilde S}^{\piok}}}

\newcommand{\Bp}[0]{\ensuremath{ B_p}}
\newcommand{\Bv}[0]{\ensuremath{ B_v}}

\newcommand{\BpPlus}[0]{\ensuremath{\Bp}}

\newcommand{\DeltaGbound}[0]{\ensuremath{SAH(F+D+H^{\frac{3}{2}})}}

\newcommand{\realoptimalitygap}[0]{\ensuremath{ \vtrue{t}(s) - \votrue{t}(s) }}
\newcommand{\optimalitygap}[0]{\ensuremath{ \vo{t}(s) - \votrue{t}(s) }}
\newcommand{\bonus}[1]{\ensuremath{\( \poa{#1}{t} - \pohata{#1}{t} \)^{\top}\vo{t+1} }}

\newcommand{\estimation}[1]{\ensuremath{ \( \pohata{#1}{t} - \potruea{#1}{t} \)^{\top}\vtrue{t+1} }}
\newcommand{\estimationa}[1]{\ensuremath{ \( \pohata{#1}{t} - \potruea{#1}{t} \)^T\vtrue{t+1} }}
\newcommand{\lowerorder}[0]{\ensuremath{ \( \pohata{s}{t} - \potruea{s}{t} \)^{\top}\( \vo{t+1} - \vtrue{t+1} \) }}
\newcommand{\reward}[1]{\ensuremath{ \roa{#1}{t} - \rotruea{t}{#1} }}

\newcommand{\dg}[2]{g(p,\vtrue{#2})}

\newcommand{\dogtrue}[2]{g^{}(p,\votrue{#2})}

\newcommand{\emmalemma}[1]{\ensuremath{\( \roa{s_{#1}}{#1} - \rotruea{s_{#1}}{s_#1} + \( \poa{s_{#1}}{#1} - \potruea{s_{#1}}{#1} \)^{\top} \vo{#1+1} \) }} 
\glsxtrnewsymbol[description={Reinforcement Learning Algorithm that satisfies the Assumptions in this work}]{Alg}{\ensuremath{\Alg{}}}
\glsxtrnewsymbol[description={True MDP}]{Mstar}{\ensuremath{M^*}}
\glsxtrnewsymbol[description={Set of Feasible MDPs to Extended Value Iteration at step $k$}]{Mset}{\ensuremath{\mathcal M_k}}
\glsxtrnewsymbol[description={Optimistic MDP}]{Mtilde}{\ensuremath{\tilde M}}
\glsxtrnewsymbol[description={Optimal Policy}]{Pistar}{\ensuremath{\pi^*}}
\glsxtrnewsymbol[description={Optimistic Policy at Episode $k$}]{Pitilde}{\ensuremath{\piok}}
\glsxtrnewsymbol[description={Value Function in $M^*$ at timestep t for policy $\piok$}]{V}{\ensuremath{\votrue{t}}}
\glsxtrnewsymbol[description={Optimistic Value Function in $\tilde M$ at timestep $t$ for the optimistic policy $\piok$}]{Vtilde}{\ensuremath{\vo{t}}}
\glsxtrnewsymbol[description={Expectation with respect to the true MDP $M^*$ by starting at state $s$ and following policy $\pi$}]{epis}{\ensuremath{\E^{\pi}_{s}}}
\glsxtrnewsymbol[description={Expectation with respect to the optimistic MDP $\tilde M$ by starting at state $s$ and following policy $\pi$}]{epist}{\ensuremath{\tilde \E^{\pi}_{s}}}
\glsxtrnewsymbol[description={$\argmax_s \vos{t}{s}$}]{ostk}{\ensuremath{\ostk{tk}}}
\glsxtrnewsymbol[description={$\argmin_s \vos{t}{s}$}]{ustk}{\ensuremath{\ustk{tk}}}
\glsxtrnewsymbol[description={$\argmax_s \phit{s}{t}$}]{stilde}{\ensuremath{\tstk{tk}}}
\glsxtrnewsymbol[description={$\argmax_s \dvos{t}{s}{x}$}]{dostk}{\ensuremath{\dostk{tk}}}
\glsxtrnewsymbol[description={$\argmin_s \dvos{t}{s}{x}$}]{dustk}{\ensuremath{\dustk{tk}}}
\glsxtrnewsymbol[description={$\argmax_s \dphit{s}{t}{s}$}]{dstilde}{\ensuremath{\dtstk{tk}}}
\glsxtrnewsymbol[description={Set of states with a positive visitation probability during episode $k$ under policy $\piok$}]{Sdagger}{\ensuremath{\sdagger}}
\glsxtrnewsymbol[description={Successors of state $s$ under policy \piok, that is, set of states in the support of $p(\cdot \mid s,\piok(s))$}]{Ssdagger}{\ensuremath{\ssdagger{s}}}
\glsxtrnewsymbol[description={Successors of state $s$ under policy \piok in the MDP defined by the empirical estimates of the transition probabilities, that is, set of states in the support of \pohat{s}{t}}]{Sosdagger}{\ensuremath{\sosdagger{s}}}

\newcommand{\DeterministicDomainRegret}[1]
{\begin{proposition}
#1
If \Alg{} is run on a deterministic MDP then the regret is bounded by $\tilde O(SAH^2)$.
\end{proposition}
}
\newcommand{\MainResultMainText}[3]
{
\begin{theorem}[Problem Dependent High Probability Regret Upper Bound for \Alg{}]
#1
With probability at least $1-\delta$ the regret of \Alg{} is bounded for any time $T\leq KH$ by the minimum between
\small
\begin{equation}
#2
   \tilde  O\(\sqrt{\ComplexityQ SAT} + \LowerOrderTerm \)
\end{equation}
\normalsize
and
\small
\begin{equation}
#3
	 \tilde O\(\sqrt{\frac{\MaxReturn^2}{H}SAT} + \LowerOrderTerm\),
\end{equation}
\normalsize
jointly for all episodes $k \in [K]$.%, where $L=\polylog(S,A,T,\delta)$.
\end{theorem}
}

\newcommand{\MainResult}[5]
{
\subsection{Main Result}
\begin{theorem}[Main Result]
#1
If $\phi$ is admissible then with probability at least $1-\delta$ the cumulated regret of \Alg{} up to timestep $T$ is upper bounded by the minimum between:
\begin{align}
#4
\tilde O (\sqrt{(\ComplexityReward{}+ \Complexity{}) SAT}  +\sqrt{S}\DeltaGbound)
\end{align}
and
\begin{equation}
#5
\tilde O (\sqrt{(\ComplexityReward{}+ \ComplexityPi{}) SAT}  +\sqrt{S}\DeltaGbound + \Bv^2SAH^2)	
\end{equation}
where
\begin{align}
	F \stackrel{def}{=}& \tilde O(H\sqrt{S} + \Bv H)\\
	D \stackrel{def}{=}& \tilde O(J + \BpPlus)
\end{align}
and $\Complexity$ and $\ComplexityPi$  are problem dependent upper bounds on the following quantities:
\begin{equation}
#2
	\Complexity \geq \frac{1}{T} \sumk \sumt{1} \sumsa \visita{s}{t}{k} \dg{s}{t+1}^2 \stackrel{def}{=} \frac{1}{T} \sumk \sumt{1}\E_{\piok} \dg{s}{t+1}^2
\end{equation}
and 
\begin{equation}
#3
	\ComplexityPi \geq \frac{1}{T} \sumk \sumt{1} \sumsa \visita{s}{t}{k} \dogtrue{s}{t+1}^2 \stackrel{def}{=} \frac{1}{T} \sumk \sumt{1}\E_{\piok} \dogtrue{s}{t+1}^2.
\end{equation}
while $\ComplexityReward{}$ is defined in lemma \ref{lem:RewardEstimationAndOptimism}.
\end{theorem}
}

\newcommand{\BernsteinResult}[1]{
\begin{proposition}[Problem Independent Bound for \Alg{} with Bernstein Inequality]
#1
If \Alg{} is run with Bernstein Inequality defined in equation \ref{eqn:BernsteinInequality} with $\Bp$ and $\Bv$ and $\phi$ defined in proposition \ref{prop:BernsteinIsAdmissible} then with probability at least $1-\delta$ the regret of \Alg{} at timestep $T$ is bounded by the minimum between
\begin{equation}
    \tilde O\(\sqrt{\ComplexityQ SAT} + \LowerOrderTerm \)
\end{equation}
and
\begin{equation}
	\tilde O\(\sqrt{\frac{\(\MaxReturn\)^2}{H}SAT} + \LowerOrderTerm \).
\end{equation}
jointly for all episodes $k \in [K]$. 
\end{proposition}
}

\newcommand{\BernsteinIsAdmissible}[1]
{\begin{proposition}[Bernstein Is Admissible]
#1
Bernstein Inequality as presented in equation \ref{eqn:BernsteinInequality} satisfies assumption \ref{ass:ConfidenceIntervals} and \ref{ass:FiniteTimeBonusBound} and is therefore  admissible for \Alg{} with coefficients $J = \frac{H\ln\frac{2SAT}{\delta'}}{3} = \tilde O(H)$,  $\Bv = \sqrt{2\ln\frac{2SAT}{\delta'}}=\tilde O(1)$ and $\Bp = \sqrt{2}H\sqrt{\ln\frac{2SAT}{\delta'}} = \tilde O(H)$.
\end{proposition}
}

\def\Regret{\mathcal R}

\begin{document}

\twocolumn[
\icmltitle{Tighter Problem-Dependent Regret Bounds in Reinforcement Learning \\ without Domain Knowledge using Value Function Bounds}

% It is OKAY to include author information, even for blind
% submissions: the style file will automatically remove it for you
% unless you've provided the [accepted] option to the icml2019
% package.

% List of affiliations: The first argument should be a (short)
% identifier you will use later to specify author affiliations
% Academic affiliations should list Department, University, City, Region, Country
% Industry affiliations should list Company, City, Region, Country

% You can specify symbols, otherwise they are numbered in order.
% Ideally, you should not use this facility. Affiliations will be numbered
% in order of appearance and this is the preferred way.

\begin{icmlauthorlist}
\icmlauthor{Andrea Zanette}{icme}
\icmlauthor{Emma Brunskill}{cs}
\end{icmlauthorlist}

\icmlaffiliation{icme}{Institute for Computational and Mathematical Engineering, Stanford University, USA}
\icmlaffiliation{cs}{Department of Computer Science, Stanford University, USA}

\icmlcorrespondingauthor{Andrea Zanette}{zanette@stanford.edu}
\icmlcorrespondingauthor{Emma Brunskill}{ebrun@cs.stanford.edu}

% You may provide any keywords that you
% find helpful for describing your paper; these are used to populate
% the "keywords" metadata in the PDF but will not be shown in the document
\icmlkeywords{Exploration, Regret, Finite Horizon, Tabular, Reinforcement Learning}

\vskip 0.3in
]

% this must go after the closing bracket ] following \twocolumn[ ...

% This command actually creates the footnote in the first column
% listing the affiliations and the copyright notice.
% The command takes one argument, which is text to display at the start of the footnote.
% The \icmlEqualContribution command is standard text for equal contribution.
% Remove it (just {}) if you do not need this facility.

%\printAffiliationsAndNotice{}  % leave blank if no need to mention equal contribution
\printAffiliationsAndNotice{} % otherwise use the standard text.

\begin{abstract}
Strong worst-case performance bounds for episodic reinforcement learning exist 
but fortunately in practice RL algorithms perform much better than 
such bounds would predict. Algorithms and theory that provide strong 
problem-dependent bounds could help illuminate the key features of what 
makes a RL problem hard and reduce the barrier to using RL algorithms 
in practice. As a step towards this
we derive an algorithm and analysis for finite horizon discrete MDPs  
with state-of-the-art worst-case regret bounds and substantially tighter bounds if the RL 
environment has special features but without apriori 
knowledge of the environment from the algorithm. As a result of our analysis, 
we also help address an open learning theory question~\cite{jiang2018open} 
about episodic MDPs with a constant upper-bound on the sum of rewards, 
providing a regret bound function of the number of episodes with no 
dependence on the horizon.  
\end{abstract}

\section{Introduction}
In reinforcement learning (RL) an agent must learn how to make good decision without having access to an exact model of the world. Most of the literature for provably efficient exploration in Markov decision processes (MDPs) \cite{Jaksch10,OVR13,LH14,Dann15,Dann17,OV17,Azar17,KWY18} 
has focused on providing near-optimal worst-case performance bounds.
Such bounds are highly desirable as they do not depend on the structure of the particular environment considered and 
therefore hold for even extremely hard-to-learn MDPs.

Fortunately in practice reinforcement learning algorithms often perform far better 
than what these problem-independent bounds would suggest. While we may observe better 
or worse performance empirically on different MDPs, we would 
like to derive a more systematic understanding of what types of decision processes 
are inherently easier or more challenging for RL. This motivates 
our interest in deriving algorithms and theoretical analyses that provide problem-dependent 
bounds. 
%While there have been many  
%papers that introduce RL algorithms that exploit known features
%\cite{BLLLN09,BLLLR08,Bartlett09,WR13,Fruit18}, here we focus on problem-agnostic 
%algorithms that ensure good performance on MDPs that are easier to learn, but are 
%given no a priori information on the MDP properties that would determine 
%whether less exploration is sufficient to discover a near optimal policy. 
Ideally, such algorithms will do as well as RL solutions designed for the worst case 
if the problem is pathologically difficult and otherwise match the 
performance bounds of algorithms specifically designed for a particular 
problem subclass.  
This exciting scenario might bring considerable saving in the time spent 
designing domain-specific RL solutions and in training a human expert to judge and 
recognize the complexity of different problems. An added benefit would 
include the robustness of the RL solution in case the actual model does 
not belong to the identified subclass, yielding increased confidence to 
deploying RL to high-stakes applications. 

Towards this goal, in this paper we contribute with a new algorithm for episodic tabular reinforcement 
learning which automatically provides provably stronger regret bounds in many domains which have 
a small variance  
of the optimal value function (in the infinite horizon setting, this variance has been called the \textit{environmental norm} \cite{Maillard14}). Indeed, there is good reason to believe that some features 
of the range or variability of the optimal value function should be a critical aspect 
of the hardness of reinforcement learning in a MDP. Many worst-case bounds for 
finite-state MDPs scale with a  
\emph{worst case} bound on the range / magnitude of the value function, such as 
%For example, \textsc{Ucrl2} from \cite{Jaksch10} and a recently proposed algorithm from \cite{Azar17} enjoy a regret bound\footnote{leading order term only} of
%\begin{equation}
%\underbrace{\tilde O \(DS\sqrt{AT}\)}_{\text{Undiscounted \cite{Jaksch10}} } \quad \mid \quad \underbrace{\tilde O \(\sqrt{HSAT}\)}_{\text{Fixed Horizon  \cite{Azar17}}}
%\end{equation}
the diameter $D$ for an infinite-horizon setting and the horizon $H$ in an episodic problem. 
Note that here both $D$ and $H$ arise in the analyses as upper bounds on the (range of the) \emph{optimistic} value function across the entire MDP\footnote{Many RL algorithms with strong performance bounds rely on the principle of optimism under uncertainty and compute an optimistic value function.}. As more samples are collected, one would hope that the agent's optimistic value function converges to the true optimal value function. Unfortunately this is not the case, see for example \cite{Jaksch10,Bartlett09,Zanette18a} for a discussion of this. 
As a result, most prior analyses bounded the optimistic value function by generic quantities like $D$ or $H$ regardless of the actual behaviour of the optimal value function.% unless this domain knowledge was provided to the algorithm \cite{Bartlett09,Fruit18}.

While the majority of formal performance guarantees has focused on bounds for the worst case, there have 
been several contributions of algorithms and/or theoretical analyses focused on MDPs with particular 
structure. Such contributions have focused on the infinite horizon setting, which involves a number of subtleties that are not present in the finite horizon setting we consider, which is likely a cause of the less strong results in this setting which can require stronger input knowledge on a tighter range on the possible value function~\cite{Bartlett09,Fruit18}, or do not match in dominant terms strong bounds for the worst case setting~\cite{Maillard14}. We defer more detailed discussion of related work to Section~\ref{sec:rellit}, except to briefly highlight likely the most closely related recent result from~\cite{TM18}. Like us, Talebi and Maillard provide a problem-dependent regret bound that scales as a function of the variance of the next state distribution. However, like the aforementioned references, their focus is on the infinite horizon setting. In this setting the authors achieve their resulting regret bound under an assumption that the mixing time of the MDP is such that all states are visited at a linear rate in expectation regardless of the agent's chosen policies. This mixing rate, that could be exponential in certain MDPs, appears in the regret bound. In our, arguably simpler finite horizon setting, we do not use an assumption on the mixing rate of the MDP and we instead pursue a different proof technique to obtain strong results for this setting.

More precisely, in this paper we derive an algorithm for finite horizon discrete MDPs and associated analysis 
that yields state-of-the art worst-case regret bounds of order $\tilde O(\sqrt{HSAT})$ in the leading term while improving if the environment has next-state value function variance (i.e., small environmental norm) or bounded total possible reward. 
%\paragraph{Contribution}
%We make the following contribution: we propose an algorithm for finite horizon MDPs and associated analysis which has state-of-the art worst-case regret bounds of order $\tilde O(\sqrt{HSAT})$ all the while improving if the environment has small environmental norm. 
Compared to the existing literature, our work  
\begin{itemize}[leftmargin=*]
\itemsep0em 
\item Maintains state of the art worst-case guarantees  \cite{Azar17} for episodic finite horizon settings,
%\item Yields a regret bound that depends on the range of the optimal value function \cite{Bartlett09,Fruit18} when restricted to the episodic setting 
%and \emph{without the need of domain knowledge},
\item Improves the regret bounds of \cite{Zanette18a} when deployed in the same settings, 
\item Provides demonstration that characterizing problems using environmental norm \cite{Maillard14} can yield substantially tighter theoretical guarantees in the finite horizon setting,  %while requiring less problem inputs, and 
\item Identifies problem classes with low environmental norm which are of significant interest, including deterministic domains, single-goal MDPs, and high stochasticity domains, and
\item Helps address an open learning theory problem~\cite{jiang2018open}, showing that for their setting, we obtain a regret bound that scales with no dependence on the planning horizon in the dominant terms. 
%addresses a recently conjectured lower bound \cite{jiang2018open} in terms of dependence on the time-horizon by showing the the regret leading order term even when expressed in terms of episodes (as opposed to timesteps) does not depend on the time horizon if the value function is normalized to $1$.
\end{itemize}

%\paragraph{Organization}
The paper is organized as follows: we recall some basic definitions in Section \ref{sec:Preliminaries} and  describe the algorithm in Section \ref{sec:Algorithm}. We state and comment the main result in Section \ref{sec:MainResult}, discuss how this helps address an open learning 
theory problem in Section~\ref{sec:colt_conj} and then describe selected problem-dependent bounds in Section \ref{sec:ProblemDependentBounds}. The analysis is sketched in Section \ref{sec:TheoreticalAnalysisSketch}. Due to space constraints, most proofs are in the full report available at:\\
\url{https://arxiv.org/abs/1901.00210}.

\section{Preliminaries and Definitions}
\label{sec:Preliminaries}
In this section we introduce some notation and definitions. We consider undiscounted finite horizon MDPs \cite{SB98}, which are defined by a tuple $\mathcal M = \langle\ \mathcal S,\mathcal A,p,r, H \rangle\ $, where $\mathcal S$ and  $\mathcal A$ are the state and action spaces with cardinality $S$ and $A$, respectively. We denote by $p (s'\mid s, a)$ the probability of transitioning to state $s'$ after taking action $a$ in state $s$ while  $r(s,a) \in [0,1]$ is the  average instantaneous  reward collected. We label with $n_k(s,a)$ the visits to the $(s,a)$ pair at the beginning of the $k$-th episode. The agent interacts with the MDP starting from arbitrary initial states in a sequence of episodes $k \in [K]$( where $[K] = \{j \in \mathbb{N}: 1 \leq j \leq K \}$) of fixed length $H$  by selecting a policy $\piok$ which maps states $s$ and timesteps $t$ to actions. Each policy identifies a value function for every state $s$ and timestep $t\in[H]$ defined as $
V_t^{\piok}(s_t) = \Esa \sum_{i = t}^H r(s,a)$
which is the expected return until the end of the episode (the conditional expectation is over the pairs $(s,a)$ encountered in the MDP upon starting from $s_t$). The optimal policy is indicated with $\pi^*$ and its value function as $\vtrue{t}$.
We indicate with $\vp{t+1}$ and $\vo{t+1}$, respectively, a pointwise underestimate, respectively, overestimate, of the optimal value function and with $\pohata{s}{t}$ and $\rohata{s}{t}$ the MLE estimates of $\potruea{s}{t}$ and $\rotruea{s}{s}$. We focus on deriving a high probability upper bound on the $
\textsc{Regret}(K) \stackrel{def}{=} \sum_{k \in [K]} \( \vtrue{1}(s_k) -  \votrue{1}(s_k) \)$
to measure the agent's learning performance.
We use the $\tilde O(\cdot )$ notation to indicate a quantity that depends on $(\cdot)$ up to a $\polylog$ expression of a quantity at most polynomial in $S,A,T,K,H,\frac{1}{\delta}$. We also use the $\lesssim, \gtrsim, \simeq$ notation to mean $\leq, \geq, =$, respectively, up to a numerical constant and indicate with $\| X \|_{2,p}$ the $2$-norm of a random variable\footnote{To be precise, this is a norm between classes of random variables that are  almost surely the same} under $p$, i.e., $\| X \|_{2,p} \stackrel{def}{=} \sqrt{\E_p X^2} \stackrel{def}{=} \sqrt{\sum_{s'} p(s')X^2(s')}$ if $p(\cdot)$ is its probability mass function.

\section{\texorpdfstring{\Alg{}}{}}
\label{sec:Algorithm}

\begin{algorithm*}[!htb]
   \caption{\Alg{} for Stationary Episodic MDPs}
   \label{main:AlgorithmMainText}
\begin{algorithmic}[1]
   \STATE \textbf{Input}: $\delta' = \frac{1}{7}\delta$, $\br{s}{a} = \EmpiricalBernsteinRewardsAlgorithm{}$,  $\phi(s,a) = \EmpiricalBernsteinGenericHAlgorithm{\widehat p_k(s,a)}{\vo{t+1}{}}$, $\Bp = H\sqrt{2\ln\frac{(4SAT)}{\delta'}},\Bv =\sqrt{2\ln\frac{(4SAT)}{\delta'}}, J = H\ln\frac{(4SAT)/\delta'}{3}$.
   %\STATE $n(s,a) = r_{sum}(s,a) = p_{sum}(s',s,a) = 0, \;\; \forall s,a\in \mathcal S \times \mathcal A; \quad \overline V_{H+1}(s) = 0, \;\; \forall s \in \mathcal S$
   \FOR{$k=1,2,\dots$}
      \FOR{$t=H,H-1,\dots,1$}
     	 \FOR{$s \in \mathcal S$}
	 \FOR{$a \in \mathcal A$}
	 \STATE $\hat p = \frac{p_{sum}(\cdot,s,a)}{n_k(s,a)}$
   \STATE $\bbonus = \phi(\hat p(s,a),\overline V_{t+1}) + \frac{1}{\sqrt{n(s,a)}} \( \frac{4J + \BpPlus}{\sqrt{n_k(s,a)}} + \Bv \| \overline V_{t+1} - \underline V_{t+1} \|_{2,\hat p} \)$ \\
   \STATE $Q(a) = \min\{ H-t, \rhat{s}{\piok(s,t)}+\br{s}{a} + \hat p^\top \overline V_{t+1} + \bbonus \}$
  \ENDFOR
   \STATE $\piok(s,t) = \argmax_a Q(a)$
   \STATE $\overline V_t(s) = Q(\piok(s,t))$ 
   \STATE $\bbonus = \phi(\hat p(s,\piok(s,t)),\underline V_{t+1}) + \frac{1}{\sqrt{n(s,\piok(s,t))}} \( \frac{4J + \BpPlus}{\sqrt{n_k(s,\piok(s,t))}} + \Bv \| \overline V_{t+1} - \underline V_{t+1} \|_{2,\hat p} \)$ \\
   \STATE $\underline V_t(s) = \max\{0, \rhat{s}{\piok(s,t)} - \br{s}{\piok(s,t)} + \hat p^\top \underline V_{t+1} - \bbonus\}$
   \ENDFOR
   \ENDFOR
   \STATE Evaluate policy $\piok$ and update MLE estimates $\hat p(\cdot,\cdot)$ and $\hat r(\cdot,\cdot)$
   \ENDFOR
\end{algorithmic}
\end{algorithm*}

% \begin{algorithm*}[]
%   \caption{\Alg{} for Stationary Episodic MDPs}
%   \label{main:AlgorithmMainText}
% \begin{algorithmic}[1]
%   \STATE \textbf{Input}: failure probability $\delta$, $\delta' \stackrel{def}{=} \frac{1}{7}\delta$, $ L  \stackrel{def}{=}\sqrt{2\ln(14SAT/\delta)}$, $\overline V_{H+1} = \underline V_{H+1} =  0$
%   \STATE $\phi(p,V) \stackrel{def}{=} \EmpiricalBernsteinGenericH{p}{V}$, $ \br{s}{a} = \EmpiricalBernsteinRewards$
%   \FOR{$k=1,2,\dots$}
%       \FOR{$t=H,H-1,\dots,1$}
%      	 \FOR{$s \in \mathcal S$}
% 	 \FOR{$a \in \mathcal A$}
%   \STATE $\bbonus = \phi(\hat p(s,a),\overline V_{t+1}) + \frac{1}{\sqrt{n(s,a)}} \( \frac{56 HL}{3\sqrt{n_k(s,a)}} + L \| \overline V_{t+1} - \underline V_{t+1} \|_{\hat p(s,a)} \)$ \\
%   \STATE $Q(a) = \min\{ H-t, \rhat{s}{t}+\br{s}{a} + \hat p^\top \overline V_{t+1} + \bbonus \}$
%   \ENDFOR
%   \STATE $\piok(s,t) = \argmax_a Q(a)$, $\overline V_t(s) = Q(s,\piok(s,t))$ 
%   \STATE $\bbonus = \phi(\hat p(s,\tilde \pi(s,t)),\underline V_{t+1}) + \frac{1}{\sqrt{n(s,a)}} \( \frac{56 HL}{3\sqrt{n_k(s,\piok(s,t))}} + L \| \overline V_{t+1} - \underline V_{t+1} \|_{\hat p(s,\tilde \pi(s,t))} \)$ \\
%   \STATE $\underline V_t(s) = \max\{0, \rhat{s}{\piok(s,t)} - \br{s}{\piok(s,t)} + \hat p^\top \underline V_{t+1} - \bbonus\}$
%   \ENDFOR
%   \ENDFOR
%   \STATE Evaluate policy $\piok$ and update MLE estimates $\hat p(\cdot,\cdot)$ and $\hat r(\cdot,\cdot)$
%   \ENDFOR
% \end{algorithmic}
% \end{algorithm*}

We define the maximum per-step conditional variance (conditioning is on the $(s,a)$ pair) 
%of the immediate reward $R(\cdot,\cdot)$ random variable and next-state optimal value function 
 for a particular MDP as $\ComplexityQ$:
\begin{align}
\label{main:Complexity}
& \ComplexityQ \stackrel{def}{=} \max_{s,a,t}  \( \Var R(s,a) + \Var\limits_{s^+ \sim p(s,a)} \vtrue{t+1} (s^+) \)
\end{align}
where $R(s,a)$ is the reward random variable in $(s,a)$.
This definition is identical to the environmental norm \cite{Maillard14} but 
here we will generally refer to it as the maximum conditional value variance, in order 
to connect with other work which explicitly bounds the variance. 
%We also define %the maximum return under any policy and any realization of the rewards and transitions within an episode as \MaxReturn{}. In other words, 
%\MaxReturn{}, an upper bound on the possible total reward collected within an episode. 
% Note that both $\ComplexityQ$ is a property of a particular MDP. 

We introduce the algorithm \emph{Episodic Upper Lower Exploration in Reinforcement learning} (\Alg{}) which adopts the paradigm of ``optimism under uncertainty'' to conduct exploration. Recent work \cite{Dann15,Dann17,Azar17} 
has demonstrated how the choice of the exploration bonus is critical to enabling tighter \textit{problem-independent} performance bounds. Indeed minimax worst case regret bounds 
have been obtained by using a Bernstein-Friedman-type reward bonus defined over an empirical 
quantity related very closely to the conditional value variance $\ComplexityQ$, plus 
an additional correction term necessary to ensure optimism~\cite{Azar17}. 

Similarly, in our algorithm we use a bonus that combines an empirical Bernstein type inequality 
for estimating the $\ComplexityQ$ conditional variance, coupled with a different correction 
term which explicitly accounts for the value function uncertainty. We provide pseudocode for \Alg{} which details the main procedure in Figure \ref{main:AlgorithmMainText}.
Notice that \Alg{} has the same computational complexity as value iteration.

\section{Main Result}
\label{sec:MainResult}

Now we present our main result, which is a problem-dependent high-probability regret upper bound for \Alg{} in terms of the underlying max conditional variance $\ComplexityQ$ and maximum return. Crucially, \Alg{} is \textbf{not} provided with  $\ComplexityQ$ 
and the value of the max return. We also prove a worst-case guarantee that matches the established \cite{OV16,Jaksch10} lower bound of $\Omega(\sqrt{HSAT})$ in the dominant term. We introduce the following definition:

 \begin{definition}[Max Return]
 \label{def:MaxReturn}
 We define as $\MaxReturn \in \R$ the maximum (random) return in an episode upon following any policy $\pi$ from any starting state $s_0$, i.e., the deterministic upper bound to:
 \begin{equation}
 \sumt{1} R(s_t,\pi(s_t)) \leq \MaxReturn, \quad \forall \pi,s_0.
 \end{equation}	
 where the states $s_1,\dots,s_H$ are the (random) states generated upon following the trajectory identified by the policy $\pi$ from $s_0$.
 \end{definition}

%\subsection{Problem Dependent Regret Bound}

%Now we present our main result, which is a problem-dependent high-probability regret upper bound for \Alg{} and a worst-case guarantees that matches the established \cite{OV16,Jaksch10} lower bound of $\Omega(\sqrt{HSAT})$ in the dominant term. Before this, 

\MainResultMainText{\label{thm:MainResultMainText}}
{\label{main:ProblemDependentRegretBound}}{\label{main:ColtCaseRegretBound}}

While the maximum conditional variance $\ComplexityQ$ is always upper bounded 
by $\MaxReturn{}$ if rewards are positive and bounded, we include both forms of 
regret bound for two reasons. First, the second bound is tighter than naively 
upper bounding $\ComplexityQ \leq \MaxReturn^2$ by a factor of $H$. Second, 
we will shortly see that both quantities can provide insights 
into which instances of MDP domains can have lower regret. 

In addition, since the rewards are in $[0,1]$, we immediately have that $\MaxReturn^2 \leq H^2$, 
and thereby obtain a worst-case regret bound expressed in the following corollary: 
\begin{corollary}
\label{cor:WorstCaseRegret}
With probability at least $1-\delta$ the regret of \Alg{} is bounded for any time $T\leq KH$ by 
\begin{equation}
\label{main:WorstCaseRegretBound}	
\tilde O\(\sqrt{HSAT} + \LowerOrderTerm \).
\end{equation}
\end{corollary}
This matches in the dominant term the minimax regret problem independent 
bounds for tabular episodic RL settings~\cite{Azar17}. Therefore, the importance of 
our theorem \ref{thm:MainResultMainText} lies in providing problem dependent bounds (equation \ref{main:ProblemDependentRegretBound},\ref{main:ColtCaseRegretBound}) while simultaneously 
matching the existing best worst case guarantees (equation \ref{main:WorstCaseRegretBound}). We shall shortly show that our results help address a recent open question on 
the performance dependence of episodic MDPs on the horizon~\cite{jiang2018open}.

\subsection{Sketch of the Theoretical Analysis}
\label{sec:TheoreticalAnalysisSketch}
We devote this section to the sketch of the main point of the regret analysis that yields problem dependent bounds. Readers that wish to focus on how our results yield insight into the 
complexity of solving different problems may skip ahead to the next section. 
% We give an outline here of the main line of reasoning, focusing on how the leading order regret term gives rise to problem dependent bounds. 
Central to the analysis is the relation between the agent's optimistic MDP and the ``true'' MDP. A more detailed overview of the proof is given in section \ref{sec:AppendixOverview} of the appendix, while the rest of the appendix presents the detailed analysis under a more general framework.

\paragraph{Regret Decomposition}
Denote with $\Esa$ the expectation taken along the trajectories identified by the agent's policy $\piok$. A standard regret decomposition is given below (see \cite{Dann17,Azar17}):
\begin{align*}
& \textsc{Regret}(K) \leq \sumall \Esa \Biggm( \underbrace{\reward{s}}_{\substack{\textsc{Reward} \\ \textsc{Estimation} \\ \textsc{and Optimism}}} \\ & + \underbrace{\bonus{s}}_{\substack{\textsc{Transition} \\ \textsc{Dynamics} \\ \textsc{Optimism}}} \\  & + \underbrace{\estimation{s}}_{\substack{\textsc{Transition} \\ \textsc{Dynamics} \\ \textsc{Estimation}}} \\
& + \underbrace{\lowerorder{}}_{\substack{\textsc{Lower} \\ \textsc{Order} \\ \textsc{Term}}}  \Biggm) 
\numberthis{\label{main:RegretDecomposition}}	
\end{align*}
Here, the ``tilde'' quantities $\tilde r$ and $\tilde p$ represent the agent's optimistic estimate. 
Of the terms in equation \ref{main:RegretDecomposition}, the ``Transition Dynamics Estimation'' and ``Transition Dynamics Optimism'' are the leading terms to bound as far as the regret is concerned. The former is expressed through MDP quantities (i.e, the true transition dynamics $\potruea{s}{t}$ and the optimal value function $\vtrue{t+1}$) and hence it can be readily bounded using Bernstein Inequality, giving rise to a problem dependent regret contribution. More challenging is to show that a similar simplification can be obtained for the ``Transition Dynamics Optimism'' term which relies on the agent's optimistic estimates $\poa{s}{t}$ and $\vo{t+1}$.
\paragraph{Optimism on the System Dynamics}
Said term $\bonus{s}$ represents the difference between the agent's imagined (i.e., optimistic) transition $\poa{s}{t}$ and the maximum likelihood transition $\pohata{s}{t}$ weighted by the next-state optimistic value function $\vo{t+1}$. By construction, this is the exploration bonus which incorporates an estimate of the conditional variance over the value function. This bonus reads:
\begin{align}
\label{main:BonusTheoreticalExplanation}
& \substack{\textsc{Transition} \\ \textsc{Dynamics} \\ \textsc{Optimism}} = \substack{\textsc{Explo-} \\ \textsc{ration} \\ \textsc{Bonus}} \approx \ensuremath{\underbrace{\overbrace{\sqrt{\frac{\Var_{s\sim\pohata{s}{a}} \vo{t+1} }{n_k(s,a)}}}^{\substack{\textsc{Dominant Term} \\ \textsc{of Exploration Bonus}}} + \frac{H}{n_k(s,a)}}_{\substack{\textsc{Empirical Bernstein evaluated} \\ \textsc{with Empirical Value Function}}} \\
& + \underbrace{\(\frac{\|\vo{t+1}-\vp{t+1} \|_{\pohata{s}{a}}}{\sqrt{\nsaa{s}{t}{k}}} +\frac{H}{\nsaa{s}{t}{k}}\)}_{\textsc{Correction Bonus}}}
\end{align}
In the above expression the ``Correction Bonus'' is needed to ensure optimism because the ``Empirical Bernstein'' contribution is evaluated with the agent's estimate $\vo{t+1}$ as opposed to the real $\vtrue{t+1}$.
If we assume that $\|\vo{t+1} - \vp{t+1} \|_{\pohata{s}{t}}$ shrinks quickly enough, then the ``Dominant Term'' in equation \ref{main:BonusTheoreticalExplanation} is the most slowly decaying term with a rate $1/\sqrt{n}$. If that term involved the true transition dynamics $\potruea{s}{t}$ and value function $\vtrue{t+1}$ (as opposed to the agent's estimates $\pohata{s}{t}$ and $\vo{t+1}$) then problem dependent bounds would follow in the same way as they could be proved for the ``Transition Dynamics Estimation''.
Therefore we wish to study the relation between such ``Dominant Term'' evaluated with the agent's MDP estimates vs the MDP's true parameters. 
\paragraph{Convergence of the System Dynamics in the Dominant Term of the Exploration Bonus} 
Theorem $10$ of \cite{MP09} gives the high probability statement:
\begin{align}
\label{main:MauerAndPontil}
 \Bigg| \sqrt{\Var_{\pohata{s}{t}} \vtrue{t+1}} - \sqrt{\Var_{\potruea{s}{t}} \vtrue{t+1}} \Bigg| \lessapprox \frac{H}{\nsaa{s}{t}{k}}
\end{align}
to quantify the rate of convergence of the empirical variance using the true value function (this leads to the empirical version of Bernstein's inequality). Next, two basic  computations yield:
\begin{align*}
   \Bigg| \sqrt{\Var_{\pohata{s}{t}} \vtrue{t+1}} - \sqrt{\Var_{\pohata{s}{t}} \vo{t+1}} \Bigg|  \\
   \leq \|\vo{t+1}-\vtrue{t+1} \|_{\pohata{s}{a}}  \leq \|\vo{t+1}-\vp{t+1} \|_{\pohata{s}{a}}
   \numberthis{\label{main:VarianceConvergence}}
\end{align*}
Together, equation \ref{main:MauerAndPontil} and \ref{main:VarianceConvergence} quantify the rate of convergence of $\Var_{s\sim\pohata{s}{a}} \vo{t+1}$ to $\Var_{s\sim\potruea{s}{a}} \vtrue{t+1}$, yielding the following upper bound for the dominant term of the exploration bonus:
\begin{align*}
& \substack{\textsc{Dominant} \\ \textsc{Term of } \\ \textsc{Exploration} \\ \textsc{Bonus} }  = \sqrt{\frac{\Var_{\pohata{s}{t}} \vo{t+1}}{\nsaa{s}{t}{k}}}
\stackrel{}{\lessapprox} \underbrace{\sqrt{\frac{\Var_{\potruea{s}{t}} \vtrue{t+1}}{\nsaa{s}{t}{k}}}}_{\substack{\textsc{Gives Problem} \\ \textsc{Dependent Bounds}}} \\ 
& + \underbrace{\frac{H}{\nsaa{s}{t}{k}} + \frac{\|\vo{t+1} - \vp{t+1} \|_{\pohata{s}{t}}}{\sqrt{\nsaa{s}{t}{k}}}}_{\textsc{Shrinks Faster}}
\numberthis{\label{main:MainExplanation}}
\end{align*}
In words, we have decomposed the ``Dominant Term of the Exploration Bonus'' (which is constructed using the agent's available knowledge) as a problem-dependent contribution (that is equivalent to Bernstein Inequality evaluated as if the model was known) and a term that accounts for the distance between the the true and empirical model, expressed as (computable) upper and lower bounds 
on the value function. This additional term shrinks faster that the former. It is precisely this  ``Correction Bonus'' that we use in equation \ref{main:BonusTheoreticalExplanation} and in the definition of the Algorithm itself.
\paragraph{What gives rise to problem dependent bounds?}
Our analysis highlights \Alg{} uses a Bernstein inequality on the empirical estimate of the conditional variance of the next state values, 
%was using Bernstein Inequality evaluated with precise domain knowledge of $\potruea{s}{t}$ and $\vtrue{t+1}$ to construct problem dependent confidence intervals 
%with a correction term that guarantees optimism and is rapidly decaying. This 
 with a correction term $\|\vo{t+1} - \vp{t+1} \|_{\pohata{s}{t}}$ function of the inaccuracy of the value function estimate at the next-step states re-weighted by their relative importance as encoded in the experienced transitions $\pohata{s}{t}$. Said correction term is of high value only if the successor states do not have an accurate estimate for the value function \emph{and} they are going to be visited with high probability. A pigeonhole argument guarantees that this situation cannot happen for too long ensuring fast decay of $\|\vo{t+1} - \vp{t+1} \|_{\pohata{s}{t}}$ and therefore of the whole ``Correction Bonus'' of eq. \ref{main:BonusTheoreticalExplanation}. %As this this decay rate is also problem dependent because it depends on the problem-dependent confidence intervals at the successor states.

Our primary analysis yields a regret bound that scales directly with the 
(unknown to the algorithm) problem-dependent $\ComplexityQ$ max conditional variance of 
the next state values. We further extend 
this to a bound directly in terms of the max returns $\MaxReturn$ by using a law of total 
variance argument. 

Notice that such considerations and results would not be achievable by a naive application of an Hoeffding-like inequality as the latter would put equal weight on all successor states, but the accuracy in the estimation of $\vtrue{t+1}(\cdot)$ only shrinks in a way that depends on the visitation frequency of said successor states as encoded in $\pohata{s}{t}$. The key to enable problem dependent bound is, therefore, to re-weight the importance of the uncertainty on the value function of the successor states by the corresponding visitation probability, which Bernstein Inequality implicitly does. 

There exist other algorithms (e.g.  \cite{Dann15,Azar17}) which are based on Bernstein's inequality but to our knowledge they have not been analyzed in a way that provably yields problem dependent bounds as those presented here. 

%In view of these consideration we believe that it might be possible to derive some form of problem dependent bound for similar works based on Bernstein Inequality like \cite{Azar17,KWY18}, although this would need to be checked on a case-by-case basis. 

%We wish to relate this to the corresponding ``true'' Bernstein term which is function of \textbf
%Between the (Empirical) Bernstein inequality with the optimistic value function and the bonus \critical{CorrectionMainText{}}, the most slowly decaying term (ignoring log-factors) is $\sqrt{\frac{\Var_{\hat p} \vo{t+1}}{\nsaa{s}{t}{k}}}$. The key observation is the following: if we could evaluate $\sqrt{\frac{\Var_{p} \vtrue{t+1}}{\nsaa{s}{t}{k}}}$ instead, that is, the ``real'' Bernstein inequality with $\potruea{s}{t}$ and $\vtrue{t+1}$ then problem dependent bounds would follow because both $\potruea{s}{t}$ and $\vtrue{t+1}$ do not depend on the agent. Thus we wish to study if $\sqrt{\Var_{\hat p} \vo{t+1}} - \sqrt{\Var_{p} \vtrue{t+1}}$ converges, and at what rate. In the appendix \critical{where} we are able to express such difference as

\section{Horizon Dependence in Dominant Term}
\label{sec:colt_conj}
%Before discussing the implications of our above result in several common 
%classes of Markov decision processes with particular different forms 
%f structure, 
In this section we show that our result can help address a 
recently posed open question in the learning theory community~\cite{jiang2018open}.
The question posed centers on the whether there should exist a necessary 
dependence of sample complexity and 
regret lower bounds on the planning horizon $H$ 
for episodic tabular MDP reinforcement learning tasks. 
Existing lower 
bound results for sample complexity~\cite{Dann15} depend on the horizon, as do the best existing minimax regret bounds under asymptotic assumptions~\cite{Azar17}. However,  
such results have been derived under the common assumption of reward 
uniformity, that per-time-step rewards are bounded between 0 and 1, yielding a total value bounded by 0 and $H$. \citet{jiang2018open} instead pose a more 
general setting, in which they assume that 
the rewards are positive and $\sum_{h=1}^H r_h \in [0,1]$ holds almost surely: 
note the standard setting of reward uniformity can be expressed in this setting by 
first normalizing all rewards by dividing by $H$. The authors then 
ask that if in this new, more general setting of tabular episodic RL there is necessarily a dependence on the planning horizon in 
the lower bounds. Note that in this setting, the prior existing lower bounds on the sample complexity~\cite{Dann15} would yield no dependence on the horizon. 

For our work, the setting of Jiang and Agarwal immediately implies that
\begin{equation}
    0 \leq \vtrue{t}(s) \leq \MaxReturn \leq 1, \; \forall (s,t)\in\mathcal S\times [H].
    \end{equation}
Further, since $\vtrue{t}(s) \leq 1$ and $r(s,a) \geq 0$ we must have $r(s,a) \in [0,1]$, which is the assumption of this work. Therefore our main result (theorem \ref{thm:MainResultMainText}) applies here. Recalling that $T = KH$, we obtain an upper bound on regret as
\small
\begin{equation}
 \tilde O \left( \sqrt{SAK} + \LowerOrderTerm \right) \label{main:sparserewardbound}. 
\end{equation}
\normalsize   
Note that the planning/episodic horizon $H$ does not appear in the dominant regret term which scales polynomially with the number of episodes\footnote{This is stronger than scaling polynomially with the time $T$} $K$, and only appears in transient lower order terms that are independent of $K$.

%Notice that we express the leading order term as a \emph{function of the number of episodes} $K$ and that the planning/episode horizon $H$ only appears in the lower order term (which is independent of $K$) and the polylog term. 

In other words, \textbf{up to logarithmic dependency and transient terms, we have an upper bound 
on episodic regret that is independent of the horizon $H$}. This result answers part 
of Jiang and Agarwal's open question: for their setting, the 
regret primarily scales independently of the horizon. %Note that in contrast \citet{Azar17} has an additional secondary term that depends on $\tilde O( H^2 \sqrt{K})$, and it is not straightforward to predict how both terms that depend on $H$ and $K$ would change in this setting.

%In other words, \textbf{the regret leading order term is independent of the horizon $H$ even when the regret is expressed in terms of episodes} 
%\critical{I'd like to emphasize this}. 

% This result is not immediately implied by prior work. In particular, 
% to our knowledge the tightest prior (and indeed, minimax) regret bound results for the standard regularity assumption of per time steps $r_h \in [0,1]$ (from~\cite{Azar17}) is  
% \small
% \begin{equation}
%     \label{eqn:azar}   O\( \polylog(S,A,K,H,1/\delta) \(H \sqrt{SAK} + H^2 S^2 A + H^{3/2} \sqrt{K} \) \).
% \end{equation}
% \normalsize
% If their algorithm is provided with the knowledge that the total rewards and 
% maximum value function is bounded by 1 (instead of $H$), then our preliminary 
% analysis suggests a factor of $H$ can be saved. However, this still yields a 
% regret that grows with a dependence on the planning/episodic horizon $H$ that scales 
% polynomially with the number of episodes (the right most term in Equation~\ref{eqn:azar}, 
% which would become 
% \begin{equation}
%   O\(  \polylog(S,A,K,H,1/\delta) \sqrt{KH}\).
% \end{equation} 
Surprisingly, while \Alg{} uses a common problem-agnostic bound on the maximum possible optimal value function ($H$), it does not need to be provided with information about the domain-dependent  
maximum possible value function to attain the improved bound in the setting of the COLT conjecture of \citet{jiang2018open}. 

 It remains an open question whether we could further avoid either a dependence  
  on the planning horizon in the transient terms as well as obtaining a PAC result. In Appendix \ref{sec:COLTconjecture} we further discuss this direction. However, these results are promising: they suggest that the hardness of learning 
 in sparse reward, and long horizon episodic MDPs may not be fundamentally much harder 
 than shorter horizon domains if the total reward is bounded. 
%  We next demonstrate 
%  that several classes of episodic MDPs have additional structure that also guarantee
%  better performance (lower regret bounds), even though our algorithm need not be aware of such structure. 

\section{Problem dependent bounds}
\label{sec:ProblemDependentBounds}

We now focus on deriving regret bounds for selected MDP classes that are very common in RL. We emphasize that such setting-dependent guarantees are obtained with the same algorithm that is not informed %\emph{agnostic}
of a particular MDP's values of  $\ComplexityQ{}$ and $ \MaxReturn$. Although the described settings share common features and are sometimes subclasses of one another, they are in separate subsections due to their important relation to the past literature and their practical relevance. Importantly, they are \emph{all characterized by low \ComplexityQ{}}.

\subsection{Bounds using the range of optimal value function}
To improve over the worst case bound
% \begin{equation}
% \tilde O (DS\sqrt{AT})	
% \end{equation}
in infinite horizon RL there have been approaches that aim at obtaining stronger problem dependent bounds if the value function does not vary much across different states of the MDP. %The intuition is that in such case every ``mistake'', i.e., suboptimal action can only lead to a state $s''$ whose value function $\Vstar(s'')$ cannot be much less than that in a state $s'$ reachable under the best action. In other words $\rng \Vstar = \max_{s',s''}\Vstar(s') - \Vstar(s'')$ is an upper bound on the magnitude of the mistakes. 
If $\rng \Vstar$ is smaller than the worst-case (either $H$ or $D$ for the fixed horizon vs recurrent RL), the reduced variability in the expected return suggests that performance can benefit from  constructing tighter confidence intervals. This is achieved by \citet{Bartlett09} by providing this range to their algorithm \textsc{Regal}, achieving a regret bound:
\begin{equation}
\tilde O(\Phi S \sqrt{AT})
\end{equation}
where $\Phi \geq \rng \Vstar$ is an overestimate of the optimal value function range and is an input to the algorithm described in that paper. This means that if domain knowledge is available and is supplied to the algorithm the regret can be substantially reduced. This line of research was followed in \cite{Fruit18} which derived a computationally-tractable variant of \textsc{Regal}. However, they still require knowledge of a value function range upper bound $\Phi \geq \rng \Vstar$. Specifying a too high value for $\Phi$ would increase the regret and a too low value would cause the algorithm to fail.

Our analysis shows that, in the episodic setting,  it is possible to achieve at least the same but potentially much better level of performance \emph{without knowing the optimal value function range}. % to construct the exploration bonus. 
This follows as an easy corollary of our main regret upper bound (Theorem \ref{thm:MainResultMainText}) after bounding the environmental norm, as we discuss below.

Let $\mathcal{S}_{s,a}$ be the set of immediate successor states after one transition from state $s$ upon taking action $a$ there, that is, the states in the support of $p(\cdot \mid s,a)$ and define 
\begin{equation}
    \Phi_{succ} \stackrel{def}{=} \max_{s,a}\rng\limits_{s^+ \in   \mathcal{S}_{s,a}} \vtrue{t+1} (s^+)
\end{equation} 
as the maximum value function range \emph{when restricted to the immediate successor states}. Since the variance is upper bounded by (one fourth of) the square range of a random variable we have that:
\begin{align*}
& \ComplexityQ \stackrel{def}{=} \max_{s,a,t} ( \Var \left( R(s,a) \mid (s,a) \right) + \Var_{s^+ \sim p(s,a)} \vtrue{t+1} (s^+)) \\
 & \leq \max_{s,a,t}  \Big( 1 + (\rng_{s^+ \in   \mathcal{S}_{s,a}} \vtrue{t+1} (s^+))^2 \Big) \leq 1+ \Phi^2_{succ}.
\end{align*}
This immediately yields:
\begin{corollary}[Bounded Range of $\Vstar$ Among Successor States]
\label{cor:BoundedRange}
With probability at least $1-\delta$, the regret of \Alg{} is bounded by:
\begin{equation}
		\tilde O( \Phi_{succ}\sqrt{SAT} + \LowerOrderTerm).	
\end{equation}
\end{corollary}
A few remarks are in order:
\begin{itemize}
\itemsep0em 
\item \Alg{} does not need to know the value of $\Phi_{succ}$ or of the environmental norm or of the value function range to attain the improved bound; 
\item $\Phi_{succ}$ can be much smaller than $\rng \Vstar$ because it is the range of $\Vstar$ restricted to few successor states as opposed to across the whole domain, and therefore it is always smaller than $\Phi$, in other words:
%\begin{equation}
 $   \Phi \geq \rng \Vstar \geq \Phi_{succ}$.
%\end{equation}
\item \cite{Bartlett09,Fruit18} consider the more challenging infinite horizon setting, while our results holds for fixed horizon RL.
\end{itemize}

\subsection{Bounds on the next-state variance of \texorpdfstring{\vtrue{}}{} and  empirical benchmarks}
\label{sec:SparseMDPs}
%As explained in theorem \ref{thm:MainResultMainText}, the 
%which are common benchmarks from previous literature.

%Of course, t
% As we discuss throughout this section, the environmental norm together with our new result in Theorem \ref{thm:MainResultMainText} can help 
% explain the hardness of subclasses of MDP domains. Before considering another important subclass in the next subsection, we first 
%In this subsection we wish 
%to highlight that 
%for specific MDP classes analyzed in our submission, but equally important is how the
The environmental norm also can empirically characterize the hardness of RL in single problem instances.
This was one of the key contributions of the work that introduced the environmental norm \cite{Maillard14}, 
which evaluated the environmental norm for a number of common RL benchmark simulation tasks including mountain car, pinball, taxi, bottleneck, inventory and red herring %(see the table in paragraph 3.2 in \cite{Maillard14}). 
In these domains the environmental norm is correlated 
with the complexity of reinforcement learning in these environments, as evaluated empirically.
%Indeed, these theoretically-justified guarantees are going to be especially useful if the environmental norm is numerically small for specific problem instances. Fortunately, in all the settings considered by \cite{Maillard14}, all of which have been common benchmarks in prior literature,
Indeed, in these settings, the environmental norm is often much smaller then the maximum value function range, which can itself be much smaller than the worst-case bound $D$ or $H$. Our new results provide solid theoretical justification for the observed empirical savings.

%Another intriguing application of the 
This measure of MDP complexity also intriguingly allows us to gain more insight on 
another important simulation domain, chain MDPs like that in Figure
\ref{fig:HardMDP}. Chain MDPS have been considered a canonical example of challenging hard-to-learn RL domains, since naive strategies like $\epsilon$ greedy can take exponential time to attain satisfactory performance. By setting for simplicity $N\stackrel{def}{=}S=H$ 
%efficient exploration strategies provide domain-agnostic regret bounds of 
% $\tilde O (\sqrt{N^4 AK } + \sqrt{N^4 K} + N^4 A)$ (e.g. \cite{Azar17}).
%(for example \cite{Azar17,KWY18} should incur $\tilde O(\sqrt{HSAT}) = \tilde O(\sqrt{AN^2T})$  regret in the leading order term. 
\Alg{} provides an upper regret bound of
$\tilde O (\sqrt{NAK} + \dots)$ that is substantially tighter than a worst case bound $\tilde O(\sqrt{N^3AK} + \dots)$, at least for large $K$. 
%(leading order term only) $\tilde O(\sqrt{AT})$ leading order regret. 
This is intriguing because it suggests pathological MDPs may be even less common than expected. More details about this example are in appendix \ref{app:Chain}.
\begin{figure}[H]
\begin{center}
\resizebox{0.5\textwidth}{!}{
\begin{tikzpicture}[->, >=stealth', auto, semithick, node distance=2cm]
\tikzstyle{every state}=[fill=red,draw=white,thick,text=black,scale=0.8]
\node[state]    (A)[        minimum size=1.25cm]{$s_1$};
\node[state]    (B)[      right of=A,minimum size=1.25cm]   {$s_2$};
\node[state]    (C)[      right of=B,minimum size=1.25cm]   {$s_3$};
\node[state]    (S)[      right of=C,minimum size=1.25cm,draw=none,fill=none]   {$\cdots$};
\node[state]    (D)[      right of=S,minimum size=1.25cm]   {$s_{N-2}$};
\node[state]    (E)[      right of=D,minimum size=1.25cm]   {$s_{N-1}$};
\node[state]    (F)[      right of=E,minimum size=1.25cm]   {$s_{N}$};
\path
(A) edge[loop left]     node{$r = \frac{1}{4N}$}      (A)
(A) edge[bend left,dashed]     node{$1-\frac{1}{N}$}     (B)
(A) edge[loop above,above,dashed]     node{$\frac{1}{N}$}     (A)
(B) edge[bend left]    node{$1$}     (A)
(B) edge[bend left,dashed]     node{$1-\frac{1}{N}$}     (C)
(B) edge[loop above,above,dashed]     node{$\frac{1}{N}$}     (B)
(C) edge[bend left]    node{$1$}     (B)
(C) edge[bend left,dashed]     node{$1-\frac{1}{N}$}     (S)
(S) edge[bend left]    node{$1$}     (C)
(S) edge[bend left,dashed]     node{$1-\frac{1}{N}$}     (D)
(D) edge[bend left]    node{$1$}     (S)
(C) edge[loop above,above,dashed]     node{$\frac{1}{N}$}     (C)
(D) edge[bend left,dashed]     node{$1-\frac{1}{N}$}     (E)
(D) edge[loop above,above,dashed]     node{$\frac{1}{N}$}     (D)
(E) edge[bend left]    node{$1$}     (D)
(E) edge[bend left,dashed]    node{$1-\frac{1}{N}$}     (F)
(D) edge[loop above,above,dashed]     node{$\frac{1}{N}$}     (D)
(F) edge[bend left]    node{$1$}     (E)
(F) edge[loop right]     node{$r = 1$}      (F)
(E) edge[loop above,above,dashed]     node{$\frac{1}{N}$}     (E)
(F) edge[loop above,above,dashed]     node{$\frac{1}{N}$}     (F);
\end{tikzpicture}
}
\vspace{-.2cm}
\caption{Classical hard-to-learn MDP}
\label{fig:HardMDP}
\end{center}
\end{figure}
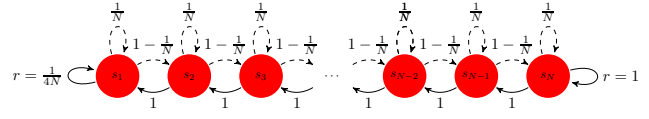
\subsection{Stochasticity in the system dynamics}
In this section we consider two important opposite classes of problems: deterministic MDPs and MDPs that are highly stochastic in that the successor state is sampled from a fixed distribution. These bounds are also a direct consequence of Theorem \ref{thm:MainResultMainText} and can be deduced from Corollary \ref{cor:BoundedRange}. 
%but are discussed separately due to their importance.
%Of course, there exists a wide spectrum of MDPs ranging from deterministic to fully stochastic ones; here for simplicity we consider only the limit cases.
 
\paragraph{Deterministic domains}
\label{sec:DeterministicDomains}
Many problems of practical interest, for example in robotics, have low stochasticity, and this immediately yields low value for $\ComplexityQ$. As a limit case, we consider domains with deterministic rewards and dynamics models. 
An agent designed to learn deterministic domains only needs to experience every transition \emph{once} to reconstruct the model, which can take up to $O \left(SA\right) $ episodes with a regret at most $O(SAH)$\cite{WR13}.

%For \Alg{} an improved bound is straightforwardly derived because  
Note in deterministic domains $\ComplexityQ{} = 0$.
Therefore if \Alg{} is run on \emph{any} deterministic MDP then the regret expression exhibits a $\log(T)$ dependence. This is a substantial improvement over prior RL regret bounds for problem-independent settings all have at least a $\sqrt{T}$ dependence. Further, a refined analysis (Appendix Section \ref{sec:DeterministicDomainAnalysis}) shows \Alg{} is close to the lower bound except for a factor in the horizon and logarithmic terms: 
\DeterministicDomainRegret{\label{prop:DeterministicDomainRegret}} 
%The result of the above proposition is important because it supports the idea that \Alg{} can inherit near-optimal performance even in domains very different from the pathologically hard MDP for which it is designed to do well. 

\paragraph{Highly mixing domains}
Recently, \cite{Zanette18a} show that it is possible to design an algorithm that can switch between the MDP and the contextual bandit framework while retaining near-optimal performance in both without being informed of the setting. They consider mapping contextual bandit to an MDP whose transitions to different states (or contexts) are sampled from a fixed underlying distribution over which the agent has no control. 

The Bandit-MDP considered in \citet{Zanette18a} is an environment with high stochasticity (the MDP is highly mixing since every state can be reached with some probability in one step). Since the transition function is unaffected by the agent, an easy computation yields $\rng \Vstar_t \leq 1$, as replicated in Appendix \ref{sec:AppendixContextualBandits}. A regret guarantee in the leading order term of order $\tilde O(\sqrt{SAT})$ for \Alg{} which matches the established lower bound for tabular contextual bandits \cite{BC12} follows from corollary \ref{cor:BoundedRange}. % Note that the agent is still interacting with an $H$-horizon MDP and it is not aware that it cannot influence the system dynamics (the ``context'', i.e., the next state, is sampled from a fixed distribution in the contextual bandit problem). 
This is useful since in many practical applications it is unclear in advance if the domain is best modeled as a bandit or a sequential RL problem. 
Our results improve over \cite{Zanette18a} since \Alg{} has better worst-case guarantees by a factor of $\sqrt{H}$. Our approach is also feasible with next-state distributions that have zero or near zero mass over some of the next states: in contrast to prior work, the inverse minimum visitation probability does not show up in our analysis.

\section{Related Literature}
\label{sec:rellit}

In infinite horizon RL, prior empirical evaluation of $\ComplexityQ{}$ in \cite{Maillard14}  
has shown encouraging performance in a number
of common benchmarks that $\ComplexityQ$  has small value and
its size relates to the hardness of solving the RL task. The
theoretical results provides a regret bound whose leading 
order term is $\tilde{O} \( \frac{1}{p_0} DS \sqrt{\ComplexityQ A T} \) $ (where $p_0$ is the minimum (non-zero) transition probability),  and generally does not improve over worst case analysis for the infinite horizon setting. Our algorithm operates in an easier setting (finite horizon) where it can improve over the worst case, but it is an open question whether an improvement is possible in infinite horizon.

Our connection with  \cite{Bartlett09,Fruit18,Zanette18a} has already been described. Here we focus on the remaining literature. We again note that the infinite horizon setting offers a number of important complexities and comparisons to the finite horizon setting (as considered here) cannot be directly made; however, as some of the closest related work lies in the infinite horizon setting, we briefly discuss it here. 
\begin{itemize}[leftmargin=*]
\itemsep0em 
\item Bounds that depend on gap between policies: In the infinite horizon setting, 
\cite{EMM06} has bounds dependent on the minimum gap in the optimal state action values between the best and second best action, and  \textsc{Ucrl2} \cite{Jaksch10} has bounds as function of the gap in the difference in the average reward between the best and second best policies. Such gaps reflect an interesting alternate structure in the problem domain: note that in prior work as these gaps become arbitrarily small, the bound approaches infinity: even in such settings, if the next state variance is small, our bound will stay bounded. An interesting future direction is to consider bounds that consider both forms of structure.
\item Regret bounds with value function approximation: In finite horizon settings, \cite{OV14} uses the Eluder dimension as a measure of the dimensionality of the problem and \cite{JKAL17} proposes the Bellman rank to measure the learning complexity.  Such measures capture a different notion of hardness than ours and do not match the lower bound in tabular settings.
\item Infinite horizon results with additional properties of the transition model: the most closely related work to ours is \cite{TM18} who also develop tighter regret bounds as a function of the next-state variance, but for infinite horizon settings. Exploration in such settings is nontrivial and the authors leverage an important assumption of ergodicity (which has also been considered in \cite{AO06}. Specifically the agent will visit every state regardless of the current policy, and the rate of this mixing appears in the regret bound. An interesting and nontrivial question is whether our results can be extended to this setting without additional assumptions on the mixing structure of the domain.

\end{itemize}
%In particular, a low bellman rank in tabular MDPs is often associated to a low rank on the transition dynamics. %We think that exploring problem dependent bounds when function approximation is required is a very interesting and nontrivial direction for future work. 
%The increased generality of the analysis in \cite{JKAL17} comes at the price of computational tractability \cite{DJKALS18} and therefore it is not comparable to the proposal here.

A natural additional question is whether prior algorithms also inherit strong problem dependent rounds. Indeed, 
recent work by \cite{Dann15,Azar17} has also used the variance of the value function at the next state in their analysis, though their final results are expressed as worst case bounds. %, it is natural to wonder if their results might immediately imply similarly strong problem dependent bounds as our \Alg. 
However, the actual bonus terms used in their algorithms are distinct from our bonus terms, perhaps most significantly in that we maintain and leverage point-wise upper and lower bounds on the value function. While it is certainly possible that their algorithms or others already attain  some form of problem dependent performance, they have not been analyzed in a way that yields problem dependent bounds. This is a technical area, and performing such analyses is a non-trivial deviation from a worst-case analysis. For example, the current worst case bounds from \citet{Azar17} yield a regret bound that scales as  $\tilde{O}( \sqrt{HSAT} + \sqrt{H^2 T} + S^2 A H^2)$ and it is a non-trivial extension to analyze how each of these terms might change to reflect problem-dependent quantities. One of our key contributions is an analysis of the rate of convergence of the empirical quantities about properties of the underlying MDP to the real ones in determining the regret bound.

\section{Future Work and Conclusion}
In this paper we have proposed \Alg{}, an algorithm for episodic finite MDPs that matches the best known worst-case regret guarantees while provably obtaining much tighter guarantees if the domain has a small variance of the value function over the next-state distribution $\ComplexityQ{}$ or a small bound in the possible achievable reward. \Alg{} does not need to know these MDP-specific quantities in advance. % the value variance in advance.
%Prior problem-dependent bounds that used value function overestimates $\Phi$ of the range of the value function $\rng \Vstar$  \cite{Bartlett09,Fruit18} across all states operate in the more unforgiving setting of infinite horizon RL and are not comparable to the proposal here. Nonetheless, for finite horizon Our work highlights the importance of the range of the optimal value function restricted to successor states $\Phi_{succ}$ in corollary \ref{cor:BoundedRange} as an upper bound to the per-step variance $\ComplexityQ{}$. 
%of our main result, theorem \ref{thm:MainResultMainText}. 
%More precisely:
%\begin{equation}
 %$   \sqrt{\ComplexityQ{}}\leq \Phi_{succ} \leq \rng \Vstar \leq \Phi$. 
 We show that \ComplexityQ{} is low for a number of important subclasses of MDPs, including: MDPs with sparse rewards, (near) determinisitic MDPs, highly mixing MDPs (such as those closer to bandits) and some classical empirical benchmarks. We also show how our result helps answer a recent open learning theory question about the necessary dependence of regret results on the episode horizon. Possible interesting directions for future work would be to examine problem-dependent bounds in the infinite horizon setting, incorporate a gap-dependent analysis, or see if such ideas could be extended to the continuous state setting.

\section*{Acknowledgments}
The authors are grateful for the high quality feedback of the reviewers, and for the comments of Yonathan Efroni and his colleagues, including Mohammad Ghavamzadeh and Shie Mannor, who helped find a mistake in an earlier draft. This work was partially supported by the Total Innovation Fellowship program, a NSF CAREER award and an ONR Young Investigator Award. 

\bibliographystyle{icml2019}
\bibliography{rl}
\appendix
\onecolumn

%%%%%%%%%%%%%%%%%%%%%%%%%%%
% \critical{ THIS GOES TO THEO ANALYSIS OR ELSEWHERE AS MOTIVATION 
% Since the agent optimistic value function can grossly overestimate the true value function $\Vstar$ in states that are not frequently visited, overexploration may occur when the bonus is constructed with this estimate of the value function. Both \cite{Bartlett09,Fruit18} use $\Phi$ to reduce the necessary exploration bonus and attain stronger performance.
% By contrast, our approach leverages an empirical Bernstein-based exploration bonus that automatically lowers the contribution to the exploration bonus of the states that are not visited often or even not reachable in one step. This way the exploration bonus is automatically ``tweaked'' based on the problem at hand. This leads to a stronger result, because the resulting regret bound depends on the range of $\Vstar$ (denoted by $\Phi_{succ}$) across states that are actual successors (i.e., reachable from the current state) as opposed to across the whole domain ($\Phi$).}
%%%%%%%%%%%%%%%%%%%%%%%%%%%
%\printunsrtglossary[type=symbols,style=long]
\tableofcontents
\textbf{Remark on constants}: throughout the appendix we use numerical constants $c_{i,j,k}$ and $\ll = \log(SAHT/\delta)$, leaving their computation implicit.
\section{Short Proofs Missing from the Main Text}
% \subsection{Proof of Corollary \ref{cor:BoundedRange}}
% \label{app:BoundedRangeProof}
% First, we move from the variance to the range of a random variable:
% \begin{equation} 
% \Var \( R(s,a) \) + \Var \(\vtrue{t+1} (s^+) \mid (s,a) \) \leq \frac{1}{4} \( \(\rng R(s,a)\)^2 + \(\rng_{s^+ \in \mathcal{S}_{(s,a)}} V^*_{t+1}(s^+) \)^2 \)
% \numberthis{\label{main:rngUpperBound}}
% \end{equation}

% Then the proof follows from Theorem \ref{thm:MainResultMainText} and the definition of $\ComplexityQ{}$ chained with equation \ref{main:rngUpperBound}: 
% \begin{align}
% 	\ComplexityQ{} \stackrel{eq.  \ref{main:Complexity}}{\leq} \max_{s,a,t} \Var\(\Qrv_t(s,a) \mid (s,a) \) \stackrel{eq. \ref{main:DecouplingRpV} + eq. \ref{main:rngUpperBound}}{\leq} \max_{s,a,t} \frac{1}{4}\( \( \rng R(s,a)\)^2 + \( \rng_{s^+ \in \mathcal{S}_{(s,a)}}  \vtrue{t+1}(s^+)\)^2 \) \leq \frac{1}{4}\( 1+\Phi^2 \).
% \end{align}
\subsection{Euler on Chain}
\label{app:Chain}
Chain MDPs are commonly given as examples of challenging exploration domains because simple strategies like $\epsilon$-greedy can take an exponential time to learn. We now discuss an intriguing result for the chain shown in figure \ref{fig:HardMDP} which is nearly identical to a previously introduced one \cite{OV17}. Like in that domain, each episode is  
$N = H = S$ timesteps long.
The optimal policy is to go right % head for the rightmost node and 
which yields a reward of 1 for the episode. %the unitary reward before the end of the episode. 
The transition probabilities are assigned in way that scales the optimal value function so that it is of order $1$. %$\vtrue{1}(s_1) = (1-\frac{1}{N})^{N-1} \leq 1$. 
Since the rewards are deterministic we immediately obtain (up to a constant that does not depend on $N$): 
\begin{equation}
 	\ComplexityQ{} \leq \max_{s,a,t}\Var \(\Qrv(s,a) \mid (s,a)\) \leq \max_{s,a,t}\Var \(\vtrue{t+1} (s^+)\mid (s,a)\) \lesssim \frac{1}{N}\(1-\frac{1}{N}\) \leq \frac{1}{N}
 \end{equation} 
since $\vtrue{t+1} (s^+)$ is essentially dominated by a Bernoulli random variable with success parameter $(1-1/N)$ times an appropriate scaling factor of order one. Therefore \Alg{}'s regret\footnote{This is valid for any class of MDP which shares these properties, implying that the regret for the MDP in figure \ref{fig:HardMDP} can be even smaller} is dominated by a term which is 
%leading order term in the regret is of order:
%\begin{equation}
%\label{eqn:StrategicMDPbound}
$	\tilde O \(\sqrt{\ComplexityQ{} SAT}\)  = \tilde O \(\sqrt{\frac{1}{N} \times  N \times A \times T}\) = \tilde O\( \sqrt{AT}\).$
%\end{equation}
Notice that the lower order term should be added to the above expression and this is likely to be significant particularly for small $T$. However, we remark that the result above follows directly from Theorem \ref{thm:MainResultMainText} whose proof does not attempt to make the lower order term problem-dependent. Our result is substantially smaller than the typically reported bounds for this case, which are dominated by a $\tilde O(\sqrt{HSAT})$ term. 

There are two main factors that lead to the above simplification for this class of MDPs: $\vtrue{}$ is of order $1$ and not $N = H$ like in hard-to-learn MDPs which yields the lower bounds \cite{Dann15} and also the variance decreases as we let $N$ increase, each of which ``removes'' a factor of $\sqrt{N}$ from the known worst-case bound $\tilde O(\sqrt{HSAT} ) = \tilde O(\sqrt{N^2AT})$ after substituting $N=S=H$. 

% One might wonder if the above result is only obtainable if the rewards are deterministic. Fortunately we would obtain a significant improvement also in the stochastic reward variant presented in~\cite{OV17}, where \Alg{} would yield a regret that is dominated by $\tilde O(\sqrt{SAT})$, which is still $\sqrt{H}$ better than previously reported results.

To enable \Alg{}'s level of performance on this (and other) MDPs one needs to carefully control both the confidence intervals of the rewards and of the transition probabilities (as \Alg{} does), since Hoeffding-like concentration inequalities for the rewards alone would already induce a $\Theta (\sqrt{SAT}) = \Theta(\sqrt{NT})$ contribution to the regret expression.

%This example is significant because it suggests that even some previously considered very challenging MDPs may be amenable to carefully designed exploration. 
%variations of this MDP class are often given as an example where heuristic exploration schemes like dithering or Boltzmann may fail to learn in a meaningful time as they do not plan the exploration effort. As we see here, however, a carefully designed exploration agent not only overcomes that difficulty but in fact it attains much stronger guarantees than it was previously thought possible using a worst-case analysis. 
This result is intriguing because it suggests that truly pathological MDP classes (which induce $\Omega(\sqrt{HSAT})$ regret) are even more uncommon than previously thought. 

\subsection{\texorpdfstring{\Alg{}}{} on Tabular Contextual Bandits}
\label{sec:AppendixContextualBandits}
Contextual multi-armed bandits are a generalization of the multiarmed bandit problem, see for example \cite{BC12} for a survey. In their simplest possible formulation they entail a discrete set of contexts or states $\{s = 1,2,\dots \}$ and actions $\{a = 1,2,\dots \}$ and the expected reward $r(s,a)$ depends on both the state and action. After playing an action, the agent transitions to the next states according to some fixed distribution $\mu \in \R^{|\mathcal{S}|}$ over which the agent has no control.

In principle, such problem can be recast as an MDP in which the next state is independent of the prior state and action. Consider an $H$-horizon MDP which maps to a tabular contextual bandit problem: the transition probability is identical $p(s'|s,a) = \mu(s')$ for all states and actions, where $\mu$ is a fixed  distribution from which the next states are sampled. In such MDP Define the ``best'' and ``worst'' context at time $t$, respectively: $\overline s_t \stackrel{def}{=} \argmax V^*_t(s)$ and $\underline s_t \stackrel{def}{=} \argmin V^*_t(s)$ and recall that the transition dynamics $p(\cdot \mid s,a) = \mu$ depends nor on the action $a$ nor on the current state $s$. We have that: 
  \begin{equation}
  \label{main:VStarBackup}
    \begin{cases}
      V^*_t(\overline s_t) = \max _a \( r( \overline s_t,a) + \mu^\top V^*_{t+1} \) \\
      V^*_t(\underline s_t) = \max _a \( r( \underline s_t,a) + \mu^\top V^*_{t+1} \)
    \end{cases}
  \end{equation}
  which immediately yields a bound on the range of the value function of the successor states:
  \begin{align}
  \label{main:RngVStar}
\rng \vtrue{t+1} = V^*_t(\overline s_t) - V^*_t(\underline s_t)  = 
 \max_a r( \overline s_t,a) -  \max _a r( \underline s_t,a) \leq 1
  \end{align}
 
where the last inequality follows from the fact that rewards are bounded $r(\cdot,\cdot) \in [0,1]$. 
%Having an absolute bound on the range of $\vtrue{t+1}$ and immediately implies a bound on the range of $\Qrv_t(s,a)$ for any triplet $s,a,t$ by chaining equation \ref{main:DecouplingRpV} and \ref{main:rngUpperBound}:
%\begin{equation}
%	\Var \(\Qrv_t(s,a) \mid (s,a) \) \leq \frac{1}{2} \Big( \underbrace{\rng_{\vphantom{s^+ \in \mathcal{S}_{(s,a)}}} \( R(s,a)\)}_{\leq 1} + \underbrace{\rng_{s^+ \in \mathcal{S}_{(s,a)}} \( V^*_{t+1}(s^+) \)}_{\leq 1} \Big) \leq 1 
%\end{equation}
%from which
%\begin{equation}
%	\ComplexityQ{} = \frac{1}{T} \sumk \sumt{1} \E_{(s,a) \sim \piok} \(  \Var \( \Qrv_t(s,a) \mid (s,a)\) \) \leq 1
%\end{equation}
%follows.
%Plugging $\ComplexityQ{}$ into \ref{main:ProblemDependentRegretBound} we obtain:
Therefore $\Phi \leq 1$ and corollary \ref{cor:BoundedRange} yields a high probability regret upper bound of order
\begin{equation}
\tilde O \( \sqrt{SAT} + \LowerOrderTerm{} \)
\end{equation}
for \Alg{}. This means that \Alg{} automatically attains the lower bound in the dominant term for tabular contextual bandits \cite{BC12} if deployed in such setting. We did not try to improve the lower order terms for this specific setting, which may give an improved bound.

\section{Average Per-Episode Sample Complexity for the Setting of \cite{jiang2018open}}
\label{sec:COLTconjecture}
Jiang and Agarwal~\yrcite{jiang2018open} also ask about the dependence of 
a lower bound on the sample complexity on the planning horizon. While 
our work focuses on a regret analysis, and does not provide  
PAC sample complexity results, we can use our regret results to bound 
with high probability the number of episodes needed to ensure
that the \textit{average} per-episode regret is less than $\epsilon$
 To do so, we obtain the average per-episode loss of \Alg{} by dividing by the number of episodes $K$:
 \small
 \begin{align}
     \tilde O\( \left(\frac{\sqrt{SA}}{\sqrt{K}} + \frac{\LowerOrderTerm)}{K} \) \).
 \end{align}
 \normalsize
 From here we can seek for the smallest $K$ such that the average error is smaller than $\epsilon$, obtaining:
 \small
 \begin{equation}
 \label{main:sparsesamplecomplexity}
      \tilde O \( \left(\frac{SA}{\epsilon^2} + \frac{\LowerOrderTerm}{\epsilon}\) \right)
 \end{equation}
 \normalsize
 episodes before the \emph{average per-episode} error is smaller than $\epsilon$. 
 For small $\epsilon << SA$, the first term dominates, which is again independent 
 of a polynomial dependence on $H$.  Of course, in order to formally obtain a PAC result (a worst case upper bound on the number of $\epsilon$-suboptimal episodes) the algorithm would need to be modified to be PAC. In particular, the exploration bonus for a given $(s,a)$ pair should be designed so it does not increase with time $T$ if $(s,a)$ is not visited. In practice this means replacing the $\log(T)$ factor of the exploration bonuses with something like $\log(n)$ where $n$ is the visit count to a specific state-action pair and adjusting the numerical constant to make sure the exploration bonuses / confidence intervals are still valid with high probability. Please see \cite{Dann17} for a detailed explanation of how to proceed with the algorithm design and analysis in this case\footnote{Indeed, the algorithm described in \cite{Dann17} is structurally similar to ours}.

\section{Appendix Overview and Proof Preview}
\label{sec:AppendixOverview}
We start by giving an overview of the proof that leads to the main result. This setting is more general than the one presented in the main text. In particular we 1) define a class of concentration inequalities for the transition dynamics 2) show that \Alg{} achieves strong problem dependent regret bounds with any concentration inequality satisfying these assumptions.
In particular, the main result presented in the main text follows as a corollary of the potentially more general analysis presented. We now give a preview of the proof of the main result, assuming rewards are known, though we later relax this assumption. We start by recalling \Alg{} with yet-to-specify confidence intervals on the transition dynamics.

\subsection{Algorithm}
The algorithm is presented in figure \ref{alg:Algorithm}.
\begin{algorithm*}[!htb]
   \caption{\Alg{} for Stationary Episodic MDPs}
   \label{alg:Algorithm}
\begin{algorithmic}[1]
   \STATE \textbf{Input}: confidence interval $\br{\cdot}{\cdot}$ $\phi(\cdot,\cdot)$ with failure probability $\delta$, constants $\Bp,\Bv$.
   \STATE $n(s,a) = r_{sum}(s,a) = p_{sum}(s',s,a) = 0, \;\; \forall s,a\in \mathcal S \times \mathcal A; \quad \overline V_{H+1}(s) = 0, \;\; \forall s \in \mathcal S$
   \FOR{$k=1,2,\dots$}
      \FOR{$t=H,H-1,\dots,1$}
     	 \FOR{$s \in \mathcal S$}
	 \FOR{$a \in \mathcal A$}
	 \STATE $\hat p = \frac{p_{sum}(\cdot,s,a)}{n(s,a)}$
   \STATE $\bbonus = \phi(\hat p(s,a),\overline V_{t+1}) + \frac{1}{\sqrt{n(s,a)}} \( \frac{4J + \BpPlus}{\sqrt{n_k(s,a)}} + \Bv \| \overline V_{t+1} - \underline V_{t+1} \|_{2,\hat p} \)$ \\
   \STATE $Q(a) = \min\{ H-t, \rhat{s}{\piok(s,t)}+\br{s}{a} + \hat p^\top \overline V_{t+1} + \bbonus \}$
  \ENDFOR
   \STATE $\piok(s,t) = \argmax_a Q(a)$
   \STATE $\overline V_t(s) = Q(\piok(s,t))$ 
   \STATE $\bbonus = \phi(\hat p(s,\piok(s,t)),\underline V_{t+1}) + \frac{1}{\sqrt{n(s,\piok(s,t))}} \( \frac{4J + \BpPlus}{\sqrt{n_k(s,\piok(s,t))}} + \Bv \| \overline V_{t+1} - \underline V_{t+1} \|_{2,\hat p} \)$ \\
   \STATE $\underline V_t(s) = \max\{0, \rhat{s}{\piok(s,t)} - \br{s}{\piok(s,t)} + \hat p^\top \underline V_{t+1} - \bbonus\}$
   \ENDFOR
   \ENDFOR
   \STATE
   $s_1 \sim p_0$
   \FOR{$t=1,\dots H$}
   \STATE $a_t = \piok(s_t,t);\quad r_t \sim p_R(s_t,a_t);\quad s_{t+1}\sim p_P(s_t,a_t)$
   \STATE $n(s_t,a_t)++;\quad p_{sum}(s_{t+1},s_t,a_t)++$
   \ENDFOR
   \ENDFOR
\end{algorithmic}
\end{algorithm*}

\subsection{Optimism}
The goal of this section is to show that \Alg{} guarantees optimism with high probability. As is well known, optimism is the ``driver'' of exploration as it allows to overestimate the regret by a concentration term with high probability.

Let's consider the planning process at the beginning of episode $k$. This is detailed in lines $4$ to $16$ of algorithm \ref{main:AlgorithmMainText}. In order to guarantee finding a pointwise optimistic value function $\vo{t} \geq \vtrue{t}$ a bonus is added to account for ``bad luck'' in the system dynamics experienced up to episode $k$. If the agent knew the value of the confidence interval for the system dynamics $\phi(\potruea{s}{t},\vtrue{t+1})$ then optimism could be inductively guaranteed (i.e., assuming that, by induction, $\vo{t+1} \geq \vtrue{t+1}$ holds pointwise) if said confidence interval holds:
\begin{align}
 \vo{t} & = \rotruea{t}{s} + \pohata{s}{t}^{\top}\vo{t+1} + \phi(\potruea{s}{t},\vtrue{t+1}) \\
& \geq \rotruea{t}{s} + \pohata{s}{t}^{\top}\vtrue{t+1} + \phi(\potruea{s}{t},\vtrue{t+1}) \geq \rotruea{t}{s} + \potruea{s}{t}^{\top}\vtrue{t+1} \geq \vtrue{t+1}
\end{align}
If the above conclusion is true for every action  then it is true for the maximizer as well.
Unfortunately the agent knows nor the real transition dynamics nor the optimal value function to evaluate $\phi$. Instead, it only has access to the estimated transition dynamics $\pohata{s}{t}$ and to an overestimate of the value function $\vo{t+1}$. Unfortunately the confidence interval $\phi$ evaluated with such quantities $\phi(\pohata{s}{t},\vo{t+1})$ is not guaranteed to overestimate $\phi(\potruea{s}{t},\vtrue{t+1})$ and optimism may not be guaranteed. To remedy this the agent can try to estimate the difference 
\begin{align}
\label{main:DeltaPhiMainText}
& | \phi(\pohata{s}{t},\vo{t+1}) - \phi(\potruea{s}{t},\vtrue{t+1})|
\end{align}
and add a correction term to account for that difference. A similar problem is faced in \cite{Azar17} where the authors  propose an optimistic bonus which guarantees optimism when using the empirical Bernstein Inequality. By distinction, our way of constructing the bonus works with \emph{any} concentration inequality satisfying assumption \ref{ass:ConfidenceIntervals} and \ref{ass:FiniteTimeBonusBound}, as described in the Appendix Section~\ref{sec:assumption}.
Precisely, optimism is dealt with in Appendix section \ref{app:Optimism}; in lemma \ref{lem:DeltaPhi} we show how to bound equation \ref{main:DeltaPhiMainText} obtaining the result below:
\begin{align}
\label{main:DeltaPhi}
& | \phi(\pohata{s}{t},V) - \phi(\potruea{s}{t},\vtrue{t+1})|  \leq \frac{ \Bv \|V - \vtrue{t+1} \|_{2,\hat p}}{\sqrt{\nsaa{s}{t}{k}}} + \frac{ \BpPlus + 4J}{\nsaa{s}{t}{k}}.
\end{align}
This is essentially a consequence of the definition of admissible bonus, i.e., Definition  \ref{def:AdmissiblePhi} (Appendix Section~\ref{sec:assumption}).
Unfortunately the upper bound in equation \ref{main:DeltaPhi} depends on $\vtrue{t+1}$ which is not known, so the problem is still unsolved. However, as we show in lemma \ref{lem:OptimismOverestimate} in the appendix it suffices to (pointwise) overestimate $\vo{t+1}-\vtrue{t+1}$. To this aim, the algorithm maintains an underestimate of $\vtrue{t+1}$ which we call $\vp{t+1}$. Equipped with this underestimate, we define the \emph{Exploration Bonus} in definition \ref{def:TransitionBonus} (Appendix Section \ref{sec:transbonus}), which we report below:
\begin{equation}
\label{main:Bonus}
\bbonus(\pohata{s}{t},\vo{t+1},\vp{t+1}) \stackrel{def}{=} \phi(\pohata{s}{t},\vo{t+1}) + \ExtraBonus.
\end{equation}
Importantly, Equation \ref{main:Bonus} only uses quantities that are known to the agent: the functional form of $\phi(\cdot,\cdot)$, the maximum likelihood estimate $\pohata{s}{t}$, the overestimate and underestimate, \vo{t+1} and \vp{t+1} respectively, of the optimal value function at the next timestep, the visit count $\nsaa{s}{t}{k}$ and the constants $J,\Bp,\Bv$. Notice that the norm $\| \cdot \|_{2,\hat p}$ is computed using $\pohata{s}{t}$ which is known to the agent as opposed to $\potruea{s}{t}$. If $\vo{t+1}$ and $\vp{t+1}$ bracket $\vtrue{t+1}$ then we have that the bonus of equation  \ref{main:Bonus} overestimates the admissible confidence interval $\phi(\potruea{s}{t},\vtrue{t+1})$ that we could construct if we knew $\potruea{s}{t}$ and $\vtrue{t+1}$, that is:
\begin{equation}
	\label{main:BonusIsOptimistic}
	\bbonus(\pohata{s}{t},\vo{t+1},\vp{t+1}) \geq \phi(\potruea{s}{t},\vtrue{t+1})
\end{equation}
This is proved in Proposition \ref{prop:TransitionBonusIsOptimistic} in the appendix. At this point we have all the elements to show optimism. In fact, we need a little more effort than simply optimism because the algorithm has to maintain a valid bracket of the optimal value function:
\begin{equation}
\label{main:AlgorithmBracketsVstar}
\vp{t} \leq \vtrue{t}  \leq \vo{t}  \quad \text{(pointwise)}
\end{equation} 
This is done in Proposition \ref{prop:AlgorithmBracketsVstar} in the appendix and it simply relies on an induction argument.

At this point we have guaranteed optimism but we relied on the construction of confidence intervals for the value function, to which we turn our attention next.

\subsection{Confidence Interval for the Value Function}
During its execution, \Alg{} implicitly construct confidence interval for the value function with the property defined by equation \ref{main:AlgorithmBracketsVstar}.
Precisely in Proposition \ref{prop:DeltaOptimism} we relate the distance $\vo{t}(s) - \vp{t}(s)$ to the number of visits to the $(s,a)$ pairs in the trajectories originated upon following policy \piok{} \emph{on the true MDP} with high probability. In other words, assuming  that confidence intervals hold we provide a way to relate the accuracy of the agent's estimate of the value function to a concentration term that depends on the number of visits to the $(s,a)$ pairs that the agent is expected to encounter by following that policy, obtaining up to a constant:
\begin{align}
\label{main:DeltaOptimism}
\vo{t}(s) - \vp{t}(s) \lesssim & \sum_{\tau = t}^H \E \( \min \Big\{ \frac{F+D}{\sqrt{\nsaa{s_\tau}{\tau}{k}}},H \Big\} \Bigm| s,\piok{}{} \)
\end{align}
for some $F,D$ defined in Proposition \ref{prop:DeltaOptimism}.

This serves as an estimate of the confidence interval for the optimal value function itself. The importance of the lemma lies in connecting a property of the algorithm (the difference between the ``optimistic'' and the ``pessimistic'' value function) to the uncertainty in the various states encountered in the MDP (upon following \piok{}) weighted by the \emph{true} visitation probability.

\subsection{Regret Bound}
We are finally ready to discuss the regret bounds that leads to the main result of Theorem \ref{thm:MainResult} which is proved in Appendix Section \ref{app:RegretAnalysis} along with the related lemmata. We recall the following regret decomposition which is standard in recent analysis \cite{Dann17,OV16}:
\begin{align}
& \text{Regret}(K) \stackrel{def}{=} \sum_k^K V_1^*(s_{1k}) -  V_1^{\piok}(s_{1k}) \\
&  \leq \sum_{k=1}^K \sum_{t \in [H]} \sum_{(s,a)} w_{tk}(s,a) \Biggm( \underbrace{\(\tilde r_k(s,a) - r(s,a) \)}_{\text{Reward Estimation and Optimism}} + \underbrace{\(\tilde p_k(s,a) - \hat p(s,a) \)^\top \overline V_{t+1}^{\piok}}_{\text{Transition Dynamics Optimism}} \\ +  & \underbrace{\(\hat p_k(s,a) - p(s,a) \)^\top V_{t+1}^{*}}_{\text{Transition Dynamics Estimation}} + \underbrace{\(\hat p_k(s,a) - p(s,a) \)^\top \( \vo{t+1} - \vtrue{t+1} \)}_{\text{Lower Order Term}}  \Biggm) \\
\numberthis{}
\end{align}	
In later sections we bound each term individually; here we just touch on the order of magnitude of the leading order term which is the ``Transition Dynamics Optimism." We begin by using the bonus added during the planning step (Definition \ref{def:TransitionBonus}):	
\begin{align}
& \sumk \sumt{1}\sumsLk  \visita{s}{t}{k} \bonus{s} \stackrel{def}{\leq} \sumk \sumt{1}\sumsLk  \visita{s}{t}{k} \bbonus(\pohata{s}{t},\vo{t+1},\vp{t+1})
\end{align}
where the inequality follows from the fact that the we ''cap'' the backup term $\poa{s}{t}\vo{t+1} \leq H$ (see the $\min$ in the main algorithm). By definition of the bonus (definition \ref{def:TransitionBonus}) we get:
\begin{align} 
& c_{200} \sumk \sumt{1}\sumsLk  \visita{s}{t}{k} \( \phi(\pohata{s}{t},\vo{t+1}) + \frac{\Bv \devihat{t+1}}{\sqrt{\nsaa{s}{t}{k}}} + \frac{\BpPlus+J}{\nsaa{s}{t}{k}}\),
\end{align}
for some constant $c_{200}$. Equation \ref{main:DeltaPhi} ensures that $\phi(\pohata{s}{t},\vo{t+1})$ and $\phi(\potruea{s}{t},\vtrue{t+1})$ are close, leading to essentially the same upper bound up to a constant:
\begin{align}
& \stackrel{}{\lesssim} \sumk \sumt{1}\sumsLk  \visita{s}{t}{k} \( \phi(\potruea{s}{t},\vtrue{t+1}) + \frac{\Bv \devihat{t+1}}{\sqrt{\nsaa{s}{t}{k}}} + \frac{\BpPlus+J}{\nsaa{s}{t}{k}}\)
\end{align}

Using the functional form for $\phi$ we obtain the upper bound below $(c)$:
\begin{align}
& \stackrel{c}{\lesssim} \sumk \sumt{1}\sumsLk  \visita{s}{t}{k} \( \underbrace{\frac{\dg{s}{t+1}}{\sqrt{\nsaa{s}{t}{k}}}}_{\text{Leading Order Term}} + \underbrace{\frac{J+\BpPlus}{\nsaa{s}{t}{k}} + \frac{ \Bv \devihat{t+1} }{\sqrt{\nsaa{s}{t}{k}}}}_{\text{Lower Order Term}} \)
\end{align}
An induction argument coupled with equation \ref{main:DeltaOptimism}  shows that $\|\vo{t+1} - \vp{t+1} \|_{2,\hat p }$ shrinks at a rate $\frac{1}{\sqrt{\nsaa{s}{t}{k}}}$. Thus, the ``Lower Order Term'' shrinks at a rate $\frac{1}{n}$, and ultimately gives a regret contribution independent on $T$ except for a log factor. The leading order term shrinks at a rate $\frac{1}{\sqrt{n}}$, giving the $\tilde O\( \sqrt{\ComplexityQ{} SAT} \)$ contribution which is the leading order term. Further, since $g(\cdot,\cdot)$ here depends on $\vtrue{t+1}$, for $t \in [H]$, this is a problem-dependent  (and concentration-inequality-dependent) bound, as we wanted.

In the full proof as follows, we also include uncertainty over the reward function.
 
\section{Failure Events and their Probabilities}
We now discuss the failure events and the assumption for the concentration inequalities that lead to the definition of \Alg{}. We then verify that Bernstein Inequality satisfies these assumptions, leading to a practical algorithm.

\subsection{Empirical Bernstein Inequality for the Rewards}
We recall the Empirical Bernstein Inequality\footnote{Note the change of $n_k(s,a)-1$ to $n_k(s,a)$ compared to \cite{MP09} in the lower order term, and the doubling of the constant for that term since $\frac{1}{n_k(s,a)-1} \leq \frac{2}{n_k(s,a)}$ for $n_k(s,a) \geq 2$.} \cite{MP09} for estimating the rewards:
\begin{definition}[Reward Empirical Bernstein]
\label{lem:RewardEmpiricalBernstein}
Let $R(s,a) \in [0,1]$ be the reward random variable in state $s$ upon taking action $a$ and let $\widehat \Var R(s,a)$ be its sample variance. The following holds true with probability at least $1-\delta'$:
\begin{equation}
	\label{eqn:EmpiricalBernsteinRewards}
\Big| \rohata{s}{t} - \rotruea{s}{s} \Big| \leq \EmpiricalBernsteinRewards{}.
\end{equation}
\end{definition}
This concentrates fast to the actual reward variance \info{Check Constants}:
\begin{lemma}[Delta $\phi_r$]
\label{lem:DeltaPhiReward}
With probability at least $1-\delta'$ it holds that: 
\begin{align}
\label{eqn:DeltaPhiReward}
|\sqrt{\widehat{\Var} R(s,a)} - \sqrt{\Var R(s,a)}| \leq \sqrt{\frac{4 \ln (2SAT/\delta')}{\nsaa{s}{t}{k}}}.
\end{align}
jointly for all states, actions, and timesteps.
\end{lemma}
\begin{proof}
Analogous to Theorem 10 in \cite{MP09} with a union bound argument over the states, the actions and the the timesteps. Note the change of $n_k(s,a)-1$ to $n_k(s,a)$ compared to \cite{MP09} and the doubling of the constant since $\frac{1}{n_k(s,a)-1} \leq \frac{2}{n_k(s,a)}$ for $n_k(s,a) \geq 2$.	
\end{proof}

\subsection{Admissible Confidence Intervals on the Transition Dynamics}
\label{sec:assumption}
In this section we define a class of confidence intervals that are admissible for \Alg{} for which our analysis holds.
The aim is to ensure that $| \(\hat p_k(s,a) - p(s,a)\)^\top\vtrue{t+1} |$ is bounded with high probability throughout the execution of the algorithm. Said concentration inequality should be tight so that successor states with low visitation probability have low impact. The former requirement is formalized in equation \ref{eqn:AdmissiblePhiHolds} and the latter in equation \ref{eqn:gVbound}, which we report below.

\begin{assumption}[Confidence Intervals]
\label{ass:ConfidenceIntervals}
With probability at least $1-\delta'$ it holds that:
\begin{equation}
| \(\hat p_k(s,a) - p(s,a)\)^\top\vtrue{t+1} | \leq \phi(p(s,a),\vtrue{t+1})
\label{eqn:AdmissiblePhiHolds}
\end{equation}
jointly for all timesteps $t$, episodes $k$, states $s$ and actions $a$. 
We assume that $\phi(p,V)$ takes the following functional form:
\begin{align}
\label{eqn:AdmissiblePhiFunctionalForm}
& \phi(p,V) = \frac{g(p,V)}{\sqrt{n_k(s,a)}} + \frac{j(p,V)}{n_k(s,a)} 
\end{align}
where $j(p,v) \leq J \in \R$. In particular we assume the following constraint on the functional form of $g(\cdot,\cdot)$:
\begin{align}
\label{eqn:gVbound}
| g(p,V_1) - g(p,  V_2) | & \leq \Bv \| V_1 - V_2 \|_{2,p}
\end{align}
and if the value function is uniform then:
\begin{equation}
\label{ass:gzero}
g(p,\alpha \1) = 0, \quad \forall \alpha \in \R.
\end{equation}
\end{assumption}
Equation \ref{eqn:AdmissiblePhiFunctionalForm} refers to the functional form of $\phi(\cdot,\cdot)$ which is the concentration inequality on the transition dynamics. Equation \ref{eqn:AdmissiblePhiFunctionalForm}  identifies two contributions: a leading order term which scales with $1/\sqrt{n}$ and a lower order term that scales with $1/n$.
Equation \ref{eqn:gVbound} plays a crucial role. It posits a requirement on the functional form of the leading order term of the concentration inequality when the coefficients $\vtrue{t+1}(s')$ are changed. Precisely, it quantifies how the concentration inequality changes if we change $\vtrue{t+1}$, formalizing the intuition that if $p(s'\mid s,a)$ is small then changing $\vtrue{t+1}(s')$ should have little impact on the bound given by the concentration inequality. It also implies that $g(p,V)$ depends on $V$ only through the entries that correspond to the support of $p$, that is, it depends on $V(s)$ if $p(s) \not = 0$. Practically speaking, if a successor $s'$ cannot be visited from $s$ then the value function at $s'$ does not directly impact $s$, as one would hope.

The next assumption deals with the rate of convergence of the leading order term seen as a function of $\hat p$. Under mild assumptions, as $\hat p$ converges to $p$, a function of $\hat p$  converges as well. The assumption below is the corresponding non-asymptotic requirement:
 
\begin{assumption}[Finite Time Bonus Bound]
\label{ass:FiniteTimeBonusBound}
With probability at least $1-\delta'$ it holds that:
\begin{align}
\label{eqn:gPbound}
| g(\pohata{s}{t}, \vtrue{t+1}) - g(\potruea{s}{t}, \vtrue{t+1})  | & \leq \frac{\BpPlus}{\sqrt{\nsaa{s}{t}{k}}}
\end{align}
jointly for all episodes $k$, timesteps $t$, states $s$, actions $a$ and some constant \Bp{} that does not depend on $\nsaa{s}{t}{k}$.
\end{assumption}

In both assumption \ref{ass:ConfidenceIntervals} and \ref{ass:FiniteTimeBonusBound} the constants $\Bv$ and $\BpPlus$ can depend on the input parameters (e.g., $S,A,H,T,\frac{1}{\delta}$ etc...).

We pose the following definition:
\begin{definition}[Admissible $\phi$]
\label{def:AdmissiblePhi}
	If $\phi$ satisfies assumption \ref{ass:ConfidenceIntervals} and \ref{ass:FiniteTimeBonusBound} then we say that $\phi$ is admissible for \Alg{}.
\end{definition}

\begin{corollary}[Max $\phi$]
\label{cor:MaxPhi}
Let the function $g(\cdot,\cdot)$ be defined as in equation \ref{eqn:AdmissiblePhiFunctionalForm}. Combining equations \ref{eqn:gVbound} and \ref{ass:gzero} (where $V_2 = \vec 0$) and recalling monotonicity of norms of random variables one immediately obtains:
\begin{equation}
\label{eqn:MaxPhi}
|g(p,V)| \leq \Bv \| V_1 \|_{2,p} \leq \Bv \| V_1 \|_\infty \leq \Bv H.
\end{equation}
\end{corollary}

\subsection{Bernstein's Inequality}
We now show that Bernstein Inequality is admissible for \Alg{}.
\BernsteinIsAdmissible{\label{prop:BernsteinIsAdmissible}}
\begin{proof}
Bernstein's inequality guarantees that with probability at least $1-{\delta'}$ we have that:
\begin{equation}
\label{eqn:BernsteinInequality}
| \(\hat p_k(s,a) - p(s,a)\)^\top\vtrue{t+1} | \leq \Bernstein{} \stackrel{def}{=}  \phi(p(s,a),\vtrue{t+1})  
\end{equation}
jointly for all timesteps, states $s$ and actions $a$ after a union bound argument over the states $s$, actions $a$ and timesteps. Thus, equation \ref{eqn:AdmissiblePhiHolds} holds. Here $J = \frac{H\ln\frac{2SAT}{\delta'}}{3}$ and $\sqrt{2\ln\frac{2SAT}{\delta'}}\stackrel{def}{\leq} L$ so that 
$$g(p(s,a),\vtrue{t+1}) \stackrel{def}{=} \sqrt{\Var_p \vtrue{t+1}} \times \sqrt{2 \ln \frac{2SAT}{\delta'}} \leq L \sqrt{\Var_p \vtrue{t+1}}.$$ Thus equation \ref{eqn:AdmissiblePhiFunctionalForm} holds. 
Consider the mean-centered random variables $\overline V_1 = V_1 - \E V_1$ and $\overline V_2 = V_2 - \E V_2$. Then:
\begin{align}
\sqrt{\Var(V_1)} & = \sqrt{\Var(\overline V_1)} =  \sqrt{\E(\overline V_1)^2}  =  \| \overline V_1 \|_{2,p} = \| \overline V_2 + \overline V_1 - \overline V_2 \|_{2,p} \\
& \leq \| \overline V_2 \|_{2,p} + \| \overline V_1 - \overline V_2 \|_{2,p} =\sqrt{\E(\overline V_2)} + \sqrt{\E(\overline V_1-\overline V_2)^2} \\
& = \sqrt{\Var(\overline V_2)} +  \sqrt{\E(V_1-V_2)^2 - \( \E(V_1-V_2)\)^2} \\
& = \sqrt{\Var(V_2)} + \sqrt{\Var(V_1-V_2)}.
\end{align}
where the inequality is Minkowski's inequality (i.e., the triangle inequality for norm of random variables).
Rearranging we get:
\begin{align}
| g(p,V_1) - g(p,V_2) | \leq L| \sqrt{\Var(V_1)} - \sqrt{\Var(V_2)} | \leq L\sqrt{\Var(V_2-V_1)} \leq L\|V_2-V_1 \|_{2,p}
\end{align}
and so $\Bv = L$ in equation \ref{eqn:gVbound}.

Finally, a variation\footnote{Note the change of $n_k(s,a)-1$ to $n_k(s,a)$ compared to \cite{MP09} and the doubling of the constant for that term since $\frac{1}{n_k(s,a)-1} \leq \frac{2}{n_k(s,a)}$ for $n_k(s,a) \geq 2$.} of theorem 10 from \cite{MP09} ensures that:
\begin{equation}
	\label{eqn:DeltaV}
	\Bigm| \| \vtrue{t} \|_{2,\hat p} - \| \vtrue{t} \|_{2,p} \Big| \leq H\sqrt{\frac{4\ln{\frac{2SAT}{\delta'}}}{
	\nsaa{s}{t}{k}}} \stackrel{}{=} \frac{\BpPlus}{\sqrt{\nsaa{s}{t}{k}}}
\end{equation}
with probability at least $1-\delta'$ jointly for all states $s$, actions $a$ and possible values for $n$ after a union bound on these quantities. Hence $\Bp = \sqrt{2}H L$ and assumption \ref{ass:FiniteTimeBonusBound} is satisfied as well. 
This concludes the verification that Bernstein's inequality satisfies both \ref{ass:ConfidenceIntervals} and \ref{ass:FiniteTimeBonusBound} and is thus admissible for the algorithm.

%Symmetry concludes the proof that for Bennett's inequality equation \ref{eqn:gVbound} holds with $\Bv  \leq \sqrt{2}L$. Finally we need to verify \ref{eqn:gPbound}. If $\| V \|_{2,p} = 0$ then the bound holds. Assume $\| V \|_{2,p} \not = 0$. Then consider:
%\begin{align}
% \( \| V \|_{2,p} - \| V \|_{2,\hat p} \)\( \| V \|_{2,p} + \| V \|_{2,\hat p} \) & = \| V \|_{2,p}^2 - \| V \|_{2,\hat p}^2 \\
% & =  p^\top(V - \E V \1)^2 - \hat p^\top(V - \E V \1)^2 \\
% & =  \( p- \hat p \) ^\top(V - \E V \1)^2.
% \end{align}
%Division by $\( \| V \|_{2,p} + \| V \|_{2,\hat p} \)$ yields:
%\begin{align}
%\frac{\( p- \hat p \) ^\top(V - \E V \1)^2}{\( \| V \|_{2,p} + \| V \|_{2,\hat p} \)} \leq \frac{\( p- \hat p \) ^\top(V - \E V \1)^2}{\| V \|_{2,p}}
%\end{align}
%\warn{X}
\end{proof}
\subsection{Other Failure Events and Their Probabilities}
\label{subsec:FailureEventAndTheirProbabilities}
An independent use of Bernstein inequality also gives with probability at least $1-\delta'$ jointly for all states $s$, successors $s'$, actions $s$ and values for $\nsaa{s}{t}{k}$ the following component-wise bound on the failure event (see \cite{Azar17} for a derivation \info{a $\delta'$? can be saved if Bernstein is use throughoutt}):
\begin{align}
\label{eqn:cip}
| \pohatasp{s}{t}{s'} - \potruespa{s}{t}{s'} | \leq \sqrt{\frac{\potruespa{s}{t}{s'}(1-\potruespa{s}{t}{s'})\ln\frac{2TS^2A}{\delta'}}{\nsaa{s}{t}{k}}} + \frac{\ln\frac{2TS^2A}{\delta'}}{3\nsaa{s}{t}{k}}.
\end{align}
Moreover, \cite{Weissman03} gives the following high probability bound on the one norm of the Maximum Likelihood Estimate \info{Can this be derived from Bernstein to save a $\delta'$?}; in particular, with probability at least $1-\delta'$ it holds that:
\begin{align}
\label{eqn:Weissman}
\|\pohata{s}{t} -\potruea{s}{t} \|_1 \leq \sqrt{\frac{2S\ln\frac{2SAT}{\delta'}}{\nsaa{s}{t}{k}}}
\end{align}
jointly for all states $s$, actions $a$ and possible values for $\nsaa{s}{t}{k}$ after a union bound argument on these quantities.
Finally, with probability at least $1-\delta'$ the following holds for every state-action pair, timestep and episode (see for example \cite{dann2019policy}, failure event $F^N$ in section B.1, for the proof):
\begin{equation}
\label{eqn:wnevent}
n_k(s,a) \ge  \frac{1}{2}\sum_{j< k}w_{j}(s,a) - \lnnsa
\end{equation}
where $w_{tj}(s,a)$ is the probability of visiting the $(s,a)$ pair in timestep $t$ of episode $j$ under the chosen policy and $\sum_{\tau \in H}w_{\tau j}(s,a)$ is the sum of the probabilities of visiting the $(s,a)$ pair in episode $j$.

\begin{lemma}[Failure Probability]
\label{lem:FailureProbability}
If $\delta' = \frac{1}{7}\delta$ then equation \ref{eqn:EmpiricalBernsteinRewards}, \ref{eqn:DeltaPhiReward},\ref{eqn:BernsteinInequality},\ref{eqn:DeltaV}, \ref{eqn:cip},\ref{eqn:Weissman}, \ref{eqn:wnevent} hold jointly with probability at least $1-\delta$. When this happens we say that we \Alg{} is \emph{outside of the failure event}.
\end{lemma}
\begin{proof}
By union bound.	
\end{proof}

\section{Optimism}
\label{app:Optimism}
In this section we show that \Alg{} computes optimistic bounds on $Q$.
\subsection{Rewards}
In view of the empirical Bernstein Inequality for the rewards in lemma \ref{lem:RewardEmpiricalBernstein} we define as reward bonus:
\begin{definition}[Reward Bonus]
\label{def:RewardBonus}
\begin{equation}
 \br{s}{a} \stackrel{def}{=} \EmpiricalBernsteinRewards{}.
\end{equation}
\end{definition}

\begin{lemma}[Reward Bonus is Optimistic]
\label{lem:RewardBonusIsOptimistic}
Outside of the failure event it holds that:
\begin{align}
	\rohata{s}{t} + \br{s}{a} & \geq r(s,a) \\
	\rohata{s}{t} - \br{s}{a} & \leq r(s,a) 
\end{align}
\end{lemma}
\begin{proof}
	By Definition \ref{def:RewardBonus} and lemma \ref{lem:RewardEmpiricalBernstein}.
\end{proof}

\subsection{Transition Dynamics}
\label{sec:transbonus}
\begin{lemma}[Delta $\phi$]
\label{lem:DeltaPhi}
If $\phi$ is admissible for \Alg{} then for all $V \in \R^{S}$: 
\begin{align}
\label{eqn:DeltaPhi}
& | \phi(\pohata{s}{t},V) - \phi(\potruea{s}{t},\vtrue{t+1})|  \leq \frac{ \Bv \|V - \vtrue{t+1} \|_{2,\hat p}}{\sqrt{\nsaa{s}{t}{k}}} + \frac{ \BpPlus + 4J}{\nsaa{s}{t}{k}}.
\end{align}
\end{lemma}
Outside of the failure event the above lemma deals with the functional form of $\phi$; there are no ``failure events'' or probabilities to be considered here.
\begin{proof}
From the LHS of Equation \ref{eqn:DeltaPhi} by adding and subtracting $\phi(\pohata{s}{t},\vtrue{t+1})$ we get to an expression equivalent to the LHS of Equation \ref{eqn:DeltaPhi}:
\begin{equation}
(\ref{eqn:DeltaPhi}) =  | \phi(\pohata{s}{t},V) - \phi(\pohata{s}{t},\vtrue{t+1}) + \phi(\pohata{s}{t},\vtrue{t+1}) - \phi(\potruea{s}{t},\vtrue{t+1})|
\end{equation}
The triangle inequality allows to split the above equation into the upper bound below:
\begin{align}
\leq | \phi(\pohata{s}{t},V) - \phi(\pohata{s}{t},\vtrue{t+1}) | + |\phi(\pohata{s}{t},\vtrue{t+1}) - \phi(\potruea{s}{t},\vtrue{t+1})|
\end{align}
Next we can use the constraint on $\phi$. In particular, condition \ref{eqn:AdmissiblePhiFunctionalForm} implies that the above equation can be upper bounded as below:
\begin{align}
& \leq \Big| \frac{g(\pohata{s}{t},V) - g(\pohata{s}{t},\vtrue{t+1})}{\sqrt{\nsaa{s}{t}{k}}} \Big| \\
& +  \Big| \frac{g(\pohata{s}{t},\vtrue{t+1}) - g(\potruea{s}{t},\vtrue{t+1})}{\sqrt{\nsaa{s}{t}{k}}} \Big| + \frac{4J}{\nsaa{s}{t}{k}}
\end{align} 

Finally, the functional constraints on $g$ of Equation \ref{eqn:gVbound} together with Assumption \ref{ass:FiniteTimeBonusBound}, respectively, bound each of the above terms:
\begin{align}
\Big| \frac{g(\pohata{s}{t},V) - g(\pohata{s}{t},\vtrue{t+1})}{\sqrt{\nsaa{s}{t}{k}}} \Big| & \leq \frac{\Bv \|V - \vtrue{t+1} \|_{2,\hat p}}{\sqrt{\nsaa{s}{t}{k}}} \\
\Big| \frac{g(\pohata{s}{t},\vtrue{t+1}) - g(\potruea{s}{t},\vtrue{t+1})}{\sqrt{\nsaa{s}{t}{k}}} \Big| & \leq \frac{\BpPlus}{\nsaa{s}{t}{k}}
\end{align} 
completing the proof.
\end{proof}
Lemma \ref{lem:DeltaPhi} is crucial in that it allows to relate how far off is the concentration inequality $\phi$ (computed using the empirical estimates for $p$ and $V$) from the one computed using the ``real'' values, which would guarantee optimism. In other words, if one can compute $\|\vo{t} - \vtrue{t} \|_{2,\hat p}$ then this estimate can be used with lemma \ref{lem:DeltaPhi} to derive a bonus $\bbonus$, function of the empirical quantities $\pohata{s}{t}$ and $\vo{t+1}$, which is guaranteed to overestimate $\phi(\potruea{s}{t},\vtrue{t+1})$. Ultimately, the purpose is to construct a ``bonus'' that overestimates $\phi$ without being ``much larger'' than $\phi$. This is the motivation behind the following definition, which will eventually lead to optimism of \Alg{} while ensuring a regret that is problem-dependent.

\begin{definition}[Transition Bonus]
\label{def:TransitionBonus}
Define the bonus $\bbonus(\cdot,\cdot,\cdot)$:
\begin{equation}
\label{eqn:Bonus}
\bbonus(\pohata{s}{t},\vo{t+1},\vp{t+1}) \stackrel{def}{=} \phi(\pohata{s}{t},\vo{t+1}) + \ExtraBonus.
\end{equation}
\end{definition}
En-route to showing optimism we first show the value of having an overestimate \emph{and} and underestimate of \vtrue{t+1} in lemma \ref{lem:OptimismOverestimate}. This allows to relate the bonus of definition \ref{def:TransitionBonus} to the concentration inequality identified by $\phi(\potruea{s}{t},\vtrue{t+1})$. The result of the first lemma is summarized below:

\begin{lemma}[Optimism Overestimate]
\label{lem:OptimismOverestimate}
For any transition probability vector $p$, (i.e., such that $\| p \|_1 = 1$) and any $V \in \R^S$ if:
\begin{equation}
\vp{t+1} \leq V \leq \vo{t+1}
\end{equation}
holds pointwise then
\begin{align}
	\label{eqn:OptimismOverestimate}
\devistarv{t+1} & \leq \devi{t+1} \\
\label{eqn:PessimismOverestimate}
\devistarpessimisticv{t+1} & \leq \devi{t+1}
\end{align}
holds.
\end{lemma}
\begin{proof}
The hypothesis ensures:
\begin{align}
0 \leq \vo{t+1}(s') - V(s') \leq \vo{t+1}(s') - \vp{t+1}(s').
\end{align}
Since these are positive quantities we can square them and preserve the order of the inequality:
\begin{align}
0 \leq \(\vo{t+1}(s') - V(s')\)^2 \leq \(\vo{t+1}(s') - \vp{t+1}(s')\)^2.
\end{align}
A linear combination of the above terms, weighted by the component of $p$, i.e., $p(s')$ gives the second moment squared:
\begin{align}
0 \leq \sum_{s'} p(s') \(\vo{t+1}(s') - V(s')\)^2 \leq \sum_{s'} p(s') \(\vo{t+1}(s') - \vp{t+1}(s')\)^2.
\end{align}
Taking square-root yields:
\begin{equation}
\devistarv{t+1}\leq \devi{t+1}.
\end{equation} 
Equation \ref{eqn:PessimismOverestimate} is proved analogously.
\end{proof}

The above lemma ensures the result below:
\begin{proposition}[Transition Bonus is Optimistic]
\label{prop:TransitionBonusIsOptimistic}
If the following condition hold:
\begin{enumerate}
\item $\phi$ is admissible 
\item $\vp{t+1} \leq \vtrue{t+1} \leq \vo{t+1}$ pointwise 
\end{enumerate}
then the following holds true:
\begin{align}
\label{eqn:BonusIsOptimistic}
\bbonus(\pohata{s}{t},\vo{t+1},\vp{t+1}) & \geq \phi(\potruea{s}{t},\vtrue{t+1}) \\
\label{eqn:BonusIsPessimistic}
\bbonus(\pohata{s}{t},\vp{t+1},\vo{t+1}) & \geq \phi(\potruea{s}{t},\vtrue{t+1}).
\end{align}
If condition in equation \ref{eqn:BonusIsOptimistic} is satisfied then we say that the bonus $\bbonus(\cdot,\cdot)$ is \emph{optimistic for \Alg{}}.
\end{proposition}
This is the key result that will ensure optimism of the algorithm later.
\begin{proof}
Condition $2)$ coupled with lemma \ref{lem:OptimismOverestimate} ensures:
\begin{equation}
\label{lemeqn:deviok}
\frac{\devihat{t+1}}{\sqrt{\nsaa{s}{t}{k}}} \geq \frac{\devihatstar{t+1}}{\sqrt{\nsaa{s}{t}{k}}}.
\end{equation}
Together, equation \ref{lemeqn:deviok} and assumption \ref{ass:FiniteTimeBonusBound} imply the first of the following inequalities on the bonus:
\begin{align}
\label{eqn:UniversalBonus}
& \bbonus(\pohata{s}{t},\vo{t+1},\vp{t+1}) \stackrel{def}{=} \phi(\pohata{s}{t},\vo{t+1}) + \ExtraBonus \\
& {\geq} \phi(\pohata{s}{t},\vo{t+1}) + \frac{4J + \BpPlus }{\sqrt{\nsaa{s}{t}{k}}} + \frac{\Bv \devihatstar{t+1}}{\sqrt{\nsaa{s}{t}{k}}} \\
& {\geq} \phi(\potruea{s}{t},\vtrue{t+1}).
\end{align}
while the second inequality is ensured by lemma \ref{lem:DeltaPhi}, completing the proof of equation \ref{eqn:BonusIsOptimistic}. Equation \ref{eqn:BonusIsPessimistic} is proved analogously.
\end{proof}
Proposition \ref{prop:TransitionBonusIsOptimistic} states that the bonus defined in Definition \ref{def:TransitionBonus}, which was constructed out of the admissible confidence interval for $\phi$, overestimates $\phi$. This is enough to guarantee optimism of the algorithm:

\subsection{Algorithm is Optimistic}
\begin{proposition}[Algorithm Brackets \vtrue{t}]
\label{prop:AlgorithmBracketsVstar}
Outside of the failure event if \Alg{} is run with an admissible $\phi$ then:
\begin{equation}
\label{eqn:AlgorithmBracketsVstar}
\vp{t} \leq \vtrue{t}  \leq \vo{t}  \quad \text{(pointwise)}
\end{equation}
holds for all timesteps $t$ and episodes $k$.
\end{proposition} 
\begin{proof}
We proceed by induction. Suppose equation \ref{eqn:AlgorithmBracketsVstar} holds for all states $s$ in timestep $t+1$. If 
\begin{equation}
\rhat{s}{t}+\br{s}{a} + \pohat{s}{t}^\top\vo{t+1} + \bbonus(\pohat{s}{t},\vo{t+1},\vp{t+1}) \geq H-t	
\end{equation}
holds then we are done. If the above does not hold then maximization over the actions in the optimistic MDP justifies the last inequality below, while the first inequality is justified by lemma \ref{lem:RewardBonusIsOptimistic}:
\begin{align}
\vo{t} & = \rohat{s}{t} + \br{s}{\piok(s,t)} + \pohat{s}{t}^\top\vo{t+1} + \bbonus(\pohat{s}{t},\vo{t+1},\vp{t+1}) \\ 
& \geq \rotrue{s}{t} + \pohat{s}{t}^\top\vo{t+1} + \bbonus(\pohat{s}{t},\vo{t+1},\vp{t+1}) \\ 
& \geq \rtrue{s}{t} + \phat{s}{t}^\top\vo{t+1} + \bbonus(\phat{s}{t},\vo{t+1},\vp{t+1}).
\end{align}
Next, the inductive hypothesis $\vtrue{t+1} \leq \vo{t+1}$ yields the following lower bound:
\begin{align}
& \geq \rtrue{s}{t} + \phat{s}{t}^\top\vtrue{t+1} + \bbonus(\phat{s}{t},\vo{t+1},\vp{t+1}).
\end{align}
Proposition \ref{prop:TransitionBonusIsOptimistic} finally gives:
\begin{align}
& \geq \rtrue{s}{t} + \phat{s}{t}^\top\vtrue{t+1} + \phi(\ptrue{s}{t},\vtrue{t+1})
\end{align}
Since $\phi$ is admissible we get:
\begin{align}
& \geq \rtrue{s}{t} + \ptrue{s}{t}^\top\vtrue{t+1} = \vtrue{t}(s)
\end{align}
This holds for every state $s$, completing the proof that \Alg{} is ``optimistic''. %Next $\votrue{t}  \leq \vtrue{t}$ holds by definition of optimal policy. 
It remains to show ``pessimism'', again by induction. If 
\begin{align}
	 \rohat{s}{t} - \br{s}{\piok(s,t)} + \pohat{s}{t}^\top\vp{t+1} - \bbonus(\pohata{s}{t},\vp{t+1},\vo{t+1}) \leq 0  
\end{align}
we are done. If this is not the case then an upper bound is given by proposition \ref{prop:TransitionBonusIsOptimistic} and lemma \ref{lem:RewardBonusIsOptimistic}:
\begin{align}
 \vp{t}(s) & = \rohat{s}{t} - \br{s}{\piok(s,t)} + \pohat{s}{t}^\top\vp{t+1} - \bbonus(\pohata{s}{t},\vp{t+1},\vo{t+1}) \\
 & \leq \rotrue{s}{t} + \pohat{s}{t}^\top\vp{t+1} - \bbonus(\pohata{s}{t},\vp{t+1},\vo{t+1}) \\
		 & \leq  \rotrue{s}{t} + \pohat{s}{t}^\top\vp{t+1} - \phi(\potrue{s}{t},\vtrue{t+1}) \\
\end{align}
The inductive hypothesis $ \vp{t+1} \leq \vtrue{t+1}$ justifies the following upper bound:
\begin{align}
\leq & \rotrue{s}{t} + \pohat{s}{t}^\top\vtrue{t+1} - \phi(\potrue{s}{t},\vtrue{t+1}) \\
\end{align}
Finally, since $\phi$ is admissible we get:
\begin{align}
\leq & \rotrue{s}{t} + \potrue{s}{t}^\top\vtrue{t+1}
\end{align}
By definition the optimal policy $\pi^*$ must achieve a higher value:
\begin{align}
\leq & \rtrue{s}{t} + \ptrue{s}{t}^\top\vtrue{t+1} = \vtrue{t}(s)
\end{align}
completing the proof.
\end{proof}

\section{Delta Optimism}% \texorpdfstring{\vo{t} - \vp{t}}{}}
\label{sec:DeltaOptimism}
\begin{proposition}[Delta Optimism \info{Check This, especially the move from hat p to p}]
\label{prop:DeltaOptimism}
Outside of the failure event for \Alg{}  it holds that:
\begin{align}
\label{eqn:DeltaOptimism}
\vo{t}(s) - \vp{t}(s)  \leq  \cc{\ref{prop:DeltaOptimism},3,2} & \sum_{\tau = t}^H \E \( \min \Big\{ \frac{\gtruea{s_{\tau}}{\tau}}{\sqrt{\nsaa{s_{\tau}}{\tau}{k}}} +  \frac{F}{\sqrt{\nsaa{s_{\tau}}{\tau}{k}}} + \frac{D}{\nsaa{s_\tau}{\tau}{k}},H\Big\} \Bigm| s,\piok{}{} \) \\
\leq \cc{\ref{prop:DeltaOptimism},3,3} & \sum_{\tau = t}^H \E \( \min \Big\{ \frac{F+D}{\sqrt{\nsaa{s_\tau}{\tau}{k}}},H \Big\} \Bigm| s,\piok{}{} \)
\end{align}
where 
\begin{align}
	F \stackrel{def}{=}& (H\sqrt{S} + \Bv H)L \label{eqn:F} \\
	D \stackrel{def}{=}& (J + \BpPlus)L \label{eqn:D}
\end{align}

and the conditional expectation $\E\(\cdot \mid s,\piok\)$ is with respect to the states $s_\tau$ encountered during the $k$-th episode upon following policy \piok{} after visiting state $s$ in timestep $t$. We use the convention that the terms in RHS corresponding to $n_k(s_{\tau},a) = 0$ are bounded by $H$, which is the maximum difference in $\vo{t}(s') - \vp{t}(s')$.
\end{proposition}

\begin{proof}
By the planning step for the action $a$ chosen by \Alg{} it holds that:
\begin{equation}
\begin{cases}
 \vo{t}(s) \leq \rhat{s}{a} + \br{s}{a} + \pohata{s}{t}^\top \vo{t+1} + \bbonus(\pohata{s}{t},\vo{t+1},\vp{t+1}) \\
 \vp{t}(s) \geq \rhat{s}{a} - \br{s}{a} + \pohata{s}{t}^\top \vp{t+1} - \bbonus(\pohata{s}{t},\vp{t+1},\vo{t+1})
\end{cases}
\end{equation}
Notice that the exploration bonuses $b^r_k,b^{pv}_k$ are bounded by $\tilde O(H)$ by construction in the algorithm, as well as the value function. This justifies the ``hard bound'' of $H$ that appears in equation \ref{eqn:DeltaOptimism}, which we drop for the rest of the proof to simplify the notation.
Subtraction yields:
\begin{equation}
\vo{t}(s) - \vp{t}(s) \leq \pohata{s}{t}^\top \( \vo{t+1} - \vp{t+1} \) + \bbonus(\pohata{s}{t},\vo{t+1},\vp{t+1}) +  \bbonus(\pohata{s}{t},\vp{t+1},\vo{t+1}) + 2\br{s}{a}.
\end{equation}
Now we substitute the definition of bonus (Definition \ref{def:TransitionBonus}) to obtain:
\begin{align}
\vo{t}(s) - \vp{t}(s) & \leq 2\br{s}{a} + \pohata{s}{t}^\top \( \vo{t+1} - \vp{t+1} \) \\
& +\phi(\pohata{s}{t},\vo{t+1}) + \ExtraBonus \\
& + \phi(\pohata{s}{t},\vp{t+1}) + \ExtraBonus.
\end{align}
With the help of lemma \ref{lem:DeltaPhi} we relate $\phi$ evaluated at the empirical quantities to the ``real'' $\phi(\potruea{s}{t},\vtrue{t+1})$, leading to the following upper bound of the above equation:
\begin{align}
& \leq  2\br{s}{a}+ \pohata{s}{t}^\top \( \vo{t+1} - \vp{t+1} \) + 2\phi(\potruea{s}{t},\vtrue{t+1}) +  4 \( \ExtraBonus \). 
\end{align}
Now, we want $\potruea{s}{t}^\top \( \vo{t+1} - \vp{t+1} \)$ to appear instead of $\pohata{s}{t}^\top \( \vo{t+1} - \vp{t+1} \)$ to do induction on the ``true'' MDP and so we add and subtract the former to obtain:
\begin{align}
& =\potruea{s}{t}^\top \( \vo{t+1} - \vp{t+1} \) + 2\br{s}{a} + \( \pohata{s}{t} - \potruea{s}{t} \)^\top \( \vo{t+1} - \vp{t+1} \) \\
& + 2\phi(\potruea{s}{t},\vtrue{t+1}) + 4 \( \ExtraBonus \).
\end{align}
and using the definition of $\phi$ we finally have:
\begin{align*}
& \leq 2\br{s}{a} + \potruea{s}{t}^\top \( \vo{t+1} - \vp{t+1} \) + \( \pohata{s}{t} - \potruea{s}{t} \)^\top \( \vo{t+1} - \vp{t+1} \) \\
& + 2\frac{g(\potruea{s}{t},\vtrue{t+1})}{\sqrt{\nsaa{s}{t}{k}}} + 2\frac{J}{\nsaa{s}{t}{k}}  + 4 \( \ExtraBonus \).
\numberthis{\label{eqn:InDeltaOptimismProofBeforeSimplify}}
\end{align*}
\info{%The following is really a quick way to conclude :)  A better way to do this is to use lemma \ref{lem:LowerOrderBridge}. 
Quite Remarkably these are all lower order terms except the one that contains $g$}. 
To deal with term $\( \pohata{s}{t} - \potruea{s}{t} \)^\top \( \vo{t+1} - \vp{t+1} \)$ we use Holder's inequality and the fact that we are outside of the failure event so that equation \ref{eqn:Weissman} holds:
\begin{align}
& \leq 2\br{s}{a} + \potruea{s}{t}^\top \(\vo{t+1} - \vp{t+1} \) + \|\pohata{s}{t} - \potrue{s}{t} \|_1 \| \votrue{t+1} - \vp{t+1} \|_{\infty}  \\
& + 2\frac{g(\potruea{s}{t},\vtrue{t+1})}{\sqrt{\nsaa{s}{t}{k}}} + 2\frac{J}{\nsaa{s}{t}{k}}  + 4 \( \ExtraBonus \) \\
& \leq 2\br{s}{a} + \potruea{s}{t}^\top \(\vo{t+1} - \vp{t+1} \) + H \sqrt{\frac{S}{\nsaa{s}{t}{k}}} \times L  \\
& + 2\frac{g(\potruea{s}{t},\vtrue{t+1})}{\sqrt{\nsaa{s}{t}{k}}} + 2\frac{J}{\nsaa{s}{t}{k}}  + 4 \( \frac{4J + \BpPlus}{\nsaa{s}{t}{k}} + \frac{\Bv H}{\sqrt{\nsaa{s}{t}{k}}} \). \\
\end{align}
Induction gives the statement when coupled with the fact that $\vo{t+1}(s)-\vp{t+1}(s) \leq H$ and that $\br{s}{a} \leq \cc{\ref{prop:DeltaOptimism},3,1} \frac{L}{\sqrt{\nsaa{s}{t}{k}}}$ from the definition of Bernstein inequality. The second inequality in the theorem statement is given by $\sqrt{n}\leq n$ for $n\geq 1$ coupled with corollary \ref{cor:MaxPhi}. Log factors are incorporated into $L$.
\end{proof}

\section{The ``Good'' Set \texorpdfstring{\boldmath{\Ltk}}{}}
\label{app:Ltk}
We now introduce the set $L_{k}$. The construction is similar to \cite{Dann17} although we modify it here for our to handle the regret framework (as opposed to PAC) under stationary dynamics. The idea is to partition the state-action space at each episode into two sets, the set of episodes that have been visited sufficiently often (so that we can lower bound these visits by their expectations using standard concentration inequalities) and the set of $(s,a)$ that were not visited often enough to cause high regret. In particular:
\begin{definition}[The Good Set]
\label{def:TheGoodSet}
The set $\Ltk{}$ is defined as:
\begin{equation}
\label{eqn:TheGoodSet}
\Ltk \stackrel{def}{=} \Big\{ (s,a) \in \mathcal S \times \mathcal A : \frac{1}{4}\sum_{j < k} w_{j}(s,a) \geq \lnnsa + H \Big\}.
\end{equation}
\end{definition}
The above definition enables the following lemma that relates the realized number of visits to a state to their visit probabilities:
\begin{lemma}[Visitation Ratio]
\label{lem:VisitationRatio}
Outside the failure event if $(s,a) \in \Ltk$ then
\begin{equation}
\label{eqn:VisitationRatio}
n_{k}(s,a) \geq \frac{1}{4}\sum_{j \leq k} w_{j}(s,a)
\end{equation}
holds.
\end{lemma}
\begin{proof}
Outside the failure event equation \ref{eqn:wnevent} justifies the first passage below:
\begin{align}
n_k(s,a) & \geq  \frac{1}{2}\sum_{j< k}w_{j}(s,a) - \lnnsa \\
& = \frac{1}{4}\sum_{j< k}w_{j}(s,a) +  \frac{1}{4}\sum_{j< k}w_{j}(s,a) - \lnnsa \geq \frac{1}{4}\sum_{j< k}w_{j}(s,a) + H \geq \frac{1}{4}\sum_{j< k}w_{j}(s,a) + w_k(s,a) \geq \frac{1}{4}\sum_{j\leq k}w_{j}(s,a)
\end{align}
while the second inequality holds because $(s,a) \in \Ltk$ by assumption and the third because $w_k(s,a) \leq H$.
\end{proof}
Finally, the following corollary ensures that if $(s,a) \not \in \Ltk$ then it will contribute very little to the regret:
\begin{lemma}[Minimal Contribution]
\label{lem:MinimalContribution}
It holds that:
\begin{align}
\sumk \sumt{1} \sumsnLk \visita{s}{t}{k} \leq   \cc{\ref{lem:MinimalContribution},3,1}SAHL.	
\end{align}
\end{lemma}
\info{Can Remove an H here and down in the lower order term. This shows up as one less $\sqrt{H}$ in the main result; Remove $L_{k}$ and put $L_k$, see cmab analysis}
\begin{proof}
By definition \ref{def:TheGoodSet}, if $(s,a) \not \in \Ltk$ then
\begin{equation}
\label{eqn:minimalcontri}
\frac{1}{4}\sum_{j \leq k} w_{j}(s,a) < \lnnsa + H
\end{equation} 
holds. Now sum over the $(s,a)$ pairs not in $\Ltk$, the timesteps $t$ and episodes $k$ to obtain:
\begin{align}
\sumk \sumt{1} \sumsnLk \visita{s}{t}{k} =  \sumsa \sumk \visita{s}{}{k} \1\{(s,a)\not \in \Ltk\} \leq \sumsa \( 4\lnnsa + 4H \) \leq \cc{\ref{lem:MinimalContribution},3,1}SAH\ll
\end{align}

\end{proof}
\section{Regret Analysis}
\label{app:RegretAnalysis}

We begin our regret analysis of \Alg{}. We will carry out the analysis outside of the failure event to derive a high probability regret bound.
%TODO Make the Main Result Appear
\MainResult{\label{thm:MainResult}}{\label{eqn:Complexity}}{\label{eqn:ComplexityPi}}{\label{eqn:ProblemDependentComplexityBound}}{\label{eqn:AlgorithmDependentComplexityBound}}

Note that $F$ and $D$ are identical to the definitions given in Equations~\ref{eqn:F} and~\ref{eqn:D} respectively. 
\begin{proof}
Outside of the failure event proposition \ref{prop:AlgorithmBracketsVstar} guarantees optimism and thus:
\begin{equation}
\label{eqn:Optimism}	
\realoptimalitygap \leq \optimalitygap
\end{equation}

holds for any state and time, and in particular in particular for $t=1$. Lemma E.15 in \cite{Dann17} is a standard decomposition that allows us to claim:
\begin{align}
& \textsc{Regret}(K) \stackrel{def}{=}   \sum_{k = 1}^K V^{\pi^*}_1(s) - V^{\piok}_1(s) \stackrel{Optimism}{\leq} \sum_{k = 1}^K \overline V^{\piok}_1(s) - V^{\piok}_1(s) \\
& \leq \sumk \sumt{1} \E \emmalemma{t} \\
&  = \sumk \sumt{1} \sumsLk \visita{s}{t}{k} \( \( \roa{s}{t} - \rotruea{s}{s} \) + \( \poa{s}{t} - \potruea{s}{t} \)^{\top} \vo{t+1} \) \\
& + \sumk \sumt{1}   \sumsnLk \visita{s}{t}{k} \underbrace{\( \( \roa{s}{t} - \rotruea{s}{s} \) + \( \poa{s}{t} - \potruea{s}{t} \)^{\top}\vo{t+1}    \)}_{\leq H} \\
&  = \sumk \sumt{1} \sumsLk \visita{s}{t}{k} \( \( \roa{s}{t} - \rotruea{s}{s} \) + \( \poa{s}{t} - \potruea{s}{t} \)^{\top} \vo{t+1} \) + c_{200,1} LSAH^2, 
\end{align}
for some constant $c_{200,1}$, where the bound in the last passage follows from Lemma \ref{lem:MinimalContribution}.
By adding and subtracting $\pohata{s}{t}^\top\vo{t+1}$ and also $\potruea{s}{t}^\top\vtrue{t+1}$ to the above we get the upper bound below:
\begin{align*}
& \leq \sum_{k=1}^K \sum_{t \in [H]} \sum_{(s,a)} w_{tk}(s,a) \Biggm( \underbrace{\(\tilde r_k(s,a) - r(s,a) \)}_{\text{Reward Estimation and Optimism}} + \underbrace{\(\tilde p_k(s,a) - \hat p(s,a) \)^\top \overline V_{t+1}^{\piok}}_{\text{Transition Dynamics Optimism}} \\ +  & \underbrace{\(\hat p_k(s,a) - p(s,a) \)^\top V_{t+1}^{*}}_{\text{Transition Dynamics Estimation}} + \underbrace{\(\hat p_k(s,a) - p(s,a) \)^\top \( \overline V_{t+1}^{\piok} - V^*_{t+1} \)}_{\text{Lower Order Term}}  \Biggm) + c_{200,1} LSAH^2. \\
\numberthis{\label{eqn:chrisdecomposition}}
\end{align*}	
Here $\visita{s}{t}{k}$ is the visitation probability to the $(s,a)$ pair at timestep $t$ in episode $k$.
Each term is bounded in lemmata \ref{lem:RewardEstimationAndOptimism},\ref{lem:TransitionDynamicsEstimation},\ref{lem:TransitionDynamicOptimims} and \ref{lem:LowerOrderTerm} to obtain
\begin{align}
& \leq \cc{\ref{thm:MainResult},1,3} \Bigm( \Bigm( \ROEbound + \TDObound + \\
	& \LOTbound+SAH^2 \Bigm)\ll^2)  \\
	& \leq \cc{\ref{thm:MainResult},1,4} \(  (\sqrt{\ComplexityReward{} SAT} + \TDEbound + \Bp SA +\sqrt{S}\DeltaGbound)\ll^3\) 
\end{align}
after simplification. 
Cauchy-Schwartz immediately implies the following bound:
\begin{align}
\leq \cc{\ref{thm:MainResult},1,5} \( (\sqrt{(\ComplexityReward{}+ \Complexity{}) SAT}  +\sqrt{S}\DeltaGbound)\ll^3 \) 
\end{align}
after absorbing the constant into the lower order term.

We now do the same argument but instead use the variants of Lemmas~\ref{lem:TransitionDynamicsEstimation} and \ref{lem:TransitionDynamicOptimims}  that express their bounds as a function of $\ComplexityPi$. This yields: 
\begin{align}
\leq  \cc{\ref{thm:MainResult},1,6} \((\sqrt{(\ComplexityReward{}+ \ComplexityPi{}) SAT}  +\sqrt{S}\DeltaGbound)\ll^3 + \Bv\sqrt{SAH^2\Regret(K)\ll} \).
\end{align}
We can re-express this regret bound as follows. Outside of the failure event, and where $s_{1k}$ is the arbitrary starting state in episode $k$), we have:
\begin{align}
\sumk \( \vtrue{1}(s_{1k}) - \votrue{1k}(s_{1k}) \) \stackrel{def}{=} \Regret(K) \leq \cc{\ref{thm:MainResult},1,6}  Y + \cc{\ref{thm:MainResult},1,6} M\sqrt{\Regret(k)}	
\end{align}
with 
\begin{align}
Y & = 	(\sqrt{(\ComplexityReward{}+ \ComplexityPi{}) SAT}  +\sqrt{S}\DeltaGbound)\ll^3 \\
M & = \Bv\sqrt{SAH^2\ll}.
\end{align}
This is satisfied as long as:
\begin{align}
\Regret(K)- \cc{\ref{thm:MainResult},1,6} M\sqrt{\Regret(k)}  - \cc{\ref{thm:MainResult},1,6}  Y \leq  0.
\end{align}
We can solve the quadratic equation (quadratic in $\sqrt{\Regret(K)}$). This implies that the largest that $\sqrt{\Regret(K)}$ can be is:
\begin{align}
	\sqrt{\Regret(K)} \leq \frac{1}{2}\(  \cc{\ref{thm:MainResult},1,6}  M + \sqrt{\cc{\ref{thm:MainResult},1,6}^2  M^2 + 4\cc{\ref{thm:MainResult},1,6} Y }\).
\end{align}
By squaring and applying Cauchy-Schwartz we obtain 
\begin{align}
	\Regret(K) \leq  \cc{\ref{thm:MainResult},1,6}^2  M^2 + \cc{\ref{thm:MainResult},1,6}^2  M^2 + 4\cc{\ref{thm:MainResult},1,6} Y, 
\end{align}
completing the proof of the main result.
\end{proof}

\subsection{Regret Bounds with Bernstein Inequality}
We now specialize the result of Theorem \ref{thm:MainResult} when Bernstein Inequality is used. First we check that Bernstein's Inequality satisfies assumption \ref{ass:ConfidenceIntervals} and \ref{ass:FiniteTimeBonusBound}. Bernstein's inequality guarantees that with probability at least $1-{\delta'}$ we have that:
\begin{equation}
\label{eqn:Bernstein}
| \(\hat p_k(s,a) - p(s,a)\)^\top\vtrue{t+1} | \leq \Bernstein{} \stackrel{def}{=}  \phi(p(s,a),\vtrue{t+1}).  
\end{equation}
after a union bound on the number of states $S$,  actions $A$ and visits $1,...,T$ to the specific state-action pair $(s,a)$.

Proposition \ref{prop:BernsteinIsAdmissible} combined with Theorem \ref{thm:MainResult} and a recursive application of the law of total variance is the proof of the following proposition:
\BernsteinResult{}
\begin{proof}
Proposition \ref{prop:BernsteinIsAdmissible} shows that Bernstein Inequality of equation \ref{eqn:BernsteinInequality} is admissible with $\Bp = \tilde O \( H \)$, $\Bv = \tilde O \( 1\)$, $J = \tilde O \(H\)$ so that $F+D = c_{200,3}\ll H\sqrt{S}$, for some constant $c_{200,3}$, by direct computation. This allows us to apply Theorem \ref{thm:MainResult} and compute an explicit form for the lower order term and the constants $\Complexity{}$, $\ComplexityPi{}$.

To obtain the problem dependent bound notice that with the definition of $\ComplexityQ{}$ in the main text in equation \ref{main:Complexity} and of the $\Qrv_t(\cdot,\cdot)$ random variables in the same section \info{be more precise} in the main text:
\begin{align}
\ComplexityReward{} + \Complexity{} & \stackrel{def}{=} \\
& = \frac{1}{T} \sumk \sumt{1}\E_{(s,a) \sim \piok} \( \Var \(R(s,a) \mid (s,a) \) \)   +  \frac{1}{T} \sumk \sumt{1}\E_{(s,a) \sim \piok} \Var \(  \vtrue{t+1}(s^+) \mid (s,a)\) \\
& = \frac{1}{T} \sumk \sumt{1}\E_{(s,a) \sim \piok} \( \Var \(R(s,a) \mid (s,a) \)  + \Var \(  \vtrue{t+1}(s^+) \mid (s,a)\) \) \\
& \stackrel{(a)}{=} \frac{1}{T} \sumk \sumt{1}\E_{(s,a) \sim \piok} \( \Var \(R(s,a)  + \vtrue{t+1}(s^+) \mid (s,a)\) \) \\
& = \frac{1}{T} \sumk \sumt{1}\E_{(s,a) \sim \piok} \( \Var \(\Qrv_t(s,a) \mid (s,a)\) \) \stackrel{def}{\leq} \ComplexityQ{}
\end{align}
Notice that $(a)$ follows by independence of the sampled reward and transition given an $(s,a)$ pair. This gives the problem dependent bound (also in the main text, Theorem \ref{thm:MainResultMainText}). 

To obtain the problem-independent worst case guarantee we use a Law of Total Variance argument.
Using the variant given by equation \ref{eqn:AlgorithmDependentComplexityBound} in Theorem \ref{thm:MainResult} we need to bound :
\begin{align}
\ComplexityPi{} & = \frac{1}{T}\sumk\sumt{1}\E_{\piok}\(  \Var_{\piok} \(\votrue{t+1}(s_{t+1}) \Big| s_t \) \Big| s_1 \) \\ 
& = \frac{1}{T}\sumk \E_{\piok}\( \(\sumt{1} \rotrue{s_t}{t} - \votrue{1}(s_1) \)^2 \Big| s_1 \)  \leq \frac{1}{T}K\MaxReturn^2 = \frac{\MaxReturn^2}{H}
\end{align}
where the second equality follows from a law of total variance argument (see \cite{Azar17} for example) reproduced in lemma \ref{lem:LTV} yielding the stated worst case bound.
Expression $\E_{\piok}\( \(\sumt{1} \rotrue{s_t}{t} - \votrue{1}(s_1) \)^2 \Big| s_1 \)$ is the variance of the returns (with fixed rewards) induced by the MDP dynamics upon starting from $s_1$ and following $\piok$. Since the random per episode return $\sumt{1} R(s_t,\piok(s_t)) \leq \MaxReturn$, it must be that $\sumt{1} r(s_t,\piok(s_t)) \leq \MaxReturn$ as well.
%\footnote{Notice that the upper bound on the return is not a `constraint' on the rewards: the rewards random variables are still independent}. 
The variance of a random variable is upper bounded by the range square, justifying the inequality. Finally, 
plugging in $\ComplexityPi{}$ and $\ComplexityReward{}$ into equation \ref{eqn:AlgorithmDependentComplexityBound} in Theorem \ref{thm:MainResult} concludes the proof of the result.

This proposition is also restated in the main text as Theorem \ref{thm:MainResultMainText}.
\end{proof}

\subsection{Regret Bound in Deterministic Domain with Bernstein Inequality}
\label{sec:DeterministicDomainAnalysis}
We now examine the regret of \Alg{} when used with Bernstein Inequality in deterministic domains.

\DeterministicDomainRegret{}
\begin{proof}
Define as $\mathcal{N}$ as the set of episodes in which the agent visits an $(s,a)$ that is not in $\Ltk$. Since the domain is deterministic, each time an $(s,a)$ pair is visited we have $\visita{s}{t}{k} = 1$ 
and hence there can be at most $\tilde O(H)$ episodes in which $(s,a)$ is visited but $(s,a) \not \in \Ltk$. Since there are at most $SA$ state and action pairs we have that there are at most $\tilde O(SAH)$ such episodes, with a regret at most $\tilde O(SAH^2)$. Therefore for any starting state $s_k$:
\begin{equation}
\sum_{k \in \mathcal N} \vtrue{1}(s_k) - \votrue{1}(s_k) \leq \cc{\ref{prop:DeterministicDomainRegret},2,3}  SAH^2.
\end{equation}
Under the episodes not in $\mathcal N$ there is zero probability of visiting a new $(s,a)$ pair, and therefore the maximum likelihood estimate the transition probability is exact. That is, using optimism $(a)$:
\begin{align}
	\textsc{Regret}(K) & = \sumk \vtrue{1}(s_k) - \votrue{1}(s_k) \stackrel{(a)}{\leq} \sumk \vo{1}(s_k) - \votrue{1}(s_k) \\
	& = \sum_{k \not \in \mathcal N} \vo{1}(s_k) - \votrue{1}(s_k) + \sum_{k \in \mathcal N} \vo{1}(s_k) - \votrue{1}(s_k) \\	
	& \leq \sum_{k \not \in \mathcal N} \vtrue{1}(s_k) - \votrue{1}(s_k) +  \cc{\ref{prop:DeterministicDomainRegret},2,3} SAH^2 \\
	& \leq \sum_{k \not \in \mathcal N} \sumt{1} \sumsa \visita{s}{t}{k} \(\tilde r(s,a) - r(s,a) + \bonus{s}\) +  \cc{\ref{prop:DeterministicDomainRegret},2,3}SAH^2.
\end{align}
\paragraph{Bounding the Rewards}
An application of lemma \ref{lem:RewardEstimationAndOptimism} yields:
\begin{equation}
\sum_{k \not \in \mathcal N} \sumt{1} \sumsa \visita{s}{t}{k} \(\tilde r(s,a) - r(s,a)\) \leq \cc{\ref{prop:DeterministicDomainRegret},2,1} SA\ll^3	
\end{equation}
since $\ComplexityReward{} = 0$. The lemma can be applied because if an episode is not in $\mathcal N$ then all $(s,a)\in \Ltk$. 

For the rest of the proof we focus on bounding the exploration bonus:
\begin{align}
	& \sum_{k \not \in \mathcal N} \sumt{1} \sumsa \visita{s}{t}{k} \bonus{s}  \\
	& \leq \sum_{k \not \in \mathcal N} \sumt{1} \sumsa \visita{s}{t}{k} b(\pohata{s}{t},\vo{t+1},\vp{t+1}) \\
	& \leq \sum_{k \not \in \mathcal N} \sumt{1} \sumsa \visita{s}{t}{k} \( \phi(\pohata{s}{t},\vo{t+1}) + \ExtraBonus \) . 
\end{align}
using the definition of bonus \ref{def:TransitionBonus} (here $\phi(\cdot,\cdot)$ is the true Bernstein Inequality evaluated with the empirical quantities).
Before bounding the above term we need to understand how the correction term behaves.

\paragraph{Bounding the Delta Optimism in Deterministic Domains}
We wish to show that
\begin{equation}
\label{eqn:DeltaOptimismDeterministic}
\vo{t}(s) - \vp{t}(s) \leq C_1 \sumtau{t} \frac{H}{\nsaa{s_{\tau}}{t}{k}}\times \ll	
\end{equation}
where $C_1$ is some absolute numeric constant and $s_t$ are the states encountered upon following the agent chosen policy. To achieve this proceed as in proposition \ref{prop:DeltaOptimism} until equation \ref{eqn:InDeltaOptimismProofBeforeSimplify} to get:
\begin{align}
\vo{t}(s) - \vp{t}(s) & \leq 2\br{s}{a} + \potruea{s}{t}^\top \( \vo{t+1} - \vp{t+1} \) + \( \pohata{s}{t} - \potruea{s}{t} \)^\top \( \vo{t+1} - \vp{t+1} \) \\
& + 2\frac{g(\potruea{s}{t},\vtrue{t+1})}{\sqrt{\nsaa{s}{t}{k}}} + 2\frac{J}{\nsaa{s}{t}{k}}  + 4 \( \ExtraBonus \) \\
& \stackrel{(a)}{=} 2\br{s}{a} + \potruea{s}{t}^\top \( \vo{t+1} - \vp{t+1} \) + 2\frac{J}{\nsaa{s}{t}{k}}  + 4 \( \ExtraBonus \)
\end{align}
where $(a)$ follows from the fact that the maximum likelihood is exact for episodes not in $\mathcal{N}$ and so the relevant terms above vanish from the expression. If Bernstein Inequality is used then as explained in proposition \ref{prop:BernsteinIsAdmissible} $\Bv = \tilde O(1),\Bp = \tilde O(H), J=\tilde O(H)$ and also $\br{s}{a} = C_1/\nsaa{s}{t}{k} \times \polylog$ since both the variance and the empirical variance are zero. Therefore for appropriate constants $C_1,C_2,\dots$ the above inequality can be written as:
\begin{align}
& = \( \frac{C_1+C_2H}{\nsaa{s}{t}{k}} + (\vo{t+1}(s_{t+1}) - \vp{t+1}(s_{t+1})) + C_3\sqrt{\frac{(\vo{t+1}(s_{t+1}) - \vp{t+1}(s_{t+1}))^2}{\nsaa{s}{t}{k}}} \) \ll \\
& = \( \frac{C_1+C_2H}{\nsaa{s}{t}{k}} + (\vo{t+1}(s_{t+1}) - \vp{t+1}(s_{t+1})) + C_3\frac{\vo{t+1}(s_{t+1}) - \vp{t+1}(s_{t+1})}{\sqrt{\nsaa{s}{t}{k}}} \) \ll \\
& \leq \( \frac{C_1+C_2H}{\sqrt{\nsaa{s}{t}{k}}} + (\vo{t+1}(s_{t+1}) - \vp{t+1}(s_{t+1})) + C_3 \frac{H}{\sqrt{\nsaa{s}{t}{k}}} \) \ll \\
& \leq (\vo{t+1}(s_{t+1}) - \vp{t+1}(s_{t+1})) + C_4 \frac{H}{\sqrt{\nsaa{s}{t}{k}}} \ll \\
& \leq \sumtau{t} C_4 \frac{H}{\sqrt{\nsaa{s_{\tau}}{t}{k}}} \ll \\
\end{align}
The last passage follows by induction and completes the proof of equation \ref{eqn:DeltaOptimismDeterministic} for episodes $\not \in \mathcal{N}$. Equipped with this it remains to bound the bonus on the transition dynamics.

\paragraph{Bounding the Transition Dynamics}
Proceed as in lemma \ref{lem:TransitionDynamicOptimims} up to equation \ref{eqn:InOptimismForTransitionDynamicsKeyEquation}. Since we are using Bernstein Inequality, $\Complexity{} = 0, \Bp = \tilde O(H),J=\tilde O(H)$ in deterministic domains and so the regret reads:
\begin{align}
& \sum_{k \not \in \mathcal N} \sumt{1} \sumsa \visita{s}{t}{k} \bonus{s} \leq \cc{\ref{prop:DeterministicDomainRegret},2,6} \Bigm( SAH  + \\
&  + \underbrace{\Bv \sqrt{\sum_{k \not \in \mathcal N} \sumt{1}\sumsLk \frac{\visita{s}{t}{k}}{\nsaa{s}{t}{k}}}}_{\tilde O (SA)} \times \sqrt{ \sum_{k \not \in \mathcal N} \sumt{1}\sumsLk \visita{s}{t}{k} \devihat{t+1}^2} \\
&  \cc{\ref{prop:DeterministicDomainRegret},2,7} SAH + \cc{\ref{prop:DeterministicDomainRegret},2,8}\sqrt{SA \ll} \times \sqrt{ \sum_{k \not \in \mathcal N} \sumt{1}\sumsLk \visita{s}{t}{k} \devihat{t+1}^2}
\Bigm)
\end{align}
We focus on the last factor:
\begin{align}
&\sqrt{ \sum_{k \not \in \mathcal N} \sumt{1}\sumsLk \visita{s}{t}{k} \pohata{s}{t}(\vo{t+1} - \vp{t+1})^2} = \sqrt{ \sum_{k \not \in \mathcal N} \sumt{1}\sumsLk (\vo{t+1}(s_{t+1}) - \vp{t+1}(s_{t+1}))^2} \\
& \leq \cc{\ref{prop:DeterministicDomainRegret},2,20} \sqrt{ \sum_{k \not \in \mathcal N} \sumt{1}\sumsLk \(\sumtau{t} \frac{H}{\sqrt{\nsaa{s_{\tau}}{t}{k}}}\)^2} \\
& \leq  \cc{\ref{prop:DeterministicDomainRegret},2,21} \sqrt{H \sum_{k \not \in \mathcal N} \sumt{1}\sumsLk \sumtau{t} \(\frac{H}{\sqrt{\nsaa{s_{\tau}}{t}{k}}}\)^2} \\
& \leq  \cc{\ref{prop:DeterministicDomainRegret},2,22} \sqrt{H^3 \sum_{k \not \in \mathcal N} \sumt{1}\sumsLk \sumtau{t} \(\frac{1}{\sqrt{\nsaa{s_{\tau}}{t}{k}}}\)^2} \\
& \leq \cc{\ref{prop:DeterministicDomainRegret},2,23} \sqrt{H^3 \sum_{k \not \in \mathcal N} \sumt{1}\sumsLk \sumtau{1} \(\frac{1}{\sqrt{\nsaa{s_{\tau}}{t}{k}}}\)^2} \\
& \leq  \cc{\ref{prop:DeterministicDomainRegret},2,24} \sqrt{H^4 \sum_{k \not \in \mathcal N} \sumt{1}\sumsLk \frac{\visita{s}{t}{k}}{\nsaa{s}{t}{k}}} \\
& \leq \cc{\ref{prop:DeterministicDomainRegret},2,25} \sqrt{SA\ll}H^2 
\end{align}
Thus
\begin{align}
& 	\sum_{k \not \in \mathcal N} \sumt{1} \sumsa \visita{s}{t}{k} \bonus{s} \leq \cc{\ref{prop:DeterministicDomainRegret},2,7} SAH\polylog +\cc{\ref{prop:DeterministicDomainRegret},2,9} \sqrt{SA\ll}  \times \cc{\ref{prop:DeterministicDomainRegret},2,10} \sqrt{SA \ll}H^2\\
	 & \leq \cc{\ref{prop:DeterministicDomainRegret},2,11} SAH^2\ll.
\end{align}
\paragraph{Concluding the Proof of the Regret Bound on Deterministic Domain}
Summing the regret for episodes not in $\mathcal N$, the reward optimism and the transition dynamics optimism one obtains the final regret bound of order:

\begin{align}
	\leq \cc{\ref{prop:DeterministicDomainRegret},2,30} SAH^2 \ll^3.
\end{align}
Notice that there are no failure events to consider, so this is a deterministic statement.
\end{proof}

\subsection{Reward Estimation and Optimism}
% Before bounding the reward, we introduce the following definition of maximum return within an episode.

\begin{lemma}[Reward Estimation and Optimism]
\label{lem:RewardEstimationAndOptimism}
Outside of the failure event it holds that:
\begin{eqnarray}
\sumk \sumt{1}\sumsLk  \visita{s}{t}{k} \( \reward{s} \)  &\leq&  \cc{\ref{lem:RewardEstimationAndOptimism},3,1} \ll^3 \( \left( \ROEbound \right) \)  \\
&=& \cc{\ref{lem:RewardEstimationAndOptimism},3,2}  \ll^3 \times \left(\sqrt{\frac{\MaxReturn^2}{H}SAT} + SA \right).
\end{eqnarray}
where 
\begin{equation}
\label{eqn:ComplexityReward}
\ComplexityReward{} = \frac{1}{T} \( \sumk \sumt{1}\sum_{(s,a)\in \Ltk}  \visita{s}{t}{k} \Var R(s,a) \) \leq \frac{\MaxReturn^2}{H}
\end{equation}

\end{lemma}
\begin{proof}
The optimistic reward is obtained by adding the reward bonus the empirical reward estimate:
\begin{align}
& \sumk \sumt{1}\sumsLk  \visita{s}{t}{k} \( \reward{s} \)  \leq \sumk \sumt{1}\sumsLk \visita{s}{t}{k} \br{s}{a} \\
& \leq \cc{\ref{lem:RewardEstimationAndOptimism},3,4} \sumk \sumt{1}\sumsLk \visita{s}{t}{k} \( \sqrt{\frac{2 \widehat \Var R(s,a) \ln \( \frac{4SAT}{\delta'} \) }{\nsaa{s}{t}{k}}} + \frac{7 \ln \( \frac{4SAT}{\delta'} \)}{3 \nsaa{s}{t}{k}} \)  \\
& \leq \cc{\ref{lem:RewardEstimationAndOptimism},3,5} \sumk \sumt{1}\sumsLk \visita{s}{t}{k} \( \sqrt{\frac{ \widehat \Var R(s,a)}{\nsaa{s}{t}{k}}} + \frac{1}{ \nsaa{s}{t}{k}} \)  \times   3 \ln \( \frac{4SAT}{\delta'} \) \\
& \leq \cc{\ref{lem:RewardEstimationAndOptimism},3,6} \sumk \sumt{1}\sumsLk \visita{s}{t}{k} \( \sqrt{\frac{ \( \sqrt{ \Var R(s,a)} + \sqrt{2 \ln (2SAT / \delta') / \nsaa{s}{t}{k}}   \)^2 }{\nsaa{s}{t}{k}}} + \frac{1}{\nsaa{s}{t}{k}} \) \times  3 \ln \( \frac{4SAT}{\delta'} \) \\ 
& \leq \cc{\ref{lem:RewardEstimationAndOptimism},3,7} \sumk \sumt{1}\sumsLk \visita{s}{t}{k} \( \sqrt{\frac{ \Var R(s,a)}{\nsaa{s}{t}{k}} } + \sqrt{ \frac{ 2 \ln (2SAT / \delta') / \nsaa{s}{t}{k}  }  { \nsaa{s}{t}{k} }  }+\frac{1}{\nsaa{s}{t}{k}} \) \times  3 \ln \( \frac{4SAT}{\delta'} \) \\
& \leq \cc{\ref{lem:RewardEstimationAndOptimism},3,8} \sumk \sumt{1}\sumsLk \visita{s}{t}{k} \( \sqrt{\frac{ \Var R(s,a)}{\nsaa{s}{t}{k}} } + \frac{ \sqrt{  2 \ln (2SAT / \delta') }}  { \nsaa{s}{t}{k} }  +\frac{1}{\nsaa{s}{t}{k}} \) \times  3 \ln \( \frac{4SAT}{\delta'} \) \\
& \leq \cc{\ref{lem:RewardEstimationAndOptimism},3,9}   \ll^2  \times \sumk \sumt{1}\sumsLk \visita{s}{t}{k} \( \sqrt{\frac{ \Var R(s,a)}{\nsaa{s}{t}{k}} } +  \frac{1}{\nsaa{s}{t}{k}} \)  \\
& \leq \cc{\ref{lem:RewardEstimationAndOptimism},3,10} \ll^2  \times \left( \sqrt{\sumk \sumt{1}\sumsLk 
\frac{\visita{s}{t}{k}}{\nsaa{s}{t}{k}}}  \sqrt{ \sumk \sumt{1}\sumsLk \visita{s}{t}{k} \Var R(s,a)}  + \right. \\
& \left( \sumk \sumt{1}\sumsLk  \frac{\visita{s}{t}{k}}{\nsaa{s}{t}{k}} \right)  \label{eqn:extra_term}
\end{align}
where the fourth line follows from lemma \ref{lem:DeltaPhiReward} that bounds the difference 
between the empirical and estimated variances, and the following inequalities come from algebraic 
manipulations and consolidating the $\polylog(S,A,H,T,1/\delta')$ terms into a single expression 
and moving this to outside the sum (since they are independent of the variables in the sum). 
The final inequality follows from Cauchy Schwartz, yielding the result after the application of lemma \ref{lem:wn}.

To compute the upper bound of equation \ref{eqn:ComplexityReward} we proceed as follows. The $\visita{s}{t}{k}$ are the probability of visiting state $s$ and taking action $a$ there in timestep $t$ of episode $k$ given the policy selected by the agent in episode $k$.  The core idea is that $\ComplexityReward{}$ is a per-step average of the reward variance within an episode. Regardless of the policy followed by the agent, the sum of reward random variables $R(\cdot,\cdot)$ cannot exceed \MaxReturn{}. Notice that the rewards random variables are independent when conditioned on the trajectories. In particular, for any fixed trajectory $s_1,\dots,s_H$ and any fixed policy $\pi$ we have that:
\begin{align}
\sum_{t=1}^H R(s_t,\pi(s_t,t)) \leq \MaxReturn	
\end{align}
by definition \ref{def:MaxReturn}. This in particular holds for $\piok$.
Squaring yields:
\begin{align}
\( \sum_{t=1}^H R(s_t,\piok(s_t,t))\)^2 \leq \MaxReturn^2	
\end{align}
We now take expectation over the reward random variables, still conditioned on the trajectory $s_1,\dots,s_H$ to obtain:
\begin{align}
\E \( \sum_{t=1}^H R(s_t,\piok(s_t,t)) \mid s_1,\dots,s_H \)^2 \leq \MaxReturn^2.
\end{align}
Using $\Var X \leq \E X^2$ for a generic random variable $X$ we can write:
\begin{align}
\Var \( \sum_{t=1}^H R(s_t,\piok(s_t,t)) \mid s_1,\dots,s_H\) \leq \MaxReturn^2.
\end{align}
Now we can take the expectation over the trajectories induced by $\piok$ to obtain:
\begin{align}
\E_{(s_1,\dots,s_H)} \Var \( \sum_{t=1}^H R(s_t,\piok(s_t)) \mid s_1,\dots,s_H\) \leq \MaxReturn^2.
\end{align}
Conditioned on the state-action, the reward random variables are independent and so we can take the summation outside:
\begin{align}
\E_{(s_1,\dots,s_H)} \sum_{t=1}^H \Var \(  R(s_t,\piok(s_t)) \mid s_1,\dots,s_H\) \leq \MaxReturn^2.
\end{align}
Written with the usual notation:
\begin{align}
\sumt{1}\sumsa \visita{s}{t}{k} \Var R(s,a)\leq \MaxReturn^2.
\end{align}
Finally summing over $k$ and dividing by the time elapsed:
\begin{align}
\ComplexityReward{} =
 \frac{1}{T}\sumk \sumt{1}\sumsa \visita{s}{t}{k} \Var R(s,a)\leq \frac{K}{T}\MaxReturn^2 = \frac{\MaxReturn^2}{H}.
\end{align}

\end{proof}

\subsection{Transition Dynamics Estimation}

\begin{lemma}[Transition Dynamics Estimation]
\label{lem:TransitionDynamicsEstimation}
Outside of the failure event if $\phi$ is admissible then it holds that:
\begin{equation}
\sumk \sumt{1}\sumsLk  \visita{s}{t}{k} \estimationa{s} \leq \cc{\ref{lem:TransitionDynamicsEstimation},3,0}  \(\TDEbound\)\ll.
\end{equation}
The following bound also holds:
\begin{align}
& \sumk \sumt{1}\sumsLk  \visita{s}{t}{k} \estimationa{s} \\
& \leq \cc{\ref{lem:TransitionDynamicsEstimation},3,1} \( \(\TDEboundPi + \DeltaGbound \) \ll + \Bv \sqrt{SAH^2\Regret(K)\ll} \).
\end{align}
\end{lemma}

\begin{proof}
Using the definition of $\phi$, outside of the failure event it holds that:
\begin{align}
& \sumk \sumt{1}\sumsLk  \visita{s}{t}{k} \estimationa{s} \stackrel{}{\leq} \sumk \sumt{1}\sumsLk  \visita{s}{t}{k} \( \frac{\dg{s}{t+1}}{\sqrt{\nsaa{s}{t}{k}}} + \frac{J}{\nsaa{s}{t}{k}} \).
\end{align}
Next, Cauchy-Schwartz justifies the following upper bound:
\begin{align}
&\stackrel{}{\leq}  \sqrt{\sumk \sumt{1}\sumsLk  \visita{s}{t}{k}\(\dg{s}{t+1}\)^2} \sqrt{\sumk \sumt{1}\sumsLk \frac{\visita{s}{t}{k} }{\nsaa{s}{t}{k}}} + J\sumk \sumt{1}\sumsLk  \frac{\visita{s}{t}{k} }{\nsaa{s}{t}{k}}
\end{align}
Finally, using lemma \ref{lem:wn} and definition \ref{def:TransitionBonus} of $\Complexity $ we can obtain the statement:
\begin{align}
&\stackrel{}{\leq}  \cc{\ref{lem:TransitionDynamicsEstimation},3,2}  \sqrt{\sumk \sumt{1}\sumsLk  \visita{s}{t}{k}\(\dg{s}{t+1}\)^2} \times \cc{\ref{lem:TransitionDynamicsEstimation},3,3}  \(\sqrt{SA \ll} \)  +  \cc{\ref{lem:TransitionDynamicsEstimation},3,4} \(JSA\)\ll \leq \cc{\ref{lem:TransitionDynamicsEstimation},3,5} \( \TDEbound \)\ll.
\end{align}
To obtain the second bound, we use a similar argument coupled with lemma \ref{lem:BoundBridge}.
\end{proof}

\subsection{Transition Dynamics Optimism}
\begin{lemma}[Transition Dynamics Optimism]
\label{lem:TransitionDynamicOptimims}
Outside of the failure event if $\phi$ is admissible it holds that:
\begin{align}
& \sumk \sumt{1}\sumsLk  \visita{s}{t}{k} \bonus{s} = \\
 & \leq  \cc{\ref{lem:TransitionDynamicOptimims},3,1} \(\TDObound\)\ll^2.
\end{align}
The bound below also hold:
\begin{align}
& \sumk \sumt{1}\sumsLk  \visita{s}{t}{k} \bonus{s} = \\
 & \leq \cc{\ref{lem:TransitionDynamicOptimims},3,2}\(  \( \TDOboundPi\) \ll^2 + \Bv \sqrt{SAH^2\Regret(K)\ll} \).
\end{align}
\end{lemma}
\begin{proof}
We begin by using definition \ref{def:TransitionBonus} (for the bonus) to justify $(a)$:	
\begin{align}
& \sumk \sumt{1}\sumsLk  \visita{s}{t}{k} \bonus{s} \\
& \stackrel{a}{\leq} \cc{\ref{lem:TransitionDynamicOptimims},3,3} \sumk \sumt{1}\sumsLk  \visita{s}{t}{k} \( \phi(\pohata{s}{t},\vo{t+1}) + \frac{\Bv \devihat{t+1}}{\sqrt{\nsaa{s}{t}{k}}} + \frac{\BpPlus+J}{\nsaa{s}{t}{k}}\) \\
& \stackrel{b}{\leq} \cc{\ref{lem:TransitionDynamicOptimims},3,4} \sumk \sumt{1}\sumsLk  \visita{s}{t}{k} \( \phi(\potruea{s}{t},\vtrue{t+1}) + \frac{\Bv \devihat{t+1}}{\sqrt{\nsaa{s}{t}{k}}} + \frac{\BpPlus+J}{\nsaa{s}{t}{k}}\)
\end{align}
while $(b)$ is justified by lemma \ref{lem:DeltaPhi} and \ref{lem:OptimismOverestimate}. 
Using the functional form for $\phi$ we obtain the upper bound below $(c)$:
\begin{align}
& \stackrel{c}{\leq} \cc{\ref{lem:TransitionDynamicOptimims},3,4} \sumk \sumt{1}\sumsLk  \visita{s}{t}{k} \( \underbrace{\frac{\dg{s}{t+1}}{\sqrt{\nsaa{s}{t}{k}}}  + \frac{J+\BpPlus}{\nsaa{s}{t}{k}} }_{\approx \text{Transition Dynamics Estimation}} + \underbrace{\frac{ \Bv \devihat{t+1} }{\sqrt{\nsaa{s}{t}{k}}}}_{\text{Lower Order Term}} \)
\end{align}

The term ``$\approx$ Transition Dynamics Estimation'' is nearly identical to what appears in the proof of lemma \ref{lem:TransitionDynamicsEstimation} and can be bounded in the same way. That is, apply Cauchy-Schwartz first and use lemma \ref{lem:wn} along with the definition of $\Complexity$ to get to the bound below:
\begin{align}
& \stackrel{d}{\leq} \cc{\ref{lem:TransitionDynamicOptimims},3,5} \sqrt{\sumk \sumt{1}\sumsLk \frac{\visita{s}{t}{k}}{\nsaa{s}{t}{k}}}  \sqrt{\sumk \sumt{1}\sumsLk \visita{s}{t}{k}\dg{s}{t+1}^2} + (J+\BpPlus)\sumk \sumt{1}\sumsLk \frac{\visita{s}{t}{k}}{\nsaa{s}{t}{k}} \\
& \leq  \cc{\ref{lem:TransitionDynamicOptimims},3,6} \( \sqrt{\Complexity SAT} + (J+\BpPlus)SA \) \ll.
\end{align}
Now we turn our attention to the ``Lower Order Term'' and apply Cauchy-Schwartz to get:
\begin{align}
\label{eqn:InOptimismForTransitionDynamicsKeyEquation}
&  \leq \cc{\ref{lem:TransitionDynamicOptimims},3,7} \Bv \sqrt{\sumk \sumt{1}\sumsLk \frac{\visita{s}{t}{k}}{\nsaa{s}{t}{k}}} \times \sqrt{ \sumk \sumt{1}\sumsLk \visita{s}{t}{k} \devihat{t+1}^2}
\end{align}
The first factor is bounded by $ c_{200,5} \Bv\sqrt{SA \ll}$ (for some constant $c_{200,5}$ by Lemma \ref{lem:wn}. Notice that
\begin{align}
	& \devihat{t+1}^2 = \pohat{s}{t}^\top \(\vo{t+1} - \vp{t+1} \)^2 \\
	 & =  \potrue{s}{t}^\top \(\vo{t+1} - \vp{t+1} \)^2 + (\pohat{s}{t} - \potrue{s}{t})^\top \(\vo{t+1} - \vp{t+1} \)^2 \\
	& =  \devi{t+1}^2 + (\pohat{s}{t} - \potrue{s}{t})^\top \(\vo{t+1} - \vp{t+1} \)^2 
\end{align}
The above inequality and $\sqrt{a+b} \leq \cc{\ref{lem:TransitionDynamicOptimims},3,9} (\sqrt{a} + \sqrt{b})$ for real $a,b$ allows us to write the following upper bound:
\begin{align}
 \Bv \sqrt{SA \ll} \cc{\ref{lem:TransitionDynamicOptimims},3,10} \times & \Biggm( \sqrt{ \underbrace{\sumk \sumt{1}\sumsLk \visita{s}{t}{k} \devi{t+1}^2}_{(a)}} \\
&  + \sqrt{\underbrace{ \sumk \sumt{1}\sumsLk \visita{s}{t}{k} \Big|\pohata{s}{t} - \potruea{s}{t})^\top \(\vo{t+1} - \vp{t+1} \)^2\Big|}_{(b)}}
\end{align}

%\begin{align}
%& \stackrel{}{\lesssim} \underbrace{ \sumk \sumt{1}\sumsLk \frac{\visita{s}{t}{k}}{\sqrt{\nsaa{s}{t}{k}}}\( \Bv \| \vo{t+1} - \vp{t+1} \|_{2,p} \)}_{(a)}+ \\
%& + \underbrace{\Bv  \sumk \sumt{1}\sumsLk  \frac{\visita{s}{t}{k}}{\sqrt{\nsaa{s}{t}{k}}} \( \devihat{t+1} - \devi{t+1}\) }_{(b)} \\
%& + \underbrace{ \sumk \sumt{1}\sumsLk  \frac{\visita{s}{t}{k}}{\sqrt{\nsaa{s}{t}{k}}}  \Bp \|\pohata{s}{t} - \potruea{s}{t}  \|_1}_{(c)} \\
%\end{align}

% for (b) there used to be this piece: \sum_{s'} \(\pohatsp{s}{t}{s'} - \potruesp{s}{t}{s'} \)\( \vo{t+1}(s') - \vp{t+1}(s') \)^2 }
To bound $(a)$ we use lemma \ref{lem:CumulativeDeltaOptimism}:
\begin{align}
(a) & \stackrel{}{\leq} \cc{\ref{lem:TransitionDynamicOptimims},3,11} \sqrt{ \sumk \sumt{1}\sumsLk \visita{s}{t}{k} \potruea{s}{t}^\top \( \vo{t+1}-\vp{t+1} \)^2} \leq  \cc{\ref{lem:CumulativeDeltaOptimism},3,12} \sqrt{\CDObound}\ll.
\end{align}
%\begin{align}
%& \stackrel{l}{\lesssim}  \tilde O \(\Bv \sqrt{SA}\) \times \sqrt{ \sumk \sumt{1} \E \( \sum_{\tau = t+1}^H \E \( \frac{D}{\sqrt{\nsaa{s_{\tau}}{\tau}{k}}} \Bigm| s,\piok{}{} \) \)^2}  \\
%& \stackrel{m}{\lesssim} \tilde O \(\Bv \sqrt{SA}\) \times \sqrt{ \sumk \sumt{1} \E \( \E \( \sum_{\tau = t+1}^H  \frac{D}{\sqrt{\nsaa{s_{\tau}}{\tau}{k}}} \)^2 \Bigm| s,\piok{}{} \) }  \\
%& \stackrel{n}{=}  \tilde O \(\Bv \sqrt{SA}\) \times \sqrt{ \sumk \sumt{1}  \E \( \sum_{\tau = t+1}^H  \frac{D}{\sqrt{\nsaa{s_{\tau}}{\tau}{k}}} \)^2 }  \\
%& \stackrel{o}{=} \tilde O \(D \sqrt{SA}\) \times \sqrt{ H \sumk \sumt{1} \sum_{\tau = t+1}^H \frac{\visita{s_{\tau}}{\tau}{k}}{{\nsaa{s_{\tau}}{\tau}{k}}} }  \\
%& \stackrel{p}{=} \tilde O \(DSAH\)  \\
%\end{align}
We now bound $(b)$, which is a lower order term. We don't leverage this fact here and bound it trivially by:
\begin{align}
(b) & \leq \cc{\ref{lem:TransitionDynamicOptimims},3,13} \sqrt{H}\sqrt{\sumk \sumt{1}\sumsLk \visita{s}{t}{k} (\pohata{s}{t} - \potruea{s}{t})^\top \(\vo{t+1} - \vp{t+1} \)}
\end{align}
The same computation as in Lemma \ref{lem:LowerOrderTerm} now gives:
 \begin{align}
 	(b) \leq \cc{\ref{lem:TransitionDynamicOptimims},3,14} \sqrt{H} \times \sqrt{\LOTbound }\ll 
 \end{align}
 This concludes the proof for the first bound.
 For the second bound proceed analogously but use the variant given by lemma \ref{lem:BoundBridge} when bounding the term ``$\approx$ Transition Dynamics Estimation'' in lemma \ref{lem:TransitionDynamicsEstimation} to obtain:
 \begin{align}
 & \sumk \sumt{1}\sumsLk \visita{s}{t}{k} \( \underbrace{\frac{\dg{s}{t+1}}{\sqrt{\nsaa{s}{t}{k}}}  + \frac{J+\BpPlus}{\nsaa{s}{t}{k}} }_{\approx \text{Transition Dynamics Estimation}} \) \\
 & \leq \cc{\ref{lem:TransitionDynamicOptimims},3,15} \( \(\sqrt{\ComplexityPi S A T} + (J+\Bp)SA + \DeltaGbound \) \ll^2 + \Bv \sqrt{SAH^2\Regret(K)\ll} \).	
 \end{align}
This concludes the proof.
 \end{proof}

\subsection{Lower Order Term}
\begin{lemma}[Lower Order Term]
\label{lem:LowerOrderTerm}
Outside of the failure event for \Alg{} it holds that \info{some are actually even lower order term and may vanish from the bound}:
\begin{align}
& \sumk \sumt{1}\sumsLk  \visita{s}{t}{k} \Big|\lowerorder \Big| = \\ &  \leq \cc{\ref{lem:LowerOrderTerm},3,1} \( \LOTbound \) \ll^2.
\end{align}
\end{lemma}
\begin{proof}
Using the concentration inequality on equation \ref{eqn:cip}  we get:
\begin{align}
& \leq \sumk \sumt{1}\sumsLk  \visita{s}{t}{k} \sum_{s'} \sqrt{\frac{\potruespa{s}{t}{s'}(1-\potruespa{s}{t}{s'})}{\nsaa{s}{t}{k}}} \Big|\vtrue{t+1}(s') - \vo{t+1}(s') \Big| \ll^{0.5}  \\
& + \sumk \sumt{1}\sumsLk  \visita{s}{t}{k} \sum_{s'} \frac{H}{\nsaa{s}{t}{k}} \ll
\end{align}
Since $\vo{t+1}-\vtrue{t+1} \leq \vo{t+1}-\vp{t+1}$ pointwise by Proposition \ref{prop:AlgorithmBracketsVstar} and 
by bounding the second term with Lemma \ref{lem:wn} and using $(1-p) \leq 1$ for $ p \in [0,1]$ we obtain:
\begin{align}
& \leq \sumk \sumt{1}\sumsLk  \visita{s}{t}{k} \sum_{s'} \sqrt{\frac{\potruespa{s}{t}{s'}}{\nsaa{s}{t}{k}}} \(\vo{t+1}(s') - \vp{t+1}(s') \)\ll^{0.5} +  \cc{\ref{lem:LowerOrderTerm},3,4} S^2AH \ll^2).
\end{align}
Cauchy-Schwartz leads to the following upper bound:
\begin{align}
& \leq  \cc{\ref{lem:LowerOrderTerm},3,5}  \sumk \sumt{1} \sumsLk \visita{s}{t}{k} \(\sqrt{\frac{ S\times \potruea{s}{t}^{\top}\( \vo{t+1} - \vp{t+1} \)^2 }{\nsaa{s}{t}{k}}}\)\ll^{0.5} + \cc{\ref{lem:LowerOrderTerm},3,5a} S^2AH\ll^2
\end{align}
One more application of Cauchy-Schwartz gives us:
\begin{align}
& \leq   \cc{\ref{lem:LowerOrderTerm},3,6}\sqrt{S} \sqrt{ \sumk \sumt{1} \sumsLk \frac{\visita{s}{t}{k}}{\nsaa{s}{t}{k}}} \times \sqrt{ \sumk \sumt{1} \sumsLk \visita{s}{t}{k}  \potruea{s}{t}^{\top}\( \vo{t+1} - \vp{t+1} \)^2 }\ll^{0.5} +   \cc{\ref{lem:LowerOrderTerm},3,7} S^2AH\ll^2
\end{align}
Recalling lemma \ref{lem:wn} and lemma \ref{lem:CumulativeDeltaOptimism} we obtain:
\begin{align}
& \leq \cc{\ref{lem:LowerOrderTerm},3,8}\(\sqrt{S} \times \sqrt{SA\ll} \) \times  \cc{\ref{lem:LowerOrderTerm},3,9}  (\sqrt{\CDObound})\ll^{0.5} + \cc{\ref{lem:LowerOrderTerm},3,10} S^2AH\ll^2
\end{align}
which can be simplified to obtain the statement.
\end{proof}

\begin{lemma}[Cumulative Delta Optimism]
\info{Check Before Submission} Outside of the failure event it holds that:
\label{lem:CumulativeDeltaOptimism}
\begin{equation}
	\sumk \sumt{1} \sumsa \visita{s}{t}{k} \potruea{s}{t}^\top \(\vo{t+1} - \vp{t+1} \)^2 = \cc{\ref{lem:CumulativeDeltaOptimism},3,1} (\CDObound) \ll.
\end{equation}
where $F$ and $D$ are defined in proposition \ref{prop:DeltaOptimism}.
\end{lemma}
\begin{proof}
Starting from the right hand side
\begin{align}
 & \sumk \sumt{1}  \sum_{(s,a)} \visita{s}{t}{k} \potruea{s}{t}^\top  \(\vo{t+1} - \vp{t+1} \)^2
 \end{align}
 we unroll the inner product between the transition probability vector and the value function obtaining 
 \begin{align}
 & =  \sumk \sumt{1}  \sum_{(s,a)} \visita{s}{t}{k} \( \sum_{s'}  p(s'\mid s,a)  \(\vo{t+1}(s') - \vp{t+1}(s') \)^2\).
  \end{align}
  Next, we move the summation operator (over $s'$) outside
 \begin{align}
  & =  \sumk \sumt{1}  \sum_{(s,a)} \sum_{s'}  \visita{s}{t}{k}  p(s'\mid s,a)  \(\vo{t+1}(s') - \vp{t+1}(s') \)^2 
   \end{align}
   and recall that $w_{t+1,k}(s',s,a) \defeq \visita{s}{t}{k}  p(s'\mid s,a)$ is the probability of taking action $a$ in $s$ and then landing in $s'$ at the next timestep. 
 \begin{align}
    & =  \sumk \sumt{1}  \sum_{(s,a)} \sum_{s'}  w_{t+1,k}(s',  s,a)  \(\vo{t+1}(s') - \vp{t+1}(s') \)^2.
     \end{align}
     Summing over all possible $s,a$ pairs one obtains the probability of being in $s'$ at timestep $t+1$
 \begin{align}
        & =  \sumk \sumt{1}  \sum_{s'}  w_{t+1,k}(s')  \(\vo{t+1}(s') - \vp{t+1}(s') \)^2 
\end{align}
which can be interpreted as an expectation over trajectories identified by $\tilde \pi_k$
 \begin{align}
  & = \sumk \sumt{1} \E_{s_{t+1}\sim \tilde \pi_k} \(\vo{t+1}(s_{t+1}) - \vp{t+1}(s_{t+1}) \)^2 \\
    & \leq \sumk \sumt{1} \E_{s_{t}\sim \tilde \pi_k} \(\vo{t}(s_{t}) - \vp{t}(s_{t}) \)^2. 
\end{align}
The last upper bound follows because we are counting over the same quantities, but we add timestep $t=1$ and drop timestep $t=H+1$ for which the value functions are zero.
Proposition \ref{prop:DeltaOptimism} justifies the first inequality below where $F$ and $D$ are defined in said proposition (here the action $a$ is the action taken by $\tilde \pi_k$ in $s_\tau$):
\begin{align}
& \leq \cc{\ref{lem:CumulativeDeltaOptimism},3,2} \sumk \sumt{1} \E_{s_{t} \sim \tilde \pi_k} \(\sum_{\tau = t}^H \E_{s_{\tau \sim \tilde \pi_k}} \frac{F+D}{\sqrt{\nsaa{s_\tau}{\tau}{k}}} \Bigm| s_{t} \)^2 \\
& \stackrel{a}{\leq}  \cc{\ref{lem:CumulativeDeltaOptimism},3,3}  H \sumk \sumt{1} \E_{s_{t} \sim \tilde \pi_k} \sum_{\tau = t}^H  \( \E_{s_{\tau \sim \tilde \pi_k}} \frac{F+D}{\sqrt{\nsaa{s_\tau}{\tau}{k}}} \Bigm| s_{t} \)^2 \\
 & \stackrel{b}{\leq} \cc{\ref{lem:CumulativeDeltaOptimism},3,4} H \sumk \sumt{1} \E_{s_{t} \sim \tilde \pi_k} \sum_{\tau = t}^H \E_{s_{\tau \sim \tilde \pi_k}} \frac{(F+D)^2}{\nsaa{s_{\tau}}{\tau}{k}} \Bigm| s_{t} \\
  & \leq \cc{\ref{lem:CumulativeDeltaOptimism},3,5} H \sumk \sumt{1} \sum_{\tau = t}^H  \E_{s_{\tau} \sim \tilde \pi_k} \frac{(F+D)^2}{\nsaa{s_{\tau}}{\tau}{k}}  \\
    &  \leq \cc{\ref{lem:CumulativeDeltaOptimism},3,6} H^2 \sumk \sumt{1}  \E_{s_{t} \sim \tilde \pi_k}  \frac{(F+D)^2}{\nsaa{s_t}{t}{k}} \\
   &  \leq \cc{\ref{lem:CumulativeDeltaOptimism},3,7} H^2 \( \sumk \sumt{1} \sumsLk \visita{s}{t}{k} \( \frac{(F+D)^2}{\nsaa{s}{t}{k}}\) + \sumk \sumt{1} \sum_{(s,a)\not \in \Ltk} \visita{s}{t}{k} H^2\)  \\
   &  \leq \cc{\ref{lem:CumulativeDeltaOptimism},3,7} H^2 (F+D)^2 \sumk \sumt{1} \sumsLk \(\frac{\visita{s}{t}{k}}{\nsaa{s}{t}{k}}\) + H^4\sumk \sumt{1} \sum_{(s,a)\not \in \Ltk} \visita{s}{t}{k}  \\
   & = \cc{\ref{lem:CumulativeDeltaOptimism},3,2} (\CDObound) \ll
% = \tilde O(SAH^2(F+D + H^{3/2})^2)
\end{align}
The last passage follows from lemma \ref{lem:wn} and \ref{lem:MinimalContribution}, while $(a)$ and $(b)$ follow from Cauchy-Schwartz and Jensen, respectively. To obtain the bound $H^2$ for the rightmost term (the one corresponding to states $(s,a) \not \in \Ltk$) we used the hard bound discussed in proposition \ref{prop:DeltaOptimism} that ``caps'' $\frac{(F+D)^2}{\nsaa{s}{t}{k}}$ at $H^2$.
\end{proof}

\subsection{Auxiliary Lemmas}
\begin{lemma}[Visitation Ratio]
\label{lem:wn}	
\begin{equation}
\sqrt{ \sumk \sumt{1}\sumsLk  \frac{\visita{s}{t}{k}}{\nsaa{s}{t}{k}}} \leq \cc{\ref{lem:wn},3,1} \sqrt{SA \ll} 
\end{equation}
\end{lemma}
\begin{proof}
Recall the definition $\sum_{t=1}^H w_{tk}(s,a) = w_k(s,a)$.
Lemma \ref{lem:VisitationRatio} ensures step $(a)$ below:
\begin{align}
    \sqrt{ \sumk \sumt{1}\sumsLk  \frac{\visita{s}{t}{k}}{\nsaa{s}{t}{k}}} & =  \sqrt{ \sumk \sumsLk  \frac{\visita{s}{}{k}}{\nsaa{s}{t}{k}}} \\
     & =  \sqrt{ \sumk \sumsa  \frac{\visita{s}{}{k}}{\nsaa{s}{t}{k}}\1\{(s,a) \in \Ltk\}} \\
    & \stackrel{(a)}{\leq}\cc{\ref{lem:wn},3,2} \sqrt{ \sumk \sumsa  \frac{\visita{s}{}{k}}{\sum_{\iota\leq k} \visita{s}{}{\iota}}\1\{(s,a) \in \Ltk\}} \\
\end{align}
It suffices to study 
\begin{equation}
\label{eqn:RatioToStudy}
     \sumk \frac{\visita{s}{}{k}}{\sum_{\iota\leq k} \visita{s}{}{\iota}}\1\{(s,a) \in \Ltk\}
\end{equation}
for a fixed $(s,a)$. The above quantity is non-zero only if $(s,a) \in \Ltk$ for some $k$. Since $\sum_{\iota\leq k}w_\iota(s,a)$ is strictly increasing with $k$, if $(s,a) \in \Ltk$ there must exist a critical episode $k_L \leq k$ (that depends on the $(s,a)$ pair) such that for all subsequent episodes $\iota\geq k_L$ we have that $(s,a) \in L_\iota$. Since by definition of $k_L$ it must be that $(s,a)\in L_{k_L}(s,a)$, we must have $ \sum_{\iota < k_L}w_\iota(s,a) + w_{k_L}(s,a) = \sum_{\iota \leq k_L}w_\iota(s,a) > 2H$ by definition \ref{def:TheGoodSet}, implying $\sum_{\iota < k_L}w_\iota(s,a) > H$. This way we lower bound the summation in the denominator as:
\begin{equation}
\sum_{\iota\leq k} \visita{s}{}{\iota} = \sum_{\iota < k_L} \visita{s}{}{\iota} + \sum_{k_L \leq \iota\leq k} \visita{s}{}{\iota} > H +  \sum_{k_L \leq \iota\leq k} \visita{s}{}{\iota}
\end{equation}
Therefore equation \ref{eqn:RatioToStudy} can be upper bounded as:
\begin{equation}
     \sumk \frac{\visita{s}{}{k}}{H + \sum_{k_L \leq \iota\leq k} \visita{s}{}{\iota}} \1\{(s,a) \in \Ltk\}.
\end{equation}
Since the indicator $\1\{(s,a) \in \Ltk\}$ is non-zero only when $k_L \leq k \leq K$, we can rewrite the above equation as:
\begin{equation}
     \sum_{k_L \leq k \leq K} \frac{\visita{s}{}{k}}{H + \sum_{k_L \leq \iota\leq k} \visita{s}{}{\iota}}.
\end{equation}
The above expression can be simplified in notation by setting $a_1 = w_{k_L}(s,a), a_2 = w_{k_L+1}(s,a),\dots, a_{K-k_L+1} = w_{K}(s,a)$.
Now define the function $F(x) = \sum_{i=1}^{\floor{x}} a_i + a_{\ceil{x}}(x-\floor{x})$, which is a function that coincides with the summation $\sum_{i=1}^{x} a_i$ for integer values of $x$ and interpolates between them. Its derivative is $f(x) = a_{\ceil{x}}$. This way we can write:

\begin{align}
 & \sum_{k=1}^{K-k_L+1}   \frac{a_k}{H+\sum_{i=1}^{k}a_i}  = \sum_{k=1}^{K-k_L+1} \frac{f(k)}{H+F(k)}
 \end{align}
 We have that $F(x) = \sum_{i=1}^{\floor{x}} a_i + a_{\ceil{x}}(\ceil{x}-x) \leq \sum_{i=1}^{\floor{x}} a_i + a_{\ceil{x}} = \sum_{i=1}^{\ceil{x}} a_i = F(\ceil{x})$
which justifies
 \begin{align}
\frac{f(\ceil{x})}{H+F(\ceil{x})} \leq  \frac{f(x)}{H+F(x)}.
 \end{align}
 Since the lhs is a step function, integrating the above yields:
 \begin{align}
 & \sum_{k=1}^{K-k_L+1} \frac{f(k)}{H+F(k)} =  \int_{0}^{K-k_L+1} \frac{f(\ceil{x})}{H+F(\ceil{x})}dx \\ 
 & \leq  \int_{0}^{K-k_L+1}\frac{f(x)}{H+F(x)} dx \\
 & = \ln(H + F(K-k_L+1)) -\ln(H+F(0)) \leq \ln(2H + KH) \leq \ln(T) \leq \ll.
 \end{align}
 Summing over all the $(s,a)$ pairs yields the result.
% [X]
% \begin{align}
%     F(x) & \stackrel{def}{=} H + \sum_{\iota < \ceil{x} \mid (s,a) \in L_\iota} \visita{s}{}{\iota} + (x - \floor{x})w_{\ceil{x}} \\
%          & \leq  H + \sum_{\iota < \ceil{x} \mid (s,a) \in L_\iota} \visita{s}{}{\iota} +  \visita{s}{}{\ceil{x}} \\
%          & =  H + \sum_{\iota \leq \ceil{x} \mid (s,a) \in L_\iota} \visita{s}{}{\iota}.
% \end{align}
% Therefore, $F(k)$ lower bounds the denominator in \ref{eqn:RatioToStudy}. Further, $F$ is continuous and differentiable except at points identified by integer values of $x$ with derivative $F'(x) = f(x) = w_{\ceil{x}}(s,a )$. Thus an upper bound to the ratio in equation \ref{eqn:RatioToStudy} is:
% \begin{align}
%              \sum_{k=1 \mid (s,a) \in \Ltk}^{K}   \frac{\visita{s}{}{k}}{ \sum_{\iota=1 \mid (s,a) \in L_\iota}^{k} \visita{s}{}{\iota} + H} \leq  \sum_{k=1 \mid (s,a) \in \Ltk} \frac{\visita{s}{}{k}}{F(k)} = \sum_{k=1 \mid (s,a) \in \Ltk} \frac{f(k)}{F(k)} \leq \int_{0}^{k} \frac{f'(x)}{f(x)} dx = \ln f(k) - \ln f(0)
% \end{align}
% because $f(x)$ is piecewise constant function and therefore 
% \begin{equation}  \sum_{k=1 \mid (s,a) \in \Ltk}^{K}   \visita{s}{}{k} \leq \int_{0}^{k} f'(x)\1\{\ceil{x} \in \Ltk\}dx.
% \end{equation}
% Finally, $f(0) = 2H$ and $f(x) = Hk $, and hence equation \ref{eqn:RatioToStudy} admits the upper bound:
% \begin{equation}
%      \sumk   \frac{\visita{s}{}{k}}{\sum_{\iota\leq k} \visita{s}{}{\iota}} = \tilde O(1).
% \end{equation}
% Summing over all the $(s,a)$ pairs yields the result.
\end{proof}

\begin{lemma}[Bound Bridge]
\label{lem:BoundBridge}
\begin{align}
& \sumk \sumt{1}\sumsLk  \visita{s}{t}{k} \frac{\dg{s}{t+1}}{\sqrt{\nsaa{s}{t}{k}}}	- \sumk \sumt{1}\sumsLk  \visita{s}{t}{k} \frac{\dogtrue{s}{t+1}}{\sqrt{\nsaa{s}{t}{k}}} \\
 & \leq  \cc{\ref{lem:BoundBridge},3,1} \( \DeltaGbound \ll + \Bv \sqrt{SAH^2\Regret(K) \ll} \).
\end{align}
\end{lemma}
\begin{proof}
Equation \ref{eqn:gVbound} in assumption \ref{ass:ConfidenceIntervals} ensures:
\begin{align}
& \sumk \sumt{1}\sumsLk  \visita{s}{t}{k} \frac{\dg{s}{t+1}}{\sqrt{\nsaa{s}{t}{k}}}	- \sumk \sumt{1}\sumsLk  \visita{s}{t}{k} \frac{\dogtrue{s}{t+1}}{\sqrt{\nsaa{s}{t}{k}}} \\
& \leq \Bv \sumk \sumt{1}\sumsLk  \visita{s}{t}{k} \frac{\| \vo{t+1} - \votrue{t+1} \|_{2,p}}{\sqrt{\nsaa{s}{t}{k}}}.
\end{align}
By adding and subtracting $\vtrue{t+1,k}$ inside the norm operator we obtain:
\begin{align}
& = \Bv \sumk \sumt{1}\sumsLk  \visita{s}{t}{k} \frac{\| \vo{t+1} - \vtrue{t+1,k} + \vtrue{t+1,k} - \votrue{t+1} \|_{2,p}}{\sqrt{\nsaa{s}{t}{k}}} \\
& \leq \underbrace{\Bv \sumk \sumt{1}\sumsLk  \visita{s}{t}{k} \frac{\| \vo{t+1} - \vtrue{t+1,k} \|_{2,p}}{\sqrt{\nsaa{s}{t}{k}}}}_{A} + \underbrace{\Bv \sumk \sumt{1}\sumsLk  \visita{s}{t}{k} \frac{\|\vtrue{t+1,k} - \votrue{t+1} \|_{2,p}}{\sqrt{\nsaa{s}{t}{k}}} }_{B}.
\end{align}
In particular, the upper bound follows by the triangle inequality. Below we bound term $A$. Lemma \ref{lem:OptimismOverestimate} and proposition \ref{prop:AlgorithmBracketsVstar} ensure the upper bound below:
\begin{align}
A \leq \Bv \sumk \sumt{1}\sumsLk  \visita{s}{t}{k} \frac{\| \vo{t+1} - \vp{t+1} \|_{2,p}}{\sqrt{\nsaa{s}{t}{k}}}
\end{align}
from which Cauchy-Schwartz yields:
\begin{align}
& \leq \Bv \sqrt{\sumk \sumt{1}\sumsLk  \frac{\visita{s}{t}{k}}{\nsaa{s}{t}{k}}} \sqrt{\sumk \sumt{1}\sumsLk  \visita{s}{t}{k}\| \vo{t+1} - \vp{t+1} \|^2_{2,p}} \\
& \leq  \cc{\ref{lem:BoundBridge},3,4}\sqrt{SA\ll} \times \Bv \sqrt{\sumk \sumt{1}\sumsLk  \visita{s}{t}{k}\| \vo{t+1} - \vp{t+1} \|^2_{2,p}} \\
& \leq  \cc{\ref{lem:BoundBridge},3,4} \sqrt{SA\ll} \times \Bv\sqrt{\sumk \sumt{1} \sumsa \visita{s}{t}{k} \potruea{s}{t}^\top \(\vo{t+1} - \vp{t+1} \)^2}  \\
& \leq \cc{\ref{lem:BoundBridge},3,4}\sqrt{SA\ll} \times \sqrt{\CDObound \ll}   \leq \cc{\ref{lem:wn},3,5} \DeltaGbound \ll
\end{align}
where the bounds follow from lemma \ref{lem:wn} and \ref{lem:CumulativeDeltaOptimism}.
It now remains to bound term $B$. By an identical argument using Cauchy-Schwartz we have that:
\begin{align}
	B & = \Bv \sumk \sumt{1}\sumsLk  \visita{s}{t}{k} \frac{\|\vtrue{t+1,k} - \votrue{t+1} \|_{2,p}}{\sqrt{\nsaa{s}{t}{k}}} \\
	& \leq \Bv \sqrt{\sumk \sumt{1}\sumsLk  \frac{\visita{s}{t}{k}}{\nsaa{s}{t}{k}}} \sqrt{\sumk \sumt{1}\sumsLk  \visita{s}{t}{k}\| \vtrue{t+1,k} - \votrue{t+1} \|^2_{2,p}} \\ 
	& \leq \cc{\ref{lem:BoundBridge},3,5} \Bv \sqrt{SA\ll}\sqrt{H^2\Regret(K)}.
\end{align}
The last passage follows from lemma \ref{lem:wn} and \ref{lem:UpperBoundOnPartialLoss}.
\end{proof}

\begin{lemma}[LTV]
\label{lem:LTV}
\info{Check}
The following inequality holds true:
\begin{equation}
\E_{\piok}\( \(\sumt{1} \rotrue{s_t}{t} - \votrue{1}(s_1) \)^2 \Big| s_1 \)  = \E_{\piok}\( \sumt{1} \Var_{\piok} \(\votrue{t+1}(s_{t+1}) \Big| s_t \) \Big| s_1 \).
\end{equation}
where the expectation $\E_{\piok}(\cdot \mid s_1)$ is taken with respect to the trajectories followed by the agent upon following policy \piok starting from $s_1$.
\end{lemma}
\begin{proof}
\begin{align}
& \E_{\piok}\( \(\sumt{1} \rotrue{s_t}{t} - \votrue{1}(s_1) \)^2 \Big| s_1 \) = \\
& = \E_{\piok}\( \(\( \rotrue{s_1}{t} + \sumt{2}\rotrue{s_t}{t} \) - \( \rotrue{s_1}{t} + \E_{\piok} \votrue{2}(s_2) \) \)^2 \Big| s_1 \) \\
& = \E_{\piok}\(\(\sumt{2}\rotrue{s_t}{t} - \E_{\piok} \votrue{2}(s_2) \)^2 \mid  s_1 \) \\
& = \E_{\piok}\(\(\sumt{2}\rotrue{s_t}{t} - \votrue{2}(s_2) + \votrue{2}(s_2) - \E_{\piok} \votrue{2}(s_2) \)^2 \Big| s_1 \) \\
& = \E_{\piok}\(\( \sumt{2}\rotrue{s_t}{t} - \votrue{2}(s_2) \)^2\Big| s_1 \) + \E_{\piok} \(\votrue{2}(s_2) - \E_{\piok}\votrue{2}(s_2) \Big| s_1 \)^2 \\
& = \E_{\piok}\(\E \( \( \sumt{2}\rotrue{s_t}{t} - \votrue{2}(s_2) \)^2\Big| s_2\) \Big| s_1 \) + \Var_{\piok} \(\votrue{2}(s_2) \Big| s_1 \) \\
& = \E_{\piok}\( \sumt{2} \E_{\piok}\( \sumt{2} \Var_{\piok} \(\votrue{t+1}(s_{t+1}) \Big| s_t \) \Big| s_2 \) \mid s_1\) + \Var_{\piok} \(\votrue{2}(s_2) \Big| s_1 \) \\
& \stackrel{}{=} \E_{\piok}\( \sumt{1} \Var_{\piok} \(\votrue{t+1}(s_{t+1}) \Big| s_t \) \Big| s_1 \) \\
\end{align}
See for example \cite{Azar12} for a proof equivalent to this.
\end{proof}

\begin{lemma}[Upper Bound on Partial Loss]
\label{lem:UpperBoundOnPartialLoss}
Define the regret (with the starting states $\{s_{1k} \}_{k=1,\dots,K}$) up to episode $K$ as:
\begin{align}
\Regret(K) \stackrel{def}{=} \sumk \(\vtrue{1} - \votrue{1k} \)(s_{1k}).
\end{align}
Then it holds that:
\begin{align}
\sumk \sumt{1} \sumsa \visita{s}{t}{k} \potruea{s}{t}^\top \(\vtrue{t+1} - \votrue{t+1} \)^2  \leq H^2 \Regret(K).
 \end{align}
 
\end{lemma}
\begin{proof}
\begin{align}
\sumk \sumt{1} \sumsa \visita{s}{t}{k} \potruea{s}{t}^\top \(\vtrue{t+1} - \votrue{t+1} \)^2 & \leq H \sumk \sumt{1} \sumsa \visita{s}{t}{k} \potruea{s}{t}^\top \(\vtrue{t+1} - \votrue{t+1} \) \\
 	& = H \sumk \sumt{1} \sum_{s'} w_{t+1,k}(s') \(\vtrue{t+1} - \votrue{t+1} \)(s') \\
 	& \stackrel{(a)}{\leq} H \sumk \sumt{1} \(\vtrue{1} - \votrue{1k} \)(s_{1k}) \\
 & = H^2 \underbrace{\sumk \(\vtrue{1} - \votrue{1k} \)(s_{1k})}_{\Regret(K)} \\
 & = H^2 \Regret(K).
\end{align}
Here $(a)$ follows from lemma \ref{lem:PolicyExchange}.

\end{proof}

\begin{lemma}
\label{lem:PolicyExchange}
		Let $s_{1k}$ be the starting state in episode $k$, and $w_{t+1,k}(s') = \sum_a \visita{s'}{t+1,}{k}$. It holds that:
	\begin{align}
	\sum_{s'} w_{t+1,k}(s') \(\vtrue{t+1} - \votrue{t+1} \)(s') \leq \(\vtrue{1} - \votrue{1k} \)(s_{1k})	
	\end{align}
\end{lemma}
\begin{proof}
Define the policy $\mu$ as the policy that follows $\piok$	 up to timestep $t$ and $\pistar$ afterwards (until the end of the episode).
We have that for any starting state $s_{1k}$:
\begin{align}
	\vtrue{1k}(s_{1k}) \geq V^{\mu}_{1k}(s_{1k}) \geq \vostrue{1k}{s_{1k}}.
\end{align}
The rightmost inequality is true because $\mu$ follows $\pistar$ once it gets to timestep $ \geq t+1$. This argument also justifies the step below:
\begin{align}
\sum_{s'} w_{t+1,k}(s') \(\vtrue{t+1} - \votrue{t+1} \)(s') = \sum_{s'} w_{t+1,k}(s') \(V^{\mu}_{t+1} - \votrue{t+1} \)(s') = V^{\mu}_{1k}(s_{1k}) -  \vostrue{1k}{s_{1k}} \leq V^{\star}_{1}(s_{1k}) -  \vostrue{1k}{s_{1k}}.
\end{align}
\end{proof}

\end{document}